\pdfoutput=1
\documentclass[twoside,11pt]{article}

\usepackage[abbrvbib, preprint]{jmlr2e}

\usepackage{amsmath}
\usepackage{booktabs}
\usepackage{hyperref}
\usepackage{cleveref}
\usepackage{algorithm}
\usepackage{algorithmic}
\usepackage{color}
\usepackage{xcolor}
\usepackage{enumitem}

\newtheorem{assumption}{Assumption}
\crefname{assumption}{assumption}{assumptions}
\Crefname{assumption}{Assumption}{Assumptions}

\newcommand{\norm}[1]{\|#1\|}
\newcommand{\inner}[2]{\langle #1, #2 \rangle}

\begin{document}

\title{Guided Flow Matching for Forward and Inverse PDE Problems with Sparse Observations: Algorithm and Theory}

\author{\name Xifeng Zhang \email <cnxifeng9819@163.com>\\
\addr School of Mathematical Science\\
Capital Normal University\\
Beijing, 100048, China
\AND
\name Jin Zhao \thanks{Corresponding author.} \email <zjin@cnu.edu.cn>\\
\addr Academy for Multidisciplinary Studies\\
Capital Normal University\\
Beijing, 100048, China
}

\maketitle

\begin{abstract}
    Reconstructing PDE solutions from sparse observations is a core challenge in scientific computing. We present FM4PDE, a flow-matching generative framework that learns the joint distribution of PDE coefficients (or initial states) and solutions (or final states), enabling both forward simulation and inverse recovery with limited paired data. At inference, sampling is guided by a composite loss that enforces agreement with sparse measurements and reduces the PDE residual; we support deterministic, stochastic, and hybrid samplers. We provide error guarantees for these guided procedures. For the deterministic optimizer, a coercivity condition ensures trajectory boundedness and a phase-wise contraction yields logarithmic complexity in the target accuracy. For the stochastic sampler, we introduce adaptive guidance and assume dissipativity of the velocity field to obtain uniform moment bounds independent of the noise-floor parameter. This leads to polynomial-time error bounds, and a matching lower bound shows constant guidance induces an unavoidable positive bias, motivating adaptivity. A hybrid deterministic–stochastic analysis is also provided. Experiments on static and time-dependent benchmark PDEs demonstrate competitive accuracy and faster inference than diffusion-based generative models.
\end{abstract}


\section{Introduction}
\label{sec:introduction}

Partial differential equations (PDEs) are fundamental to the mathematical modeling of physical, biological, and engineering systems. Both forward problems---computing the solution given governing equations, initial conditions, and boundary conditions---and inverse problems---inferring unknown parameters from observations of the system state---are central tasks in computational science \citep{strauss2007partial,evans2022partial}. In many practical settings, however, the system state can only be observed at a sparse set of spatial or temporal locations due to sensor limitations, measurement costs, or accessibility constraints. Reconstructing the full solution field from such sparse observations is therefore an important and challenging problem.

Classical numerical methods and supervised learning approaches such as Physics-Informed Neural Networks \citep{raissi2019pinns} and neural operators \citep{li2020fno,lu2021deeponet,kovachki2023neur} have achieved strong performance when full-field data or dense observations are available. These methods, however, are not directly designed for the sparse-observation setting, where the reconstruction problem is inherently ill-posed and a probabilistic treatment is natural.

Generative models provide a principled framework for such problems by learning the data distribution and generating plausible reconstructions conditioned on partial observations. Flow matching \citep{lipman2022flow,lipman2024flow,liu2022flow} has emerged as a particularly attractive generative paradigm: it learns a velocity field that transports a simple source distribution (e.g., Gaussian noise) to the target data distribution along smooth trajectories, offering stable training and efficient sampling compared with diffusion-based alternatives \citep{ho2020denoising,song2020score}. Guided sampling, in which a task-specific loss function steers the generation process \citep{ho2022classifier}, enables conditional generation without retraining.

\subsection{Related Works}
\label{related_works}

\paragraph{Physics-Informed and Operator Learning.} PINNs \citep{raissi2019pinns} embed PDE constraints directly into the neural network loss function, enabling solutions consistent with physical laws even with limited labeled data. Fourier Neural Operators (FNO) \citep{li2020fno} learn operator mappings in the Fourier domain for efficient resolution-invariant PDE solving. DeepONet \citep{lu2021deeponet} approximates nonlinear operators via a branch-trunk architecture. The general neural operator framework \citep{kovachki2023neur} extends these ideas to mappings between infinite-dimensional function spaces with provable approximation guarantees, and several further directions have been explored \citep{boulle2024operator,lanthaler2023operator,stepaniants2023learning,dalton2024boundary}. While these methods perform well with access to full solution data, they are not specifically designed to reconstruct global fields from sparse observations.

\paragraph{Generative Models for PDE Solving.} DiffusionPDE \citep{huang2024diffusionpde} learns the joint distribution of PDE coefficients and solutions using diffusion models, enabling simultaneous inference and solving under partial observations. CoCoGen \citep{jacobsen2025cocogen} applies score-based generative modeling with ControlNet-style conditioning \citep{zhang2023controlnet} for physics-informed sampling and field reconstruction. Flow-matching-based approaches such as PCFM \citep{utkarsh2025PCFM} and PBFM \citep{baldan2025PBFM} focus on enforcing physical consistency during the generative process. These methods demonstrate promising empirical results, but diffusion-based approaches typically incur substantial computational cost in both training and sampling. Moreover, rigorous convergence analysis for guided generative PDE solvers remains limited. On the theoretical side, convergence guarantees for score-based diffusion sampling have been established under various assumptions \citep{chen2022sampling,benton2023nearly}, but analogous results for guided flow matching---particularly in the PDE-constrained setting---have not been developed.

\subsection{Contributions}

We propose FM4PDE (Flow Matching for PDE-Solving), a guided flow matching framework for reconstructing global PDE solutions from sparse observations. Our main contributions are as follows.

\begin{enumerate}[label=(\roman*), topsep=4pt, itemsep=2pt]

\item \textbf{Framework and algorithm.} FM4PDE learns the joint distribution of paired PDE coefficients (or initial states) and solutions (or final states) via flow matching. During inference, a composite guidance loss $\mathcal{L} = \zeta_{\mathrm{obs}}\mathcal{L}_{\mathrm{obs}} + \zeta_{\mathrm{pde}}\mathcal{L}_{\mathrm{pde}}$, incorporating sparse observation fidelity and PDE residual constraints, steers the sampling process. We develop three sampling strategies---deterministic, stochastic, and hybrid---and present them in a unified algorithmic framework (Section~\ref{sec:methods}).

\item \textbf{Theoretical analysis: deterministic optimizer.} We identify the deterministic guided sampler as a mode-seeking optimizer and prove trajectory boundedness via a coercivity condition on the loss function combined with a dissipativity condition on the velocity field. Under a Polyak--{\L}ojasiewicz condition and a velocity-loss interaction assumption, we establish a phase-dependent loss contraction: the guidance coefficient $b_t = 1/t - 1$ drives exponential descent in Phase~A ($t < t_*$), yielding a contraction factor of $\epsilon^{2\mu}$ (where $\epsilon$ is the initial time parameter). The resulting complexity is $O(\log(1/\varepsilon))$---logarithmic in target accuracy (Section~\ref{sec:thm}).

\item \textbf{Theoretical analysis: stochastic sampler with adaptive guidance.} For the stochastic sampler, we introduce adaptive guidance $\zeta_k = c_\zeta\delta_k$ (where $\delta_k = 1 - t_k$) coupled with a dissipativity assumption on the velocity field. This combination yields uniform moment bounds that are independent of the noise floor parameter $\delta_{\min}$, breaking a circular dependence between iterate boundedness and loss descent that arises in constant-guidance analyses. The error bound achieves $O(\varepsilon^{-2}\log(1/\varepsilon))$ complexity. We also prove a matching lower bound showing that constant guidance incurs an unavoidable positive bias $V_{ss} \ge \delta_{\min}/40 > 0$, establishing the necessity of adaptive guidance (Section~\ref{sec:thm} and Appendix~\ref{apx:lower_bound}).

\item \textbf{Hybrid framework.} We combine deterministic and stochastic phases: the deterministic optimizer provides fast initial convergence via Phase~A contraction in $O(\log(1/\varepsilon))$ steps, followed by the stochastic sampler with adaptive guidance for exploration and robustness. The overall complexity is $O(\varepsilon^{-2}\log(1/\varepsilon))$ (Section~\ref{sec:thm}).

\item \textbf{Experiments.} We evaluate FM4PDE on a range of benchmark PDEs, both static (Darcy flow, Poisson, Helmholtz) and time-dependent (Navier--Stokes, Burgers, reaction-diffusion, shallow water), for forward and inverse problems under sparse observations. FM4PDE achieves competitive accuracy with significantly faster inference compared with diffusion-based generative approaches (Section~\ref{sec:experiments}).

\end{enumerate}

\subsection{Paper Organization}

The remainder of the paper is organized as follows. Section~\ref{sec:methods} introduces the flow matching framework and presents the FM4PDE algorithm, including the training procedure, the composite guidance loss, and the three sampling strategies. Section~\ref{sec:thm} establishes the theoretical properties of the guided samplers: trajectory boundedness, convergence rates, and the lower bound for constant guidance. Section~\ref{sec:experiments} presents experimental results on benchmark PDEs. Section~\ref{sec:discussion} discusses the effects of observation sparsity, step sizes, and sampling strategy choices. Section~\ref{sec:conclusion} concludes the paper. Additional theoretical analysis is provided in the Appendix.

\section{Methodology}
\label{sec:methods}

In this section, we propose a pre-trained Flow Matching framework to solve forward and inverse PDE problems. By learning the joint distribution of coefficients and solutions, our method reconstructs these fields via ODE solvers, using sparse observations and PDE constraints as guidance.

\subsection{Preliminaries: Flow Matching}

Flow Matching has recently gained significant attention as a powerful paradigm in generative modeling \citep{lipman2022flow, lipman2024flow}. In this section, we provide a concise overview of its fundamental framework. Given a source distribution $p$ and a target distribution $q$, the primary goal of Flow Matching is to learn a continuous mapping that transports samples from the source to the target distribution. Conceptually, this corresponds to learning a time-dependent flow that gradually evolves $p$ into $q$ along smooth trajectories in the data space. Generally, a standard Flow Matching model consists of four main components: data preparation, probability path design, model training, and sampling.

\paragraph{Preparing Data.} Let $X_0 \sim p$ and $X_1 \sim q$. The first step is to define the source and target distributions, and to determine how samples $X_0$ and $X_1$ are coupled. A common and simple setting assumes that $X_0$ follows a standard Gaussian distribution $\mathcal{N}(\mathbf{0}, I)$ and is independent of $X_1$. Under this assumption, their joint distribution can be written as
\begin{equation*}
    (X_0, X_1) \sim \Pi(\boldsymbol{x}_0, \boldsymbol{x}_1) = \mathcal{N}(\boldsymbol{x}_0; \mathbf{0}, I)\, q(\boldsymbol{x}_1).
\end{equation*}

\paragraph{Designing Probability Paths.} A probability path is a family of intermediate, time-dependent probability densities $\{ p_t \}_{0 \leq t \leq 1}$ that smoothly connects the source and target distributions. Specifically, $X_t \sim p_t$ denotes the distribution of samples at time $t$. For any path $p_t$, if $X_t = \psi_t(X_0) \sim p_t$ for all $t \in [0,1)$, then $p_t$ is said to be generated by a vector field $u_t$ corresponding to $\psi_t$, and the pair $(u_t, p_t)$ satisfies the continuity equation:
\begin{equation*}
    \frac{\mathrm{d}}{\mathrm{d} t} p_t(\boldsymbol{x}) + \mathrm{div}\!\left(p_t u_t\right)(\boldsymbol{x}) = 0,
\end{equation*}
where $\mathrm{div}(v)(\boldsymbol{x}) = \sum_{i=1}^d \partial_{\boldsymbol{x}^i} v^i(\boldsymbol{x})$ for a vector field $v(\boldsymbol{x}) = (v^1(\boldsymbol{x}), \ldots, v^d(\boldsymbol{x}))$.

Flow Matching aims to learn the vector field $u_t$ that transforms $p$ into $q$. Instead of modeling the marginal quantities $p_t(\boldsymbol{x})$ and $u_t(\boldsymbol{x})$ directly, it is often more tractable to construct conditional forms, namely the conditional vector field $u_t(\boldsymbol{x} \mid \boldsymbol{x}_1)$ and the conditional path $p_{t \mid 1}(\boldsymbol{x} \mid \boldsymbol{x}_1)$ for a given target sample $\boldsymbol{x}_1$. In this case, $p_t(\boldsymbol{x})$ can be viewed as a mixture over $\boldsymbol{x}_1$, and the unconditional and conditional vector fields coincide.

A widely adopted path construction is the conditionally optimal transport (or linear interpolation) path:
\begin{equation}\label{eq:cond_path}
    X_{t \mid 1} = t\boldsymbol{x}_1 + (1 - t)X_0 \sim p_{t \mid 1} = \mathcal{N}(t\boldsymbol{x}_1, (1-t)^2 I)
\end{equation}
with its corresponding conditional vector field:
\begin{equation*}
    u_t(\boldsymbol{x}|\boldsymbol{x}_1) = \frac{\boldsymbol{x}_1 - \boldsymbol{x}}{1 - t}.
\end{equation*}
This simple yet effective formulation provides a closed-form expression for the flow field, enabling stable and efficient training.

\paragraph{Training.}
The Flow Matching model parameterizes the vector field $u_t$ using a neural network $u_t^{\theta}$. The training objective minimizes the discrepancy between the predicted and true conditional vector fields:
\begin{equation}\label{eq:cond_loss}
    \mathcal{L}_{\mathrm{CFM}}(\theta) = \mathbb{E}_{t, X_t, X_1}\big[\|u_t^{\theta}(X_t) - u_t(X_t \mid X_1)\|^2\big],
\end{equation}
where $t \sim \mathcal{U}(0,1)$, $X_0 \sim p$, $X_1 \sim q$, and $X_t = tX_1 + (1 - t)X_0$. This conditional loss is equivalent in gradient to the following unconditional formulation:
\begin{equation*}
    \mathcal{L}_{\mathrm{FM}}(\theta) = \mathbb{E}_{t, X_t}\big[\|u_t^{\theta}(X_t) - u_t(X_t)\|^2\big],
\end{equation*}
which simplifies implementation while maintaining identical optimization dynamics. As shown in the Training Phase of Figure~\ref{fig:fm4pde_framework}, we train the network using a standard linear probability path-based Flow Matching workflow and the objective function in Equation~\eqref{eq:cond_loss}.

\paragraph{Sampling.}
Once trained, samples from the target distribution are generated by solving the following ordinary differential equation (ODE) defined by the learned flow:
\begin{equation}\label{eq:sample_ode}
    \frac{\mathrm{d}}{\mathrm{d} t} X_t = u_t^{\theta}(X_t), \quad \text{with } X_0 \sim p,~ t \in [0,1].
\end{equation}
This corresponds to a deterministic sampling strategy, where $X_t$ is obtained by numerically solving the ODE using standard solvers such as Euler, Runge-Kutta, or the midpoint method. In practice, the midpoint method often achieves the best balance between accuracy and stability.

To enhance sample diversity and flexibility, we can alternatively employ stochastic sampling schemes. One direct approach is to interpret the sampling procedure from an SDE perspective, where Flow Matching defines the drift term of a continuous-time stochastic process. Under this view, the sampling dynamics correspond to an Euler--Maruyama discretization of an SDE with learned drift and injected diffusion noise, providing a principled way to introduce stochasticity into deterministic flow models \citep{singh2024stochastic,lai2025principles}. Alternatively, one can first perform a deterministic update by integrating the learned velocity field to time 1, and then interpolate back to the intermediate time step \citep{baldan2025PBFM}. These stochastic corrections enrich sample diversity and help the model better capture multi-modal structures in the target distribution.

\paragraph{Flow Matching and Score Functions.} Let $\alpha_t, \sigma_t: [0, 1] \to [0, 1]$ be smooth functions satisfying
\begin{equation*}
    \alpha_0 = 0 = \sigma_1, \alpha_1 = 1 = \sigma_0 \text{ and } \dot{\alpha}_t - \dot{\sigma}_t > 0 \text{ for } t \in (0, 1). 
\end{equation*}
One useful quantity admitting a simple form in the Gaussian case is the score, defined as the gradient of the log probability. The score of the conditional path in Equation~\eqref{eq:cond_path} follows the expression
\begin{equation*}
    \nabla \log p_{t \mid 1}(\boldsymbol{x} \mid \boldsymbol{x}_1) = -\dfrac{1}{\sigma_t} (\boldsymbol{x} - \alpha_t \boldsymbol{x}_1).
\end{equation*}
Then, the conditional velocity for Gaussian paths can be written in the form
\begin{equation}\label{eq:ut_with_score}
	u_t(\boldsymbol{x} \mid \boldsymbol{x}_1) = a_t \boldsymbol{x} + b_t \nabla_{\boldsymbol{x}} \log p_{t \mid 1}(\boldsymbol{x} \mid \boldsymbol{x}_1),
\end{equation}
where
\begin{equation*}
	a_t = \dfrac{\dot{\alpha}_t}{\alpha_t}, \quad b_t = - \dfrac{\dot{\sigma}_t \sigma_t \alpha_t - \dot{\alpha}_t \sigma_t^2}{\alpha_t}.
\end{equation*}
In particular, for linear paths, $\alpha_t = t$, $\sigma_t = 1-t$, $a_t = 1/t$ and $b_t = 1/t-1$. 

\subsection{Solving PDEs with Flow Matching}

\begin{figure}[H]
    \centering
    \includegraphics[width=1\textwidth]{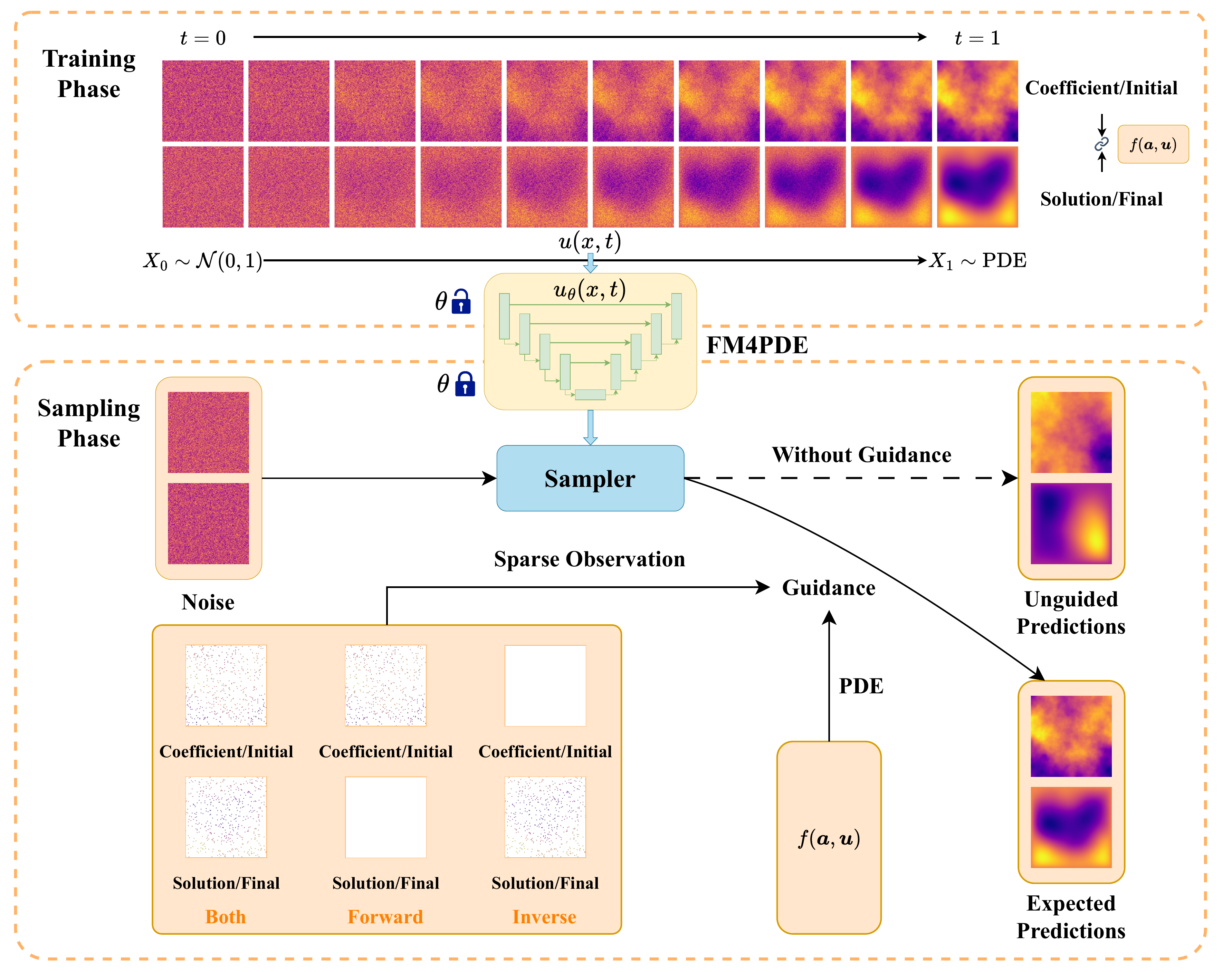}
    \caption{Framework of FM4PDE. The framework consists of two phases: (i) a training phase, where a Flow Matching network is trained on paired PDE coefficients (or initial states) and solutions (or final states); and (ii) a sampling phase, where the pretrained model is frozen and used for inference.}
    \label{fig:fm4pde_framework}
\end{figure}

\paragraph{Algorithm Overview.} The algorithm consists of two distinct phases. In the training phase, we train a Flow Matching neural network utilizing a dataset composed of paired coefficients (or initial states) and solutions (or final states), which are intrinsically linked by the underlying PDE equations. In the sampling phase, the parameters of the pre-trained model are frozen, and the proposed sampling algorithm is employed for inference. In the absence of guidance, FM4PDE generates data that inherently satisfies PDE constraints. When sparse observations are available, we integrate a PDE-based loss with a guidance mechanism to steer the generation process, yielding global reconstruction results that align with the observations. Depending on the specific pattern of the provided observations, the framework is capable of solving both forward problems, where observations are limited to coefficients or initial states, and inverse problems, where observations consist solely of solutions or final states.

\paragraph{Training Phase.} We focus on two classes of PDEs: static and dynamic systems. Let $f$ denote the PDE operator, let $\Omega$ represent a bounded domain with boundary $\partial \Omega$, and let $\boldsymbol{c} \in \Omega$ denote the spatial coordinates. First, static systems, such as Darcy flow or the Poisson equation, are governed by a time-independent function. The problem is defined as:
\begin{equation*}
    \begin{aligned}
        f(\boldsymbol{c}; \boldsymbol{a}, \boldsymbol{u}) &= 0 && \text{in } \Omega \subset \mathbb{R}^d, \\
        \boldsymbol{u}(\boldsymbol{c}) &= \boldsymbol{g}(\boldsymbol{c}) && \text{on } \partial \Omega,
    \end{aligned}
\end{equation*}
where $\boldsymbol{a} \in \mathcal{A}$ represents the PDE coefficient field, and $\boldsymbol{u} \in \mathcal{U}$ is the solution field. The term $\boldsymbol{u}|_{\partial\Omega} = \boldsymbol{g}$ specifies the boundary constraint. Our objective is to recover both $\boldsymbol{a}$ and $\boldsymbol{u}$ from sparse observations of either $\boldsymbol{a}$, $\boldsymbol{u}$, or both.

Second, we consider dynamic systems, such as the Navier-Stokes equations, formulated as
\begin{equation*}
    \begin{aligned}
        f(\boldsymbol{c}, \tau; \boldsymbol{a}, \boldsymbol{u}) &= 0 && \text{in } \Omega \times (0, \infty), \\
        \boldsymbol{u}(\boldsymbol{c}, \tau) &= \boldsymbol{g}(\boldsymbol{c}, \tau) && \text{on } \partial \Omega \times (0, \infty), \\
        \boldsymbol{u}(\boldsymbol{c}, \tau) &= \boldsymbol{a}(\boldsymbol{c}, \tau) && \text{on } \bar{\Omega} \times \{0\},
    \end{aligned}
\end{equation*}
where $\bar{\Omega} = \Omega \cup \partial \Omega$. Here, $\tau$ represents the physical time coordinate. The parameter $\boldsymbol{a} = \boldsymbol{u}_0 \in \mathcal{A}$ denotes the initial condition, while $\boldsymbol{u}$ is the time-dependent solution field. In this setting, we aim to simultaneously recover the initial state $\boldsymbol{a}$ and the solution at a specific terminal time $\mathcal{T}$, denoted as $\boldsymbol{u}_{\mathcal{T}} := \boldsymbol{u}(\cdot, \mathcal{T})$, from sparse observations.

For each specified PDE, we construct a corresponding dataset and train a Flow Matching neural network $u_t^{\theta}$ on paired samples of the form $\boldsymbol{x}_1 = (\boldsymbol{a}, \boldsymbol{u}) \in \mathcal{X}$, which means the target distribution is the joint distribution over PDE coefficients (or initial conditions) and solution fields that satisfy the governing PDE and the associated boundary or initial conditions. This phase is shown in Figure~\ref{fig:fm4pde_framework}.

\paragraph{Sampling Phase.} Let $\{t_k\}_{k=0}^{N}$ denote the time grid for the sampling phase, satisfying $t_0 < t_1 < \cdots < t_N$. Here, $N$ represents the total number of discretization steps. Accordingly, we define the time step size at iteration $k$ as $\Delta t_k := t_{k+1} - t_k$. Once $u_t^{\theta}$ is trained, the model is capable of generating samples that adhere to the underlying PDE constraints using standard Flow Matching sampling procedures, even in the absence of conditional guidance.

To enable effective sparse recovery for both forward and inverse problems, we introduce a guidance mechanism that corrects the generated trajectories during sampling. After training the Flow Matching velocity field $u_t^{\theta}$, we incorporate sparse observations and the governing PDE to guide the sampling process during inference. Let $\hat{\boldsymbol{x}}_t$ denote the generated state at time $t$, and let $\mathcal{P}$ be a projection operator that extracts values at the observation locations.
We define the observation loss as
\begin{equation}\label{eq:loss_obs}
    \mathcal{L}_{\mathrm{obs}}\left(\hat{\boldsymbol{x}}_t, \boldsymbol{x}_{\mathrm{obs}}\right)
    = \frac{1}{n}\bigl\|\mathcal{P} (\hat{\boldsymbol{x}}_t) - \boldsymbol{x}_{\mathrm{obs}}\bigr\|_2^2
    = \frac{1}{n}\sum_{j=1}^{n}
    \bigl(\hat{\boldsymbol{x}}_t(\boldsymbol{o}_j) - \boldsymbol{x}_{\mathrm{obs}}(\boldsymbol{o}_j)\bigr)^2,
\end{equation}
where $n$ is the number of observation points $\{\boldsymbol{o}_j\}_{j=1}^n$. In practice, the PDE residual is evaluated using finite difference discretizations \citep{leveque2007finite}. To enforce physical consistency, we further introduce a PDE residual loss
\begin{equation}\label{eq:loss_pde}
    \mathcal{L}_{\mathrm{pde}}\left(\hat{\boldsymbol{x}}_t; f\right) = \frac{1}{m}\left\|f(\hat{\boldsymbol{x}}_t)\right\|_2^2 
\end{equation}
where $m$ is the total number of spatial--temporal grid points, $f(\hat{\boldsymbol{x}}_t) = f(\boldsymbol{c};\hat{\boldsymbol{x}}_t)$ for static PDEs and $f(\hat{\boldsymbol{x}}_t) = f(\boldsymbol{c}, \tau;\hat{\boldsymbol{x}}_t)$ for dynamic PDEs.

Depending on the sampling strategy, guidance can be incorporated in two distinct ways: deterministic and stochastic. The deterministic approach initiates with a standard Euler step using the learned vector field $u_t^{\theta}$, yielding a predictor state $\tilde{\boldsymbol{x}}_{k+1} = \boldsymbol{x}_k + \Delta t_k u_t^{\theta}(\boldsymbol{x}_k)$. To enforce consistency with the observations and physical constraints, we introduce an observation loss $\mathcal{L}_{\mathrm{obs}}$ and a PDE residual loss $\mathcal{L}_{\mathrm{pde}}$. These terms are used to approximate the conditional score function at time $t$, as shown in Equation~\eqref{eq:guided_score}:
\begin{equation}\label{eq:guided_score}
	\nabla_{\boldsymbol{x}_t} \log p(\boldsymbol{x}_t \mid \boldsymbol{x}_1 ; \boldsymbol{x}_{\mathrm{obs}}, f) \approx \nabla_{\boldsymbol{x}_t} \log p(\boldsymbol{x}_t \mid \boldsymbol{x}_1) - \zeta_{\mathrm{obs}} \nabla_{\boldsymbol{x}_t} \mathcal{L}_{\mathrm{obs}} - \zeta_{\mathrm{pde}} \nabla_{\boldsymbol{x}_t} \mathcal{L}_{\mathrm{pde}},
\end{equation}
where $\zeta_{\mathrm{obs}}$ and $\zeta_{\mathrm{pde}}$ denote the guidance weights. By substituting Equation~\eqref{eq:guided_score} into the vector field formulation (Equation~\eqref{eq:ut_with_score}), we derive a corrected velocity field given by $u_t^{\theta}(\boldsymbol{x}_k) - b_t (\zeta_{\mathrm{obs}} \nabla_{\boldsymbol{x}_t} \mathcal{L}_{\mathrm{obs}} + \zeta_{\mathrm{pde}} \nabla_{\boldsymbol{x}_t} \mathcal{L}_{\mathrm{pde}})$. Evaluating these gradients at the predictor state $\tilde{\boldsymbol{x}}_{k+1}$ leads to the final update rule:
\begin{equation*}
    \boldsymbol{x}_{k+1} = \tilde{\boldsymbol{x}}_{k+1} - \Delta t_k b_{t_k} (\zeta_{\mathrm{obs}} \nabla_{\boldsymbol{x}_k} \mathcal{L}_{\mathrm{obs}}(\tilde{\boldsymbol{x}}_{k+1}) + \zeta_{\mathrm{pde}} \nabla_{\boldsymbol{x}_k} \mathcal{L}_{\mathrm{pde}}(\tilde{\boldsymbol{x}}_{k+1})).
\end{equation*}

At $t = 0$, the definition implies $\alpha_0 = 0$, which induces a singularity where $b_0 \to \infty$. To mitigate this numerical instability in the deterministic sampling algorithm, we employ two strategies. First, an $\varepsilon$-stabilization technique can be employed to regularize the divergence of $b_0$. Specifically, we compute $b_0 = -\frac{\dot{\sigma}_t \sigma_t \alpha_t - \dot{\alpha}_t \sigma_t^2}{\alpha_t + \varepsilon}$, where $\varepsilon > 0$ is a small constant, ensuring that $b_0$ remains a large yet finite positive value. This modification is justified by the fact that at $t = 0$, the sample is dominated by noise and deviates significantly from the true solution, thereby necessitating a strong guidance force for correction. Alternatively, one may discard the guidance term during the initial integration step. In this scheme, the guided deterministic solver commences at $t = \epsilon$ with $\epsilon > 0$, while the trajectory from $t=0$ to $t=\epsilon$ is advanced using a standard unguided Euler step.

Furthermore, to prevent excessively large gradients from causing algorithm failure, we employ gradient clipping. Let $\boldsymbol{g}_k=\zeta_{\mathrm{obs}} \nabla_{\boldsymbol{x}_k} \mathcal{L}_{\mathrm{obs}}(\tilde{\boldsymbol{x}}_{k+1}) + \zeta_{\mathrm{pde}} \nabla_{\boldsymbol{x}_k} \mathcal{L}_{\mathrm{pde}}(\tilde{\boldsymbol{x}}_{k+1})$. Given a large constant $G_c > 0$, we define the clipped gradient as $\boldsymbol{g}_k^{clip} = \boldsymbol{g}_k \cdot \min \{1, G_c / \|\boldsymbol{g}_k\|_2\}$, thus changing the update step of the deterministic solution to $\boldsymbol{x}_{k+1} = \tilde{\boldsymbol{x}}_{k+1} - \Delta t_k b_t \boldsymbol{g}_k^{clip}$.

In the stochastic sampling regime, we adopt the interpolation-based scheme proposed in \citep{singh2024stochastic,lai2025principles}. Specifically, at each integration step $k$, we first estimate the terminal solution at $t=1$ by projecting the current state along the learned velocity field $u_t^{\theta}$, i.e., $\hat{\boldsymbol{x}}_1^{(k)} = \boldsymbol{x}_k + (1 - t_k) u_t^{\theta}(\boldsymbol{x}_k)$. Subsequently, the observation loss $\mathcal{L}_{\mathrm{obs}}$ and the PDE residual $\mathcal{L}_{\mathrm{pde}}$ are evaluated on this estimated terminal state $\hat{\boldsymbol{x}}_1^{(k)}$. The state update is then performed by interpolating between a freshly sampled noise vector $\xi_k \sim \mathcal{N}(\mathbf{0}, I)$ and the guided prediction, formulated as:
\begin{equation*}
    \begin{aligned}
        \tilde{\boldsymbol{x}}_{k+1} &= (1 - t_{k+1}) \xi_k + t_{k+1} \hat{\boldsymbol{x}}_1^{(k)}, \\
        \boldsymbol{x}_{k+1} &= \tilde{\boldsymbol{x}}_{k+1} - c_{\zeta} (1 - t_k) \left(\zeta_{\mathrm{obs}} \nabla_{\boldsymbol{x}_k} \mathcal{L}_{\mathrm{obs}}\left(\hat{\boldsymbol{x}}_1^{(k)}\right) + \zeta_{\mathrm{pde}} \nabla_{\boldsymbol{x}_k} \mathcal{L}_{\mathrm{pde}}\left(\hat{\boldsymbol{x}}_1^{(k)}\right)\right),
    \end{aligned}
\end{equation*}
where $c_{\zeta} > 0$ is a scaling factor for the guidance strength. The choice of $c_{\zeta}$ is discussed in the theoretical analysis in Section~\ref{sec:thm}.

It is worth noting that when $t_{k+1}=1$, the relationship $\tilde{\boldsymbol{x}}_{k+1} = \hat{\boldsymbol{x}}_1^{(k)}$ holds, implying that no additional noise is introduced. Consequently, the final update is equivalent to performing an endpoint prediction followed by a single step of gradient descent. In practice, we can also directly output $\hat{\boldsymbol{x}}_1^{(k)}$ from the final iteration as the result. As it inherently represents the prediction of the final state, its great properties are theoretically guaranteed by \Cref{thm:stoc}.

The deterministic and stochastic samplers can also be combined. Specifically, a time threshold $t^*$ is set: the deterministic optimizer is used when $t < t^*$, followed by the stochastic sampler for $t \ge t^*$. This hybrid procedure is presented in \Cref{alg:fm4pde}. The one-step updates of the stochastic sampler and the deterministic optimizer are provided in \Cref{alg:fm4pde_stoc_sampler} and \Cref{alg:fm4pde_det_optimizer}, respectively.

\begin{algorithm}[H]
    \caption{Guided Flow Matching for PDE-solving}\label{alg:fm4pde}
    \begin{algorithmic}
        \REQUIRE {Pretrained Flow Matching $u_t^{\theta}(\boldsymbol{x})$, Affine conditional flow scheduler $(\alpha_t, \sigma_t)$, Sparse observations $\boldsymbol{x}_{\mathrm{obs}}$, PDE function $f$, Geometric grid parameter $\eta$, Guidance weights $\zeta_{\mathrm{obs}}$ and $\zeta_{\mathrm{pde}}$, Total number of grid points $m$, Number of observed points $n$, Threshold time $t^*$, Number of steps $N$, Step size parameter $c_{\zeta}$.}
        \STATE Initialize $\boldsymbol{x}_0 \sim \mathcal{N}(\mathbf{0}, I)$.
        \FOR{$k \in \{ 0, 1, \ldots, N-1 \}$}
            \IF{$t_k < t^*$}
                \STATE Phase I: Deterministic with strong contraction
                \STATE Use geometric grid: $\Delta t_k = \eta t_k$
                \STATE $\boldsymbol{x}_{k+1} = \text{DeterministicOptimizer}(u_t^{\theta}(\boldsymbol{x}), \boldsymbol{x}_{\mathrm{obs}}, f, \zeta_{\mathrm{obs}}, \zeta_{\mathrm{pde}}, m, n, \Delta t_k)$
            \ELSE 
                \STATE Phase II: Stochastic for exploration
                \STATE Use standard grid: $\Delta t_k$ such as equal step size
                \STATE $\boldsymbol{x}_{k+1} = \text{StochasticSampler}(u_t^{\theta}(\boldsymbol{x}), (\alpha_t, \sigma_t), \boldsymbol{x}_{\mathrm{obs}}, f, \zeta_{\mathrm{obs}}, \zeta_{\mathrm{pde}}, m, n, \Delta t_k, c_{\zeta})$
            \ENDIF
        \ENDFOR
        \RETURN $\hat{\boldsymbol{x}}_1 = \boldsymbol{x}_{N}$
    \end{algorithmic}
\end{algorithm}

\begin{algorithm}[H]
    \caption{Deterministic Guided Optimizer at $t_k$}\label{alg:fm4pde_det_optimizer}
    \begin{algorithmic}[1]
        \REQUIRE Pretrained Flow Matching $u_t^{\theta}(\boldsymbol{x})$, Affine conditional flow scheduler $(\alpha_t, \sigma_t)$, Sparse observations $\boldsymbol{x}_{\mathrm{obs}}$, PDE function $f$, Guidance weights $\zeta_{\mathrm{obs}}$ and $\zeta_{\mathrm{pde}}$, Total number of grid points $m$, Number of observed points $n$, Step size of the ODE solver $\Delta t_k$.

        \IF{$t_k = 0$}
            \STATE $b_{t_k} = 0$ \COMMENT{Without guidance}
        \ELSE
            \STATE $b_{t_k} = - (\dot{\sigma}_{t_k} \sigma_{t_k} \alpha_{t_k} - \dot{\alpha}_{t_k} \sigma_{t_k}^2)/\alpha_{t_k}$ \COMMENT{Compute guidance coefficient}
        \ENDIF
        \STATE $\tilde{\boldsymbol{x}}_{k+1} = \boldsymbol{x}_{k} + \Delta t_k \cdot u_{t_k}^{\theta}(\boldsymbol{x}_k)$ \COMMENT{Predict next state}
        \STATE $\mathcal{L}_{\mathrm{obs}} = \frac{1}{n} \| \mathcal{P}(\tilde{\boldsymbol{x}}_{k+1}) - \boldsymbol{x}_{\mathrm{obs}} \|_2^2$ \COMMENT{\COMMENT{Evaluate the observation loss via Equation~\eqref{eq:loss_obs}}}
        \STATE $\mathcal{L}_{\mathrm{pde}} = \frac{1}{m} \| f(\tilde{\boldsymbol{x}}_{k+1}) \|_2^2$ \COMMENT{Evaluate the PDE loss via Equation~\eqref{eq:loss_pde}}
        \STATE $\boldsymbol{g}_k = \zeta_{\mathrm{obs}} \cdot \nabla_{\boldsymbol{x}_{k}} \mathcal{L}_{\mathrm{obs}} + \zeta_{\mathrm{pde}} \cdot \nabla_{\boldsymbol{x}_{k}} \mathcal{L}_{\mathrm{pde}}$
        \STATE $\boldsymbol{g}_k^{clip} = \boldsymbol{g}_k \cdot \min \{ 1, G_c / \|\boldsymbol{g}_k\| \}$
        \STATE $\boldsymbol{x}_{k+1} = \tilde{\boldsymbol{x}}_{k+1} - b_{t_k} \cdot \boldsymbol{g}_k^{clip} \cdot \Delta t_k$ \COMMENT{Apply guidance}

        \RETURN $\boldsymbol{x}_{k+1}$
    \end{algorithmic}
\end{algorithm}

\begin{algorithm}[H]
    \caption{Stochastic Guided Sampler with Re-initialization at $t_k$}\label{alg:fm4pde_stoc_sampler}
    \begin{algorithmic}[1]
        \REQUIRE {Pretrained Flow Matching $u_t^{\theta}(\boldsymbol{x})$, Sparse observations $\boldsymbol{x}_{\mathrm{obs}}$, PDE function $f$, Guidance weights $\zeta_{\mathrm{obs}}$ and $\zeta_{\mathrm{pde}}$, Total number of grid points $m$, Number of observed points $n$, Step size parameter $c_{\zeta}$.}

        \STATE Sample $\xi_k \sim \mathcal{N}(\mathbf{0},I)$ \COMMENT{Sample fresh noise}
        \STATE $\hat{\boldsymbol{x}}_1^{(k)} = \boldsymbol{x}_k + (1 - t_k) \cdot u_{t_k}^{\theta}(\boldsymbol{x}_k)$ \COMMENT{Predict endpoint}
        \STATE $\tilde{\boldsymbol{x}}_{k+1} = (1 - t_{k+1})\xi_k + t_{k+1} \hat{\boldsymbol{x}}_1^{(k)}$ \COMMENT{Interpolate at $t_{k+1}$ with noise}
        \STATE $\mathcal{L}_{\mathrm{obs}} = \frac{1}{n} \| \mathcal{P}(\hat{\boldsymbol{x}}_1^{(k)}) - \boldsymbol{x}_{\mathrm{obs}} \|_2^2$ \COMMENT{Evaluate the observation loss via Equation~\eqref{eq:loss_obs}}
        \STATE $\mathcal{L}_{\mathrm{pde}} = \frac{1}{m} \|f(\hat{\boldsymbol{x}}_1^{(k)})\|_2^2$ \COMMENT{Evaluate the PDE loss via Equation~\eqref{eq:loss_pde}}
        \STATE $\zeta_k = c_{\zeta} (1 - t_k)$
        \STATE $\boldsymbol{x}_{k+1} = \tilde{\boldsymbol{x}}_{k+1} - \zeta_k (\zeta_{\mathrm{obs}} \cdot  \nabla_{\boldsymbol{x}_k} \mathcal{L}_{\mathrm{obs}} + \zeta_{\mathrm{pde}} \cdot \nabla_{\boldsymbol{x}_k} \mathcal{L}_{\mathrm{pde}})$ \COMMENT{Apply guidance}

        \RETURN $\boldsymbol{x}_{k+1}$
    \end{algorithmic}
\end{algorithm}

\begin{remark}[Time Grid Strategy]\label{rem:grid_strategy}
    As outlined in Algorithm~\ref{alg:fm4pde}, the sampling procedure is divided into two phases, each presenting distinct numerical challenges that necessitate tailored time grid strategies.
    \begin{itemize}
        \item The deterministic phase ($t < t^*$) employs a geometric time grid defined by $\Delta t_k = \eta t_k$. This strategy is critical to address the singularity of the guidance coefficient $b_t = 1/t - 1$, which diverges as $t \to 0$ and introduces stiffness into the system. The geometric spacing ensures that the effective step size remains bounded, i.e., $b_t \Delta t_k = \eta(1-t_k) = O(1)$, thereby preserving numerical stability near the initial time.
        \item In the stochastic phase ($t \ge t^*$), a standard uniform grid (equal step size) is sufficient. For the theoretical analysis presented in Section~\ref{sec:thm}, we require that this standard grid adheres to Assumption~\ref{ass:time_grid}.
    \end{itemize}
\end{remark}

\begin{remark}[Choices of Guidance Weights]\label{rem:guidance_weight}
    The hyperparameters $\zeta_{\mathrm{obs}}$ and $\zeta_{\mathrm{pde}}$ govern the trade-off between fidelity to sparse observations and adherence to global physical laws. Their specific values are problem-dependent, and the relative magnitude of these weights critically determines the balance between the observation and PDE residuals. 
\end{remark}

\section{Theoretical Analysis}\label{sec:thm}

In this section, we prove that our proposed algorithms can provide effective guidance and analyze the error of the guidance loss. We first discuss the error behavior of the bootstrapping loss when results are generated solely by a deterministic optimizer over the entire time horizon. We then focus on stochastic samplers, and finally examine the theoretical properties of hybrid strategies. For notational simplicity, we write the Euclidean norm $\|\cdot\|_2$ as $\|\cdot\|$ in this section.

\subsection{Assumptions}

In order to carry out theoretical analysis, we need some assumptions. Since the assumptions required by different algorithms (deterministic, stochastic, and hybrid) are not identical, we first present a set of general conditions that will be used throughout the paper, and then state the algorithm-specific assumptions immediately after the general ones. Denote $\delta_k := 1 - t_k$ , $\Delta t_k := t_{k+1} - t_k$.

\begin{assumption}[Velocity Field Regularity]\label{ass:velocity}
    The velocity field $u_t^{\theta}: [0, 1] \times \mathbb{R}^d \to \mathbb{R}^d$ is twice continuously differentiable in $\boldsymbol{x}$ and satisfies:
    \begin{enumerate}[label=(\alph*)]
        \item \label{ass:velocity_linear_growth} $\|u_t^{\theta}(\boldsymbol{x})\| \le B_u(1 + \|\boldsymbol{x}\|)$ for all $\boldsymbol{x} \in \mathbb{R}^d$, $t \in [0,1]$,
        \item \label{ass:velocity_grad_bound} $\|\nabla u_t^{\theta}(\boldsymbol{x})\| \le B_J$ for all $\boldsymbol{x}\in \mathbb{R}^d, t \in [0,1]$,
        \item \label{ass:velocity_L_smooth} $\|\nabla u_t^{\theta}(\boldsymbol{x}) - \nabla u_t^{\theta}(\boldsymbol{y})\| \le L_u \norm{\boldsymbol{x} - \boldsymbol{y}}$ for all $\boldsymbol{x}\in \mathbb{R}^d, \boldsymbol{y}\in \mathbb{R}^d, t \in [0,1]$,
        \item \label{ass:velocity_time_smooth} $\|u_t^{\theta}(\boldsymbol{x}) - u_s^{\theta}(\boldsymbol{x})\| \le L_t(1 + \|\boldsymbol{x}\|)|t - s|$ for all $\boldsymbol{x}\in \mathbb{R}^d, t \in [0,1], s \in [0,1]$.
    \end{enumerate}
\end{assumption}

\begin{assumption}[Dissipativity of Velocity Field]\label{ass:dissipative}
    There exist constants $\kappa > 0$, $R_0 \ge 0$, and $C_{drift} \ge 0$ such that for all $\boldsymbol{x} \in \mathbb{R}^d$ with $\|\boldsymbol{x}\| \ge R_0$ and all $t \in [0,1]$:
    \begin{equation*}
        \inner{\boldsymbol{x}}{u_t^{\theta}(\boldsymbol{x})} \le -\kappa \|\boldsymbol{x}\|^2 + C_{drift}.
    \end{equation*}
\end{assumption}

\begin{remark}[Global Dissipativity Form]\label{rem:global_dissip}
    Assumption~\ref{ass:dissipative} implies a global bound. For $\|\boldsymbol{x}\| < R_0$, by Assumption~\ref{ass:velocity}\ref{ass:velocity_linear_growth}:
    \begin{equation*}
        \inner{\boldsymbol{x}}{u_t^\theta(\boldsymbol{x})} \le \|\boldsymbol{x}\| \cdot B_u(1+\|\boldsymbol{x}\|) \le R_0 B_u(1+R_0).
    \end{equation*}
    Thus, for \emph{all} $\boldsymbol{x} \in \mathbb{R}^d$ and $t \in [0,1]$:
    \begin{equation}\label{eq:global_dissip}
        \inner{\boldsymbol{x}}{u_t^\theta(\boldsymbol{x})} \le -\kappa\|\boldsymbol{x}\|^2 + C'_{drift},
    \end{equation}
    where $C'_{drift} := \kappa R_0^2 + R_0 B_u(1+R_0) + C_{drift}$.
\end{remark}

Regarding the guiding loss $\mathcal{L}(\boldsymbol{x}) = \zeta_{\mathrm{obs}} \mathcal{L}_{\mathrm{obs}} (\boldsymbol{x}) + \zeta_{\mathrm{pde}} \mathcal{L}_{\mathrm{pde}} (\boldsymbol{x})$, we make the following assumptions.

\begin{assumption}[Loss Function Properties]\label{ass:loss}
    The guidance loss $\mathcal{L}:\mathbb{R}^d \to [0,\infty)$ satisfies:
    \begin{enumerate}[label=(\alph*)]
        \item \label{ass:L_Lsmooth} $\mathcal L$ is continuously differentiable and its gradient is $L_{\mathcal L}$-Lipschitz which yields $\mathcal{L}$ is $L_{\mathcal L}$-smooth, i.e., there exists a constant $L_{\mathcal{L}} > 0$, such that $\forall~\boldsymbol{x},\boldsymbol{y}\in\mathbb{R}^d$, $\|\nabla \mathcal{L}(\boldsymbol{x})-\nabla \mathcal{L}(\boldsymbol{y})\| \le L_{\mathcal L}\|\boldsymbol{x}-\boldsymbol{y}\|$.
        \item \label{ass:PL} $\norm{\nabla \mathcal{L}(\boldsymbol{x})}^2 \ge 2\mu \mathcal{L}(\boldsymbol{x})$ for some $\mu > 0$, which is known as the Polyak-{\L}ojasiewicz (PL) condition \citep{karimi2016linear}.
        \item \label{ass:x_star} There exists $\boldsymbol{x}^* \in \mathbb{R}^d$ with $\mathcal{L}(\boldsymbol{x}^*) = 0$ and $\norm{\boldsymbol{x}^*} < \infty$.
    \end{enumerate}
\end{assumption}

\begin{remark}[Derived Gradient Growth Bound]\label{rem:gradient_growth}
    Assumptions~\ref{ass:loss}\ref{ass:L_Lsmooth} and \ref{ass:x_star} together imply a gradient growth bound. First, under Assumption~\ref{ass:loss}\ref{ass:L_Lsmooth} and \ref{ass:x_star}, according to the fundamental theorem of calculus and the properties of smooth functions, we can obtain $\nabla\mathcal{L}(\boldsymbol{x}^*) = 0$. Furthermore,
    \begin{equation*}
        \norm{\nabla\mathcal{L}(\boldsymbol{x})} = \norm{\nabla\mathcal{L}(\boldsymbol{x}) - \nabla\mathcal{L}(\boldsymbol{x}^*)} \le L_{\mathcal{L}}\norm{\boldsymbol{x} - \boldsymbol{x}^*} \le L_{\mathcal{L}}(\|\boldsymbol{x}\| + \norm{\boldsymbol{x}^*}).
    \end{equation*}
    Thus, $\norm{\nabla\mathcal{L}(\boldsymbol{x})} \le G_0(1 + \|\boldsymbol{x}\|)$ holds with $G_0 := L_{\mathcal{L}}\max\{1, \norm{\boldsymbol{x}^*}\}$.
\end{remark}

There are some additional assumptions for the deterministic optimizer. To circumvent the singularity at $t=0$, where $b_0 \to \infty$, the time grid is required to start at a strictly positive time $t_0 = \epsilon > 0$, such that $\epsilon = t_0 < t_1 < \cdots < t_N$ with $\epsilon$.

\begin{assumption}
    \label{ass:interaction_coercivity}
    The loss function is assumed to satisfy the following growth condition and coercivity property:
    \begin{itemize}
        \item[(a)] \label{ass:interaction} Velocity-Loss Interaction: The velocity field and loss gradient satisfy the quadratic growth condition:
        \begin{equation}
            |\inner{\nabla \mathcal{L}(\boldsymbol{x})}{u_t(\boldsymbol{x})}| \le \beta_1 \norm{\nabla \mathcal{L}(\boldsymbol{x})}^2 + \beta_2
        \end{equation}
        for some constants $\beta_1, \beta_2 \ge 0$ with $\beta_1 < 1$.
        \item[(b)] \label{ass:coercivity} Loss Coercivity: The loss function satisfies the strong coercivity condition: there exist constants $c_{coer} > 0$ and $R_0 \ge 0$ such that for all $\|\boldsymbol{x}\| \ge R_0$:
    \begin{equation}
        \inner{\boldsymbol{x}}{\nabla \mathcal{L}(\boldsymbol{x})} \ge c_{coer} \|\boldsymbol{x}\|^2.
    \end{equation}
    \end{itemize}
\end{assumption}

\begin{assumption}[Initial Loss Bound]
    \label{ass:initial}
    The initial loss satisfies one of the following:
    \begin{enumerate}[label=(\alph*)]
        \item Uniform bound: $\mathcal{L}(\boldsymbol{x}_\epsilon) \le M_0$ for some constant $M_0 > 0$;
        \item Polynomial growth: $\mathbb{E} [\mathcal{L}(\boldsymbol{x}_\epsilon)] \le M_0 \cdot \epsilon^{-\alpha}$ for some $\alpha < 2\mu$.
    \end{enumerate}
\end{assumption}

For stochastic sampler, Algorithm \ref{alg:fm4pde_stoc_sampler}, we work on the time grid $t_0 < t_1 < \cdots < t_N \le 1 - \delta_{\min}$, and need the following assumptions.

\begin{assumption}[Well-Conditioning Near Terminal Time]\label{ass:jacobian}
    There exist $\epsilon_0 > 0$ and $\lambda_{\min} > 0$ such that for all $t, s \in [1-\epsilon_0, 1]$ and $\boldsymbol{x} \in \mathbb{R}^d$, the symmetrized product:
    \begin{equation*}
        M_{t,s}^{sym}(\boldsymbol{x}) := \tfrac{1}{2}\left[J_t(\boldsymbol{x})J_s(\boldsymbol{x})^\top + J_s(\boldsymbol{x})J_t(\boldsymbol{x})^\top\right] \succeq \lambda_{\min} I.
    \end{equation*}
\end{assumption}



\begin{assumption}[Time Grid Refinement]
    \label{ass:time_grid}
    The time grid $t_0 < t_1 < \cdots < t_N \le 1 - \delta_{\min}$ satisfies:
    \begin{enumerate}[label=(\alph*)]
        \item \label{ass:time_grid_bound} Bounded step size: $\max_k \Delta t_k \le \bar{\Delta}$ for some $0 < \bar{\Delta} \leq \epsilon_0/2$,
        \item \label{ass:time_grid_refinement} For all $k$ with $t_k \ge 1 - \epsilon_0$, we have $\Delta t_k \le c_\Delta \delta_k$ for some $c_\Delta \in (0,1/2)$,
        \item \label{ass:time_grid_positive_floor} Positive noise floor: $t_N \le 1 - \delta_{\min}$ for some $\delta_{\min} > 0$.
    \end{enumerate}
\end{assumption}


\subsection{Deterministic Sampler Analysis}

In this subsection, we first investigate the theoretical properties of the deterministic scheme. We first denote some notations as follows.
\begin{itemize}
    \item $\Phi_{t_k}(\boldsymbol{x}) := \boldsymbol{x}_k + \Delta t_k u_{t_k}^{\theta}(\boldsymbol{x}_k)$,
    \item $J_{t_k}(\boldsymbol{x}_k) := \nabla_{\boldsymbol{x}_k} \Phi_{t_k}(\boldsymbol{x}) = I + \Delta t_k \nabla u_{t_k}^{\theta}(\boldsymbol{x}_k)$,
    \item $\bar{\boldsymbol{g}}_k := \nabla \mathcal{L}(\tilde{\boldsymbol{x}}_{k+1})$ (endpoint gradient),
    \item $\boldsymbol{g}_k := J_{t_k}(\boldsymbol{x})^\top \bar{\boldsymbol{g}}_k(\boldsymbol{x})$ (backpropagated gradient),
    \item $\boldsymbol{g}_k^{clip} := \boldsymbol{g}_k \cdot \min \{ 1, G_c / \| \boldsymbol{g}_k \|\}$ with $G_c > 0$ (the clipped gradient),
    \item $V_k := \mathbb{E}[\mathcal{L}(\boldsymbol{x}_k)] = \mathcal{L}(\boldsymbol{x}_k)$ (expected loss at step $k$).
\end{itemize}

The deterministic optimizer, Algorithm \ref{alg:fm4pde_det_optimizer}, executes the following update steps:
\begin{equation*}
    \begin{aligned}
        \tilde{\boldsymbol{x}}_{k+1} &= \boldsymbol{x}_k + \Delta t_{k} \cdot u_{t_k}^{\theta}(\boldsymbol{x}_k), \\
        \boldsymbol{x}_{k+1} &= \tilde{\boldsymbol{x}}_{k+1} - \Delta t_k \cdot b_{t_k} \cdot \boldsymbol{g}_k^{clip}.
    \end{aligned}
\end{equation*}

Combining these, the update can be written as:
\begin{equation*}
    \boldsymbol{x}_{k+1} = \boldsymbol{x}_k + \Delta t_{k} \cdot u_{t_k}^{\theta}(\boldsymbol{x}_k) - \Delta t_k \cdot b_{t_k} \cdot \boldsymbol{g}_k^{clip}.
\end{equation*}
where $\Delta t_k = t_{k+1} - t_{k} = \eta t_{k}$, $k \in \{0, 1, \ldots, N - 1\}$. Then, by \Cref{ass:velocity,ass:loss}, we have
\begin{equation*}
    \begin{aligned}
        \| \boldsymbol{g}_k \| &= \| \nabla_{\boldsymbol{x}_k} \mathcal{L}(\tilde{\boldsymbol{x}}_{k+1}) \| = \| (I + \Delta t_k \nabla u_{t_k}^{\theta}(\boldsymbol{x}_k))^{\top} \nabla \mathcal{L}(\tilde{\boldsymbol{x}}_{k+1}) \| \\
        &\leq \| I + \Delta t_k \nabla u_{t_k}^{\theta}(\boldsymbol{x}_k) \| \|\nabla \mathcal{L}(\tilde{\boldsymbol{x}}_{k+1}) \| \\
        &\leq (1 + \Delta t_k B_J) G_0 (1 + \|\tilde{\boldsymbol{x}}_{k+1}\|) \\
        &\leq (1 + \Delta t_k B_J) G_0 (1 + \|\boldsymbol{x}_{k}\| + \Delta t_k \|u_{t_k}^{\theta}(\boldsymbol{x}_k)\|) \\
        &\leq (1 + \Delta t_k B_J) G_0 (1 + \|\boldsymbol{x}_{k}\| + \Delta t_k B_u (1 + \| \boldsymbol{x}_{k} \|)) \\
        &\leq (1 + \Delta t_k B_J) (1 + \Delta t_k B_u) G_0 (1 + \| \boldsymbol{x}_{k} \|) \\
        &\leq \tilde{G}_0 (1 + \|\boldsymbol{x}_k\|),
    \end{aligned}
\end{equation*}
where $B_g := \max \{ B_J, B_u \}$ and $\tilde{G}_0 = G_0 (1 + B_g)^2$. 
\begin{lemma}\label{lem:det_discrete_traj}
    Under \Cref{ass:velocity,ass:dissipative,ass:loss,ass:interaction_coercivity}, for all $\eta$ satisfying $\eta < \kappa_{\mathrm{eff}}/2C_{\beta}$ where $C_{\beta}$ is defined in the proof, the discrete trajectory satisfies $\|\boldsymbol{x}_k\| \leq R_{disc}$ for all k, where
    \begin{equation*}
        R_{disc} := \max\{R, \|\boldsymbol{x}_0\|\}, \quad R^2 := \max \left\{2 C_d \eta^2 + 2 (1 + C_d \eta^2) R_0^2, \dfrac{2 C_{drift} + C_{\beta}}{\kappa_{\mathrm{eff}}} + 1 \right\}.
    \end{equation*}
    Here $\kappa_{\mathrm{eff}} = \min\{\kappa, c_{coer}\}$, and $C_{\beta}, C_d, C_g$ are constant dependent on $B_u, \tilde{G}_0, B_J, L_{\mathcal{L}}$ (defined in the proof).
\end{lemma}

\begin{proof}
    We prove the lemma using the method of invariant sets. Although Algorithm~\ref{alg:fm4pde_det_optimizer} uses $\boldsymbol{g}_k^{clip}$, we prove the bound with the unclipped $\boldsymbol{g}_k$. This is valid because:
    \begin{enumerate}
        \item[(i)] All upper bounds on $\|\boldsymbol{g}_k\|$ remain valid for $\boldsymbol{g}_k^{clip}$ since $\|\boldsymbol{g}_k^{clip}\| \le \|\boldsymbol{g}_k\|$, 
        \item[(ii)] The coercivity inner product satisfies $\inner{\boldsymbol{x}_k}{\boldsymbol{g}_k^{clip}} = \min\{1, G_c/\norm{\boldsymbol{g}_k}\} \cdot \inner{\boldsymbol{x}_k}{\boldsymbol{g}_k}$, which preserves the sign of $\inner{\boldsymbol{x}_k}{\boldsymbol{g}_k}$. 
    \end{enumerate}
    Thus the invariant set proved below for the unclipped update is also invariant for the clipped update, and $R_{disc}$ is independent of $G_c$.

    Denoting
    \begin{equation*}
        \boldsymbol{d}_k := u_{t_k}^{\theta}(\boldsymbol{x}_k) - b_{t_k} \boldsymbol{g}_k,
    \end{equation*}
    we have
    \begin{equation}\label{eq:update_squre}
        \|\boldsymbol{x}_{k+1}\|^2 = \| \boldsymbol{x}_k + \Delta t_{k} \boldsymbol{d}_k \|^2 = \|\boldsymbol{x}_k\|^2 + 2 \Delta t_{k} \langle \boldsymbol{x}_k, \boldsymbol{d}_k \rangle + \Delta t_{k}^2 \| \boldsymbol{d}_k \|^2.
    \end{equation}

    Note that,
    \begin{equation*}
        \begin{aligned}
            \|\boldsymbol{d}_{k}\|^2 &= \|u_{t_k}^{\theta}(\boldsymbol{x}_k) - b_{t_k} \boldsymbol{g}_k\|^2 \leq 2 \| u_{t_k}^{\theta}(\boldsymbol{x}_k) \|^2 + 2 b_{t_k}^2 \| \boldsymbol{g}_k \|^2 \\
                                       &\leq 2 B_u^2 (1 + \|\boldsymbol{x}_k\|)^2 + 2 b_{t_k}^2 \tilde{G}_0^2 (1 + \|\boldsymbol{x}_k\|)^2 \\
                                       &\leq C_d (1 + b_{t_k}^2) (1 + \|\boldsymbol{x}_k\|^2) \\
                                       &\leq C_d (1 + b_{t_k})^2 (1 + \|\boldsymbol{x}_k\|^2),
        \end{aligned}
    \end{equation*}
    where $C_d := \max \{ 4 B_u^2, 4 \tilde{G}_0^2 \}$. Under the geometric step size $\Delta t_k = \eta t_k$, we have
    \begin{equation*}
        \Delta t_k \leq \eta, \quad b_{t_k} \Delta t_k = \eta (1 - t_k) \leq \eta, \quad \Delta t_k (1 + b_{t_k}) = \eta.
    \end{equation*}

    Thus,
    \begin{equation*}
        \Delta t_k^2 \|\boldsymbol{d}_k\|^2 \leq C_d \Delta t_k^2 (1 + b_{t_k})^2 (1 + \|\boldsymbol{x}_k\|^2) \leq C_d \eta^2 + C_d \eta^2 \|\boldsymbol{x}_k\|^2.
    \end{equation*}

    If $\|\boldsymbol{x}_k\| \leq R_0$, then
    \begin{equation}\label{eq:det_bound_in}
        \|\boldsymbol{x}_{k+1}\|^2 \leq 2 \|\boldsymbol{x}_k\|^2 + 2 \Delta t_k^2 \|\boldsymbol{d}_k\|^2 \leq 2 C_d \eta^2 + 2 (1 + C_d \eta^2) R_0^2 =: R_{\mathrm{in}}^2.
    \end{equation}

    Furthermore, if $\|\boldsymbol{x}_k\| > R_0$, for the second term in Equation~\eqref{eq:update_squre}, we have
    \begin{equation*}
        \begin{aligned}
            \langle \boldsymbol{x}_k, \boldsymbol{d}_k \rangle
            &= \langle \boldsymbol{x}_k, u_{t_k}^{\theta}(\boldsymbol{x}_k) \rangle - \langle \boldsymbol{x}_k, b_{t_k} \boldsymbol{g}_k \rangle \\
            &= \langle \boldsymbol{x}_k, u_{t_k}^{\theta}(\boldsymbol{x}_k) \rangle - b_{t_k}\langle \boldsymbol{x}_k, \boldsymbol{g}_k - \nabla \mathcal{L}(\boldsymbol{x}_k) \rangle - b_{t_k}\langle \boldsymbol{x}_k, \nabla \mathcal{L}(\boldsymbol{x}_k) \rangle \\
            &\leq \langle \boldsymbol{x}_k, u_{t_k}^{\theta}(\boldsymbol{x}_k) \rangle - b_{t_k} \langle \boldsymbol{x}_k, \nabla \mathcal{L}(\boldsymbol{x}_k) \rangle + b_{t_k} \|\boldsymbol{x}_k\| \|\boldsymbol{g}_k - \nabla \mathcal{L}(\boldsymbol{x}_k)\|.
        \end{aligned}
    \end{equation*}

    For $\boldsymbol{g}_k - \nabla \mathcal{L}(\boldsymbol{x}) = \Delta t_k \nabla u_t^{\theta}(\boldsymbol{x}_k) \nabla \mathcal{L} (\tilde{\boldsymbol{x}}_{k+1}) + \left( \nabla \mathcal{L} (\tilde{\boldsymbol{x}}_{k+1}) - \nabla\mathcal{L}(\boldsymbol{x}_k) \right)$, we have
    \begin{equation*}
        \|\Delta t_k \nabla u_t^{\theta}(\boldsymbol{x}_k) \nabla \mathcal{L} (\tilde{\boldsymbol{x}}_{k+1})\| \leq \Delta t_k \|\nabla u_t^{\theta}(\boldsymbol{x}_k)\| \|\nabla\mathcal{L}(\tilde{\boldsymbol{x}}_{k+1})\| \leq \Delta t_k B_J \|\nabla\mathcal{L}(\tilde{\boldsymbol{x}}_{k+1})\|,
    \end{equation*}
    and
    \begin{equation*}
        \|\nabla\mathcal{L}(\tilde{\boldsymbol{x}}_{k+1})-\nabla\mathcal{L}(\boldsymbol{x}_k)\| \le L_{\mathcal{L}}\|\tilde{\boldsymbol{x}}_{k+1}-\boldsymbol{x}_k\| = L_{\mathcal{L}}\,\Delta t_k\,\|u_{t_k}^\theta(\boldsymbol{x}_k)\|.
    \end{equation*}
    
    Under the growth bounds $\|\nabla\mathcal{L}(\boldsymbol{x})\|\le G_0(1+\|\boldsymbol{x}\|)$ and $\|u_{t}^\theta(\boldsymbol{x})\|\le B_u(1+\|\boldsymbol{x}\|)$, combining above yields
    \begin{equation*}
        \begin{aligned}
            \|\boldsymbol{g}_k-\nabla\mathcal{L}(\boldsymbol{x}_k)\| &\leq \Delta t_k\,B_J\,\|\nabla\mathcal{L}(\tilde{\boldsymbol{x}}_{k+1})\| + L_{\mathcal{L}}\,\Delta t_k\,\|u_{t_k}^\theta(\boldsymbol{x}_k)\| \\
            &\leq \Delta t_k B_J G_0 (1 + \|\tilde{\boldsymbol{x}}_{k+1}\|) + L_{\mathcal{L}} \Delta t_k B_u(1 + \|\boldsymbol{x}_k\|) \\
            &\leq \Delta t_k B_J G_0 (1 + \|\boldsymbol{x}_k\| + \Delta t_k \|u_{t_k}^\theta(\boldsymbol{x}_k)\|) + L_{\mathcal{L}} \Delta t_k B_u(1 + \|\boldsymbol{x}_k\|) \\
            &\leq \Delta t_k [G_0 B_J (1 + B_u) + L_{\mathcal{L}} B_u] (1 + \|\boldsymbol{x}_k\|) \\
            &=: C_g \Delta t_k (1 + \|\boldsymbol{x}_k\|).
        \end{aligned}
    \end{equation*}

    Thus, with the dissipativity $\langle \boldsymbol{x}, u_t^{\theta}(\boldsymbol{x}) \rangle \leq -\kappa \|\boldsymbol{x}\|^2 + C_{drift} $ and $\langle \boldsymbol{x}, \nabla \mathcal{L} (\boldsymbol{x}) \rangle \geq c_{coer} \|\boldsymbol{x}\|^2$, we have when $\|\boldsymbol{x}_k\| > R_0$,
    \begin{equation*}
        \begin{aligned}
            \langle \boldsymbol{x}_k, \boldsymbol{d}_k \rangle 
            &\leq \langle \boldsymbol{x}_k, u_{t_k}^{\theta}(\boldsymbol{x}_k) \rangle - b_{t_k} \langle \boldsymbol{x}_k, \nabla \mathcal{L}(\boldsymbol{x}_k) \rangle + b_{t_k} \|\boldsymbol{x}_k\| \|\boldsymbol{g}_k - \nabla \mathcal{L}(\boldsymbol{x}_k)\| \\
            &\leq -\kappa \|\boldsymbol{x}_k\|^2 + C_{drift} - c_{coer} b_{t_k} \|\boldsymbol{x}_k\|^2 + b_{t_k} \Delta t_k C_g (\| \boldsymbol{x}_k \| + \|\boldsymbol{x}_k\|^2).
        \end{aligned}
    \end{equation*}

    Since $r^2 - r + 1 \geq 0$, which implies $r + r^2 \leq 1 + 2r^2 \leq 2 (1 + r^2)$, we have
    \begin{equation*}
        \begin{aligned}
            \langle \boldsymbol{x}_k, \boldsymbol{d}_k \rangle
            &\leq -\kappa \|\boldsymbol{x}_k\|^2 + C_{drift} - c_{coer} b_{t_k} \|\boldsymbol{x}_k\|^2 + 2 b_{t_k} \Delta t_k C_g (1 + \|\boldsymbol{x}_k\|^2) \\
            &= -( \kappa + c_{coer}b_{t_k} ) \|\boldsymbol{x}_k\|^2 + 2 b_{t_k} \Delta t_k C_g \|\boldsymbol{x}_k\|^2 + C_{drift}  + 2 b_{t_k} \Delta t_k C_g.\\
        \end{aligned}
    \end{equation*}

    Thus,
    \begin{equation*}
        \begin{aligned}
            \langle \boldsymbol{x}_k, \boldsymbol{d}_k \rangle
            &\leq -( \kappa + c_{coer}b_{t_k} ) \|\boldsymbol{x}_k\|^2 + 2 b_{t_k} \Delta t_k C_g \|\boldsymbol{x}_k\|^2 + C_{drift}  + 2 b_{t_k} \Delta t_k C_g \\
            &\leq -( \kappa + c_{coer}b_{t_k} ) \|\boldsymbol{x}_k\|^2 + 2 \eta C_g \|\boldsymbol{x}_k\|^2 + C_{drift}  + 2 \eta C_g,
        \end{aligned}
    \end{equation*}
    and
    \begin{equation*}
        \begin{aligned}
            2 \Delta t_k \langle \boldsymbol{x}_k, \boldsymbol{d}_k \rangle
            &\leq -2 \Delta t_k ( \kappa + c_{coer}b_{t_k} ) \|\boldsymbol{x}_k\|^2 + 4 \Delta t_k \eta C_g \|\boldsymbol{x}_k\|^2 + 2 \Delta t_k C_{drift} + 4 \Delta t_k \eta C_g \\
            &\leq -2 \Delta t_k ( \kappa + c_{coer} b_{t_k} ) \|\boldsymbol{x}_k\|^2 + 4 C_g \eta^2 \|\boldsymbol{x}_k\|^2 + 2 \eta C_{drift} + 4 \eta^2 C_g.
        \end{aligned}
    \end{equation*}

    Thus,
    \begin{equation*}
        \begin{aligned}
            \|\boldsymbol{x}_{k+1}\|^2
            &\leq \|\boldsymbol{x}_{k}\|^2 - 2 \Delta t_k ( \kappa + c_{coer} b_{t_k} ) \|\boldsymbol{x}_k\|^2 + 4 C_g \eta^2 \|\boldsymbol{x}_k\|^2 \\
            &\quad + 2 \eta C_{drift} + 4 \eta^2 C_g + C_d \eta^2 + C_d \eta^2 \|\boldsymbol{x}_k\|^2 \\
            &\leq \left[ 1 - 2 \Delta t_k ( \kappa + c_{coer} b_{t_k} ) + (4 C_g + C_d) \eta^2 \right] \|\boldsymbol{x}_k\|^2 \\
            &\quad + 2 \eta C_{drift} + (4 C_g + C_d) \eta^2 \\
            &=: (1 - \alpha_k + \beta) \|\boldsymbol{x}_k\|^2 + \gamma.
        \end{aligned}
    \end{equation*}

    For the contraction rate, we compute:
    \begin{equation*}
        \begin{aligned}
            \alpha_k
            &= 2\Delta t_k (\kappa + b_{t_k} c_{coer}) = 2\eta t_k \kappa + 2\eta (1 - t_k) c_{coer} \\
            &\geq 2\eta \cdot \min\{\kappa, c_{coer}\} = 2\eta \kappa_{\mathrm{eff}} =: \alpha_{min},
        \end{aligned}
    \end{equation*}
    where $\kappa_{\mathrm{eff}} := \min \{ \kappa, c_{coer}\}$, and the inequality uses the fact that $t\kappa + (1-t)c_{coer} \ge \min\{\kappa, c_{coer}\}$ for all $t \in [0,1]$. Thus,
    \begin{equation*}
        \|\boldsymbol{x}_{k+1}\|^2 \leq (1 - \alpha_{min} + \beta ) \|\boldsymbol{x}_k\|^2 + \gamma.
    \end{equation*}

    Define $C_{\beta} := 4 C_g + C_d$, by $\eta < \frac{\kappa_{\mathrm{eff}}}{2 C_{\beta}}$, we have $\alpha_{min} - \beta = 2\eta \kappa_{\mathrm{eff}} - C_{\beta} \eta^2 > \eta \kappa_{\mathrm{eff}} > 0$. Note that
    \begin{equation}\label{eq:det_bound_out}
        \begin{aligned}
            \frac{\gamma}{\alpha_{min} - \beta} = \frac{2\eta C_{drift} + C_{\beta} \eta^2}{2\eta \kappa_{\mathrm{eff}} - C_{\beta} \eta^2} < \dfrac{2 \eta C_{drift} + C_{\beta} \eta^2}{\eta \kappa_{\mathrm{eff}}} < \dfrac{2 C_{drift} + C_{\beta}}{\kappa_{\mathrm{eff}}} + 1 := R_{\mathrm{out}}^2.
        \end{aligned}
    \end{equation}

    Defining the invariant set $R^2 := \max \{ R_{\mathrm{in}}^2, R_{\mathrm{out}}^2 \}$.
    \begin{itemize}
        \item[(i)] If $\|\boldsymbol{x}_k\|^2 > R^2$ and $\|\boldsymbol{x}_k\|^2 \geq R_0^2$: Coercivity applies, then
        \begin{equation*}
            \|\boldsymbol{x}_{k+1}\|^2 \leq (1 - \alpha_{min} + \beta ) \|\boldsymbol{x}_k\|^2 + \gamma \leq \|\boldsymbol{x}_k\|^2,
        \end{equation*}
        since for $\|\boldsymbol{x}_k\|^2 > R^2 \geq R_{\mathrm{out}}^2$, $(\alpha_{min} - \beta) W_k > (\alpha_{min} - \beta) R_{\mathrm{out}}^2 > \gamma$.
        \item[(ii)] If $\|\boldsymbol{x}_k\|^2 < R_0^2$: By Equation~\eqref{eq:det_bound_in}, $\|\boldsymbol{x}_{k+1}\|^2 \leq R_{\mathrm{in}}^2 \leq R^2$.
        \item[(iii)] If $R_0^2 < \|\boldsymbol{x}_k\|^2 \leq R^2$: Coercivity applies, since $R^2 > R_{\mathrm{out}}^2$:
        \begin{equation*}
            \begin{aligned}
                \|\boldsymbol{x}_{k+1}\|^2
                &\leq (1 - \alpha_{min} + \beta ) R^2 + \gamma \leq R^2 - \eta \kappa_{\mathrm{eff}} R^2 + \gamma \\
                &\leq R^2 - \eta (2 C_{drift} + C_{\beta} + \kappa_{\mathrm{eff}}) + 2 \eta C_{drift} + C_{\beta} \eta^2 \\
                &=R^2 - \eta \kappa_{\mathrm{eff}} + C_{\beta} \eta (1 - \eta) < R^2.
            \end{aligned}
        \end{equation*}
    \end{itemize}

    Thus, starting from any $\|\boldsymbol{x}_0\|^2$, the sequence $\{\|\boldsymbol{x}_k\|^2\}$ eventually enters $[0, R^2]$ and remains there. Hence $\|\boldsymbol{x}_k\| \le R_{disc}$ for all $k$, where $R_{disc} := \max\{R, \|\boldsymbol{x}_0\|\}$.
\end{proof}

\begin{definition}[Admissible Step Size: Geometric Grid]
    \label{def:det_step_size}
    A step size sequence $\{\Delta t_k\}_{k=0}^{N-1}$ is admissible for Algorithm~\ref{alg:fm4pde_det_optimizer} if it satisfies:
    \begin{enumerate}[label=(\roman*)]
        \item Geometric grid structure: $\Delta t_k = \eta t_k$ for a fixed constant $\eta > 0$;
        \item Step size bound: $\eta < \eta_{max}$ where
        \begin{equation}
            \eta_{max} := \min\left\{ \frac{1}{2\mu}, \frac{\kappa_{\mathrm{eff}}}{2C_{\beta}}, 1 \right\}
        \end{equation}
        ensures both Phase A stability ($(1 - \rho_k) > 0$) and trajectory boundedness (Lemma~\ref{lem:det_discrete_traj}). Here $\kappa_{\mathrm{eff}} = \min\{\kappa, c_{coer}\}$ and $C_{\beta}$ is the constant from the trajectory bound analysis;
        \item Global upper bound: $\Delta t_k \le \Delta_{\max} := \frac{1}{2(B_J + L_{\mathcal{L}}\beta_1)}$ for consistency everywhere.
    \end{enumerate}
\end{definition}

The geometric grid achieves this:
\begin{equation*}
    b_{t_k} \cdot \Delta t_k = \left(\frac{1}{t_k} - 1\right) \cdot \eta t_k = \eta(1 - t_k) \le \eta = O(1).
\end{equation*}

Under this grid, the step size bound $\eta < \eta_{max}$ is independent of $t_k$ and $\epsilon$; and the number of steps from $\epsilon$ to $t_*$ is $N_A = \frac{\log(t_*/\epsilon)}{\log(1+\eta)} = O(\log(1/\epsilon))$.

\begin{lemma}\label{lem:det_one_step}
    Under \Cref{ass:velocity,ass:loss,ass:interaction_coercivity}, one step of Algorithm~\ref{alg:fm4pde_det_optimizer} satisfies the following bounds. Denote $\rho_k = 2 \mu \Delta t_k (b_{t_k} - \beta_1)$, under the geometric grid:
    \begin{enumerate}
        \item[(a)] Phase A ($b_{t_k} > \beta_1$): Using the PL condition,
            \begin{equation*}
                V_{k+1} \leq (1 - \rho_k) V_k + \tilde{C},
            \end{equation*}
        \item[(b)] Phase B ($b_{t_k} \leq \beta_1$): Using gradient growth bounds,
            \begin{equation*}
                V_{k+1} \leq V_k + \tilde{C}^{\prime},
            \end{equation*}
    \end{enumerate}
    where $\tilde{C}$ and $\tilde{C}^{\prime}$ are defined in the proof, depending on $\beta_1, \beta_2, \eta, G_0, R_{disc}, L_{\mathcal{L}}$ and $C_d$.
\end{lemma}

\begin{proof}
    By $L_{\mathcal{L}}$-smoothness,
    \begin{equation*}
        V_{k+1} = \mathcal{L}(\boldsymbol{x}_{k+1}) \leq \mathcal{L}(\boldsymbol{x}_k) + \langle \nabla \mathcal{L}(\boldsymbol{x}_k), \boldsymbol{x}_{k+1} - \boldsymbol{x}_k \rangle + \dfrac{L_{\mathcal{L}}}{2}\|\boldsymbol{x}_{k+1} - \boldsymbol{x}_k\|^2.
    \end{equation*}

    Since $\boldsymbol{x}_{k+1} - \boldsymbol{x}_k = \Delta t_k (u_{t_k}^{\theta}(\boldsymbol{x}_k) - b_{t_k} \boldsymbol{g}_k) = \Delta t_k \boldsymbol{d}_k$, we have 
    \begin{equation*}
        \dfrac{L_{\mathcal{L}}}{2} \|\boldsymbol{x}_{k+1} - \boldsymbol{x}_k\|^2 \leq \dfrac{L_{\mathcal{L}}}{2} \cdot  C_d (1 + \|\boldsymbol{x}_k\|^2) \eta^2 \leq \dfrac{L_{\mathcal{L}} C_d}{2} (1 + R_{disc}^2) \eta^2.
    \end{equation*}

    Moreover, for $\langle \nabla \mathcal{L}(\boldsymbol{x}_k), \boldsymbol{x}_{k+1} - \boldsymbol{x}_k \rangle$,
    \begin{equation*}
        \langle \nabla \mathcal{L}(\boldsymbol{x}_k), \boldsymbol{x}_{k+1} - \boldsymbol{x}_k \rangle = \Delta t_k \langle\nabla \mathcal{L}(\boldsymbol{x}_k), u_{t_k}^{\theta}(\boldsymbol{x}_k) \rangle - \Delta t_k b_{t_k} \langle\nabla \mathcal{L}(\boldsymbol{x}_k), \boldsymbol{g}_k \rangle.
    \end{equation*}

    The second term satisfies
    \begin{equation*}
        \begin{aligned}
            -\langle\nabla \mathcal{L}(\boldsymbol{x}_k), \boldsymbol{g}_k \rangle
            &= -\langle \nabla \mathcal{L}(\boldsymbol{x}_k), \boldsymbol{g}_k - \nabla \mathcal{L}(\boldsymbol{x}_k) \rangle - \|\nabla \mathcal{L}(\boldsymbol{x}_k)\|^2 \\
            &\leq \|\nabla \mathcal{L}(\boldsymbol{x}_k)\| \|\boldsymbol{g}_k - \nabla \mathcal{L}(\boldsymbol{x}_k)\| - \|\nabla \mathcal{L}(\boldsymbol{x}_k)\|^2 \\
            &\leq -\|\nabla \mathcal{L}(\boldsymbol{x}_k)\|^2 + G_0 C_g \eta (1 + R_{disc})^2.
        \end{aligned}
    \end{equation*}

    Then
    \begin{equation*}
        \begin{aligned}
            -\Delta t_k b_{t_k} \langle\nabla \mathcal{L}(\boldsymbol{x}_k), \boldsymbol{g}_k \rangle
            &\leq - \Delta t_k b_{t_k} \|\nabla \mathcal{L}(\boldsymbol{x}_k)\|^2 + G_0 C_g \eta \Delta t_k b_{t_k} (1 + R_{disc})^2 \\
            &\leq - \Delta t_k b_{t_k} \|\nabla \mathcal{L}(\boldsymbol{x}_k)\|^2 + G_0 C_g \eta^2 (1 - t_k) (1 + R_{disc})^2 \\
            &\leq - \Delta t_k b_{t_k} \|\nabla \mathcal{L}(\boldsymbol{x}_k)\|^2 + G_0 C_g \eta^2 (1 + R_{disc})^2.
        \end{aligned}
    \end{equation*}
    
    Furthermore, with velocity-loss interaction $|\langle \nabla \mathcal{L}(\boldsymbol{x}_k), u_{t_k}^{\theta}(\boldsymbol{x}_k) \rangle| \leq \beta_1 \|\nabla \mathcal{L}(\boldsymbol{x}_k)\|^2 + \beta_2$, we obtain
    \begin{equation*}
        \begin{aligned}
            \langle \nabla \mathcal{L}(\boldsymbol{x}_k), \boldsymbol{x}_{k+1} - \boldsymbol{x}_k \rangle
            &= \Delta t_k \langle\nabla \mathcal{L}(\boldsymbol{x}_k), u_{t_k}^{\theta}(\boldsymbol{x}_k) \rangle - \Delta t_k b_{t_k} \langle\nabla \mathcal{L}(\boldsymbol{x}_k), \boldsymbol{g}_k  \rangle \\
            &\leq \Delta t_k \beta_1 \|\nabla \mathcal{L}(\boldsymbol{x}_k)\|^2 + \Delta t_k \beta_2 - \Delta t_k b_{t_k} \|\nabla \mathcal{L}(\boldsymbol{x}_k)\|^2 \\
            &\quad + G_0 C_g \eta^2 (1 + R_{disc})^2 \\
            &\leq - \Delta t_k (b_{t_k} - \beta_1) \|\nabla \mathcal{L}(\boldsymbol{x}_k)\|^2 + \eta \beta_2 + G_0 C_g \eta^2 (1+R_{disc})^2 
        \end{aligned}
    \end{equation*}

    Define $t_* := \frac{1}{1+\beta_1}$. On the one hand, when $b_{t_k} > \beta_1$, which yields $t_k < t_*$, we have
    \begin{equation*}
        \langle \nabla \mathcal{L}(\boldsymbol{x}_k), \boldsymbol{x}_{k+1} - \boldsymbol{x}_k \rangle \leq - 2 \mu \Delta t_k (b_{t_k} - \beta_1) V_k + \eta \beta_2 + G_0 C_g \eta^2 (1+R_{disc})^2,
    \end{equation*}
    since the PL condition $\|\nabla \mathcal{L}(\boldsymbol{x})\|^2 \geq 2 \mu \mathcal{L}(\boldsymbol{x}) = 2 \mu V_k$. On the other hand, when $b_{t_k} \leq \beta_1$, which yields $t_k \geq t_*$, we can obtain
    \begin{equation*}
        \begin{aligned}
            \langle \nabla \mathcal{L}(\boldsymbol{x}_k), \boldsymbol{x}_{k+1} - \boldsymbol{x}_k \rangle
            &\leq- \Delta t_k (b_{t_k} - \beta_1) \|\nabla \mathcal{L}(\boldsymbol{x}_k)\|^2 + \eta \beta_2 + G_0 C_g \eta^2 (1+R_{disc})^2 \\
            &\leq - \Delta t_k (b_{t_k} - \beta_1) G_0^2 (1 + \|\boldsymbol{x}_k\|)^2 + \eta \beta_2 + G_0 C_g \eta^2 (1+R_{disc})^2 \\
            &\leq (\beta_1 \eta t_k - (1 - t_k) \eta) G_0^2 (1 + R_{disc})^2 + \eta \beta_2 + G_0 C_g \eta^2 (1+R_{disc})^2 \\
            &\leq \beta_1 \eta G_0^2 (1 + R_{disc})^2 + \eta \beta_2 + G_0 C_g \eta^2 (1+R_{disc})^2.
        \end{aligned}
    \end{equation*}
    
    Denote $\rho_k = 2 \mu \Delta t_k (b_{t_k} - \beta_1)$, we have
    \begin{equation*}
        \begin{aligned}
            V_{k+1} 
            &\leq (1 - \rho_k) V_k + \eta \beta_2 + G_0 C_g \eta^2 (1+R_{disc})^2 + \dfrac{L_{\mathcal{L}} C_d}{2} (1 + R_{disc}^2) \eta^2 \\
            &:= (1 - \rho_k) V_k + \tilde{C}.
        \end{aligned}
    \end{equation*}
    for $t_k < t_{*}$, and
    \begin{equation*}
        \begin{aligned}
            V_{k+1}
            &\leq V_k + \beta_1 \eta G_0^2 (1 + R_{disc})^2 + \eta \beta_2 + G_0 C_g \eta^2 (1+R_{disc})^2 + \dfrac{L_{\mathcal{L}} C_d}{2} (1 + R_{disc}^2) \eta^2 \\
            &:= V_k + \tilde{C}^{\prime}.
        \end{aligned}
    \end{equation*}
    for $t_k \geq t_{*}$.
\end{proof}

With the above lemmas, we can propose the following theorem.

\begin{theorem}\label{thm:det_discrete}
    Consider Algorithm~\ref{alg:fm4pde_det_optimizer} with a time grid $\epsilon = t_0 < t_1 < \cdots < t_N \le 1$ where $\Delta t_k = t_{k+1} - t_k$. Under \Cref{ass:velocity,ass:dissipative,ass:loss,ass:interaction_coercivity}, suppose $\Delta t_k$ is sufficiently small and
    \begin{equation*}
        G_c \ge \tilde{G}_0 (1 + R_{disc}),
    \end{equation*}
    where $R_{disc}$ is the trajectory bound from Lemma~\ref{lem:det_discrete_traj} (independent of $G_c$). Then:
    \begin{equation*}
        V_N = \mathcal{L}(\boldsymbol{x}_N) \le C_{det} \cdot \epsilon^{2\mu} \cdot V_0 + C'_{det},
    \end{equation*}
    where $C_{det}$ and $C'_{det}$ are defined in the proof.
\end{theorem}

\begin{proof}
    By Lemma~\ref{lem:det_discrete_traj}, the trajectory satisfies $\norm{\boldsymbol{x}_k} \le R_{disc}$ for all $k$. By the gradient growth bound (\Cref{rem:gradient_growth}) and the Jacobian bound, $\norm{\boldsymbol{g}_k} \le \tilde{G}_0(1+\norm{\boldsymbol{x}_k}) \le \tilde{G}_0(1+R_{disc}) \le G_c$. Therefore $\boldsymbol{g}_k^{clip} = \boldsymbol{g}_k$ for all $k$, and all subsequent analysis proceeds with $\boldsymbol{g}_k$ directly.
    
    In Phase A ($t_k<t_*$), we have $\rho_k>0$. To keep the recursion coefficient nonnegative, it is necessary to impose $1-\rho_k>0$, i.e., $\rho_k<1$. With the geometric grid $\Delta t_k=\eta t_k$ and $b_{t_k}=\tfrac{1}{t_{k}}-1$, we obtain
    \begin{equation*}
        \rho_k = 2 \mu \Delta t_k \left( b_{t_k}-\beta_1 \right) = 2\mu \eta t_k \left(\frac{1}{t_k}-1-\beta_1\right) = 2\mu\eta\left(1-(1+\beta_1)t_k\right) \leq 2\mu\eta.
    \end{equation*}

    Hence, $\eta < \frac{1}{2\mu}$ guarantees $\rho_k<1$ and thus $1-\rho_k>0$ throughout Phase A. Iterating the recursion from $k=0$ to $k_*-1$ with $k_*=\min\{k:t_k\ge t_*\}$ gives, 
    \begin{equation*}
        V_{k_*} \le \left( \prod_{j=0}^{k_*-1}(1-\rho_j) \right) V_0 + \sum_{k=0}^{k_*-1}\tilde{C} \prod_{j=k+1}^{k_*-1}(1-\rho_j).
    \end{equation*}

    Using $1-x \le e^{-x}$ for $x \ge 0$, we have
    %
    \begin{equation*} 
        \begin{aligned}
            V_{k_*}
            &\le \exp \left( -\sum_{j=0}^{k_*-1}\rho_j \right) V_0 + \tilde C\sum_{k=0}^{k_*-1} \exp \left( -\sum_{j=k+1}^{k_*-1}\rho_j \right) \\
            &:= \exp \left( -S_1 \right) V_0 + \tilde{C} \sum_{k=0}^{k_*-1} \exp \left( -S_2 \right)
        \end{aligned}
    \end{equation*}

    For the first term, we have
    \begin{equation*}
        S_1 = \sum_{j=0}^{k_*-1}\rho_j = \sum_{j=0}^{k_*-1}2\mu \Delta t_j\left(b_{t_j}-\beta_1\right) = \sum_{j=0}^{k_*-1} 2\mu \eta - \sum_{j=0}^{k_*-1} 2\mu \eta (1 + \beta_1) t_j.
    \end{equation*}

    Note that
    \begin{equation*}
        \sum_{j=0}^{k_*-1} 2\mu \eta = 2 \mu \eta \cdot k_{*},
    \end{equation*}
    where $k_*$ satisfies $t_{k_*} = t_0 (1+\eta)^{k_*} = \epsilon (1+\eta)^{k_*}$. Then
    \begin{equation*}
        \ln \left( \dfrac{t_{k_*}}{\epsilon} \right) = k_* \ln(1+\eta) \implies k_* = \frac{\ln(t_{k_*}/\epsilon)}{\ln(1+\eta)} \geq \frac{\ln(t_{*}/\epsilon)}{\ln(1+\eta)},
    \end{equation*}
    and therefore
    \begin{equation*}
        \sum_{j=0}^{k_*-1} 2\mu \eta \geq 2\mu \eta \frac{\ln(t_*/\epsilon)}{\ln(1+\eta)} = 2\mu \left[ \frac{\eta}{\ln(1+\eta)} \right] \ln\left(\frac{t_*}{\epsilon}\right).
    \end{equation*}
    
    Moreover, 
    \begin{equation*}
        2\mu (1+\beta_1) \sum_{j=0}^{k_*-1} \eta t_j = 2 \mu (1+\beta_1) \sum_{j=0}^{k_*-1} \Delta t_j = 2 \mu (1 + \beta_1) (t_{k_*} - \epsilon).
    \end{equation*}

    Since $\eta > \ln (1 + \eta)$ for $\eta > 0$ and $0 < t_{k_*} - \epsilon \leq 1$, we have
    \begin{equation*}
        S_1 > 2\mu \ln\left(\frac{t_*}{\epsilon}\right) - 2 \mu (1 + \beta_1) (t_{k_*} - \epsilon) > 2\mu \ln\left(\frac{t_*}{\epsilon}\right) - 2 \mu (1 + \beta_1),
    \end{equation*}
    which yields
    \begin{equation*}
        \exp (-S_1) \leq \epsilon^{2 \mu} t_*^{-2 \mu} e^{2 \mu (1 + \beta_1)} := C_{A} \epsilon{^{2 \mu}},
    \end{equation*}
    where $C_{A} := t_*^{-2 \mu} e^{2 \mu (1 + \beta_1)}$.

    For the second term,
    \begin{equation*}
        S_2 = \sum_{j=k+1}^{k_*-1}\rho_j = \sum_{j=k+1}^{k_*-1}2\mu \Delta t_j\left(b_{t_j}-\beta_1\right) = \sum_{j=k+1}^{k_*-1} 2\mu \eta - \sum_{j=k+1}^{k_*-1} 2\mu \eta (1 + \beta_1) t_j.
    \end{equation*}

    Note that
    \begin{equation*}
        \sum_{j=k+1}^{k_*-1} 2\mu \eta = 2 \mu \eta \cdot (k_{*} - k - 1),
    \end{equation*}
    where
    \begin{equation}\label{eq:k_star_k_1}
        \frac{t_{k_*}}{t_{k+1}} = \frac{t_0(1+\eta)^{k_*}}{t_0(1+\eta)^{k+1}} = (1+\eta)^{k_* - k - 1} \implies k_{*} - k - 1 = \frac{\ln(t_{k_*}/t_{k+1})}{\ln(1+\eta)},
    \end{equation}
    and therefore
    \begin{equation*}
        \sum_{j=k+1}^{k_*-1} 2\mu \eta \geq 2\mu \eta \frac{\ln(t_*/t_{k+1})}{\ln(1+\eta)} = 2\mu \left[ \frac{\eta}{\ln(1+\eta)} \right] \ln\left(\frac{t_*}{t_{k+1}}\right).
    \end{equation*}
    
    Moreover, since
    \begin{equation*}
        \sum_{j=k+1}^{k_*-1} \Delta t_j = t_{k_*} - t_{k+1},
    \end{equation*}
    we have
    \begin{equation*}
        2\mu (1+\beta_1) \sum_{j=k+1}^{k_*-1} \eta t_j = 2\mu (1+\beta_1) \sum_{j=k+1}^{k_*-1} \Delta t_j = 2 \mu (1 + \beta_1) (t_{k_*} - t_{k+1}).
    \end{equation*}

    Since $\eta > \ln (1 + \eta)$ for $\eta > 0$ and $0 < t_{k_*} - t_{k+1} \leq 1$, we have
    \begin{equation*}
        S_2 > 2\mu \ln\left(\frac{t_*}{t_{k+1}}\right) - 2 \mu (1 + \beta_1) (t_{k_*} - t_{k+1}) > 2\mu \ln\left(\frac{t_*}{t_{k+1}}\right) - 2 \mu (1 + \beta_1).
    \end{equation*}
    which yields
    \begin{equation*}
        \exp (-S_2) \leq \left( \frac{t_{k+1}}{t_{k_*}} \right)^{2\mu} e^{2 \mu (1 + \beta_1)}.
    \end{equation*}

    Recalling Equation~\eqref{eq:k_star_k_1}, we obtain
    \begin{equation*}
        \begin{aligned}
            \sum_{k=0}^{k_*-1} \exp (-S_2) 
            \leq e^{2 \mu (1 + \beta_1)}\sum_{k=0}^{k_*-1} (1+\eta)^{-2\mu(k_* - k - 1)} \leq \dfrac{e^{2 \mu (1 + \beta_1)}}{1 - (1+\eta)^{-2 \mu}}.
        \end{aligned}
    \end{equation*}

    Therefore,
    \begin{equation*}
        V_{k_*} \leq \exp \left( -S_1 \right) V_0 + \tilde{C} \sum_{k=0}^{k_*-1} \exp \left( -S_2 \right) \leq C_{A} \epsilon^{2 \mu} V_0 + C_A^{\prime},
    \end{equation*}
    where $C_A^{\prime} := \tilde{C} \dfrac{e^{2 \mu (1 + \beta_1)}}{1 - (1+\eta)^{-2 \mu}}$.
    
    In Phase B ($t_k\ge t_*$), we have $V_N \leq V_{k_*} + \tilde{C}^{\prime} (N - k_*)$. From the relation $t_N = t_{k_*} (1+\eta)^{N-k_*} \leq 1$, we obtain
    \begin{equation*}
        N - k_* \leq \frac{\ln(1/t_{k_*})}{\ln(1+\eta)} \leq \frac{\ln(1/t_*)}{\ln(1+\eta)}.
    \end{equation*}
    Substituting this back, we obtain
    \begin{equation*}
        V_N \leq V_{k_*} + \tilde{C}^{\prime} \frac{\ln(1/t_*)}{\ln(1+\eta)} = V_{k_*} + C_B^{\prime},
    \end{equation*}
    where $C_B^{\prime} = \tilde{C}^{\prime} \frac{\ln(1+\beta_1)}{\ln(1+\eta)}$. Combining both terms, we establish
    \begin{equation*}
        V_N \leq C_A \epsilon^{2 \mu} V_0 + C_{A}^{\prime} + C_B^{\prime} := C_{det} \epsilon^{2 \mu} V_0 + C_{det}^{\prime}.
    \end{equation*}
    
    This completes the proof.
\end{proof}

\subsection{Stochastic Sampler Analysis}

This section focuses on the convergence analysis of the stochastic sampler, Algorithm \ref{alg:fm4pde_stoc_sampler}. The central challenge is proving uniform boundedness of iterates. We employ a dissipativity condition on the velocity field combined with adaptive guidance, which provides direct control over moments via a Lyapunov argument. We first introduce and redefine some notations as follows.
\begin{itemize}
    \item $\Phi_t(\boldsymbol{x}) := \boldsymbol{x} + (1-t)u_t^{\theta}(\boldsymbol{x})$,
    \item $J_t(\boldsymbol{x}) := \nabla_{\boldsymbol{x}} \Phi_t(\boldsymbol{x}) = I + (1-t)\nabla u_t^{\theta}(\boldsymbol{x})$,
    \item $\mathcal{F}_k := \sigma(\boldsymbol{x}_0, \xi_0, \ldots, \xi_{k-1})$,
    \item $\zeta_k := c_\zeta \delta_k$ (adaptive guidance strength for stochastic sampler),
    \item $\bar\zeta := \zeta_N = c_\zeta \delta_{\min}$ (terminal guidance strength),
    \item $\bar{\boldsymbol{g}}_k := \nabla \mathcal{L}(\hat{\boldsymbol{x}}_1^{(k)})$ (endpoint gradient),
    \item $\boldsymbol{g}_k := J_{t_k}(\boldsymbol{x}_k)^\top \bar{\boldsymbol{g}}_k$ (backpropagated gradient),
    \item $V_k := \mathbb{E}[\mathcal{L}(\hat{\boldsymbol{x}}_1^{(k)})]$ (expected loss at step $k$).
\end{itemize}

With these notations, we present the algorithm at step $k$:
\begin{equation*}
    \begin{aligned}
        \hat{\boldsymbol{x}}_1^{(k)} &= \boldsymbol{x}_k + (1-t_k) u_{t_k}^{\theta}(\boldsymbol{x}_k), \\
        \tilde{\boldsymbol{x}}_{k+1} &= (1-t_{k+1})\xi_k + t_{k+1} \hat{\boldsymbol{x}}_1^{(k)}, \quad \xi_k \sim \mathcal{N}(\mathrm{0}, \mathrm{I}), \\
        \boldsymbol{x}_{k+1} &= \tilde{\boldsymbol{x}}_{k+1} - c_{\zeta} \delta_k \boldsymbol{g}_k.
    \end{aligned}
\end{equation*}

We begin with several foundational lemmas.
\begin{lemma}[Regularity of Prediction Map]
    \label{lem:phi_properties}
    Under \Cref{ass:velocity}, the following hold for all $t \in [0,1]$:
    \begin{enumerate}[label=(\roman*)]
        \item \label{lem:phi_lipschitz_cont} $\Phi_t$ is $L_\Phi$-Lipschitz continuous with $L_\Phi = 1 + B_J$,
        \item \label{lem:jt_lipschitz_cont} $J_t$ is $L_u$-Lipschitz continuous,
        \item \label{lem:phi_hessian_bound} For all $\boldsymbol{v}, \boldsymbol{w} \in \mathbb{R}^d$: $\norm{\nabla^2 \Phi_t(\boldsymbol{x})[\boldsymbol{v}, \boldsymbol{w}]} \le H_\Phi \norm{\boldsymbol{v}}\norm{\boldsymbol{w}}$ where $H_\Phi = L_u$.
    \end{enumerate}
\end{lemma}

\begin{proof}
    For part \ref{lem:phi_lipschitz_cont}, let $\boldsymbol{x}, \boldsymbol{y} \in \mathbb{R}^d$. Then:
    \begin{align*}
        \norm{\Phi_t(\boldsymbol{x}) - \Phi_t(\boldsymbol{y})}
        &= \norm{(\boldsymbol{x} - \boldsymbol{y}) + (1-t)(u_t(\boldsymbol{x}) - u_t(\boldsymbol{y}))} \\
        &\le \norm{\boldsymbol{x} - \boldsymbol{y}} + (1-t)\norm{u_t(\boldsymbol{x}) - u_t(\boldsymbol{y})}.
    \end{align*}
    By the mean value theorem and \Cref{ass:velocity}\ref{ass:velocity_grad_bound}, there exists $\boldsymbol{\xi}$ on the line segment joining $\boldsymbol{x}$ and $\boldsymbol{y}$ such that:
    \begin{equation*}
        \norm{u_t(\boldsymbol{x}) - u_t(\boldsymbol{y})} = \norm{\nabla u_t(\boldsymbol{\xi})(\boldsymbol{x} - \boldsymbol{y})} \le B_J\norm{\boldsymbol{x} - \boldsymbol{y}}.
    \end{equation*}
    Therefore $\norm{\Phi_t(\boldsymbol{x}) - \Phi_t(\boldsymbol{y})} \le (1 + B_J)\norm{\boldsymbol{x} - \boldsymbol{y}}$.

    For part \ref{lem:jt_lipschitz_cont}, note that $J_t(\boldsymbol{x}) - J_t(\boldsymbol{y}) = (1-t)[\nabla u_t(\boldsymbol{x}) - \nabla u_t(\boldsymbol{y})]$, so:
    \begin{equation*}
        \norm{J_t(\boldsymbol{x}) - J_t(\boldsymbol{y})} \le (1-t)L_u\norm{\boldsymbol{x} - \boldsymbol{y}} \le L_u\norm{\boldsymbol{x} - \boldsymbol{y}}.
    \end{equation*}

    For part \ref{lem:phi_hessian_bound}, the Hessian of $\Phi_t$ satisfies $\nabla^2 \Phi_t(\boldsymbol{x}) = (1-t)\nabla^2 u_t(\boldsymbol{x})$. The Lipschitz property of $\nabla u_t$ (Assumption~\ref{ass:velocity}\ref{ass:velocity_L_smooth}) implies that for any unit vector $\boldsymbol{w}$:
    \begin{equation*}
        \norm{\nabla^2 u_t(\boldsymbol{x})[\boldsymbol{v}, \boldsymbol{w}]} = \lim_{\epsilon \to 0}\frac{1}{\epsilon}\norm{\nabla u_t(\boldsymbol{x}+\epsilon\boldsymbol{w})\boldsymbol{v} - \nabla u_t(\boldsymbol{x})\boldsymbol{v}} \le L_u\norm{\boldsymbol{v}}\norm{\boldsymbol{w}}.
    \end{equation*}
    Consequently, $\norm{\nabla^2 \Phi_t(\boldsymbol{x})[\boldsymbol{v}, \boldsymbol{w}]} \le (1-t)L_u\norm{\boldsymbol{v}}\norm{\boldsymbol{w}} \le L_u\norm{\boldsymbol{v}}\norm{\boldsymbol{w}}$.
\end{proof}

\begin{lemma}[Conditional Dynamics]
    \label{lem:conditional}
    Conditioned on $\mathcal{F}_k$, we have:
    \begin{equation*}
        \begin{aligned}
            \mathbb{E}_k[\boldsymbol{x}_{k+1}] &= t_{k+1}\hat{\boldsymbol{x}}_1^{(k)} - c_\zeta \delta_k J_{t_k}(\boldsymbol{x}_k)^\top \bar{\boldsymbol{g}}_k, \\
            \mathrm{Var}(\boldsymbol{x}_{k+1} \mid \mathcal{F}_k) &= \delta_{k+1}^2 I.
        \end{aligned}
    \end{equation*}
\end{lemma}

\begin{proof}
    According to Algorithm~\ref{alg:fm4pde_stoc_sampler}, we have
    \begin{equation*}
        \boldsymbol{x}_{k+1} = (1-t_{k+1})\xi_k + t_{k+1}\hat{\boldsymbol{x}}_1^{(k)} - c_\zeta \delta_k J_{t_k}(\boldsymbol{x}_k)^\top \bar{\boldsymbol{g}}_k = \delta_{k+1}\xi_k + t_{k+1}\hat{\boldsymbol{x}}_1^{(k)} - c_\zeta \delta_k \boldsymbol{g}_k.
    \end{equation*}
    Since $\xi_k \sim \mathcal{N}(\mathbf{0}, I)$ is sampled independently at step $k$ and is independent of $\mathcal{F}_k$, while $\hat{\boldsymbol{x}}_1^{(k)}$ and $\boldsymbol{g}_k$ are $\mathcal{F}_k$-measurable, we obtain
    \begin{equation*}
        \mathbb{E}_k[\boldsymbol{x}_{k+1}] = \delta_{k+1}\mathbb{E}[\xi_k] + t_{k+1}\hat{\boldsymbol{x}}_1^{(k)} - c_\zeta \delta_k \boldsymbol{g}_k = t_{k+1}\hat{\boldsymbol{x}}_1^{(k)} - c_\zeta \delta_k J_{t_k}^\top \bar{\boldsymbol{g}}_k.
    \end{equation*}
    For the variance, since only the random component $\delta_{k+1}\xi_k$ contributes:
    \begin{equation*}
        \mathrm{Var}(\boldsymbol{x}_{k+1} \mid \mathcal{F}_k) = \mathrm{Var}(\delta_{k+1}\xi_k) = \delta_{k+1}^2 I.
    \end{equation*}
\end{proof}

\begin{lemma}[Jensen Gap from Nonlinearity]
    \label{lem:jensen}
    Let $\boldsymbol{X}$ be a random vector with $\mathbb{E}[\boldsymbol{X}] = \boldsymbol{\mu}$ and $\mathbb{E}[\norm{\boldsymbol{X} - \boldsymbol{\mu}}^2] = \sigma^2 d$. Under \Cref{ass:velocity}, for any $t \in [0,1]$:
    \begin{equation*}
        \norm{\mathbb{E}[\Phi_t(\boldsymbol{X})] - \Phi_t(\boldsymbol{\mu})} \le \frac{H_\Phi \sigma^2 d}{2}.
    \end{equation*}
\end{lemma}

\begin{proof}
    By the multivariate Taylor expansion with integral remainder:
    \begin{equation*}
        \Phi_t(\boldsymbol{X}) = \Phi_t(\boldsymbol{\mu}) + J_t(\boldsymbol{\mu})(\boldsymbol{X} - \boldsymbol{\mu}) + \int_0^1 (1-s) \nabla^2\Phi_t(\boldsymbol{\mu} + s(\boldsymbol{X} - \boldsymbol{\mu}))[\boldsymbol{X} - \boldsymbol{\mu}, \boldsymbol{X} - \boldsymbol{\mu}] ds.
    \end{equation*}
    Taking expectations and using $\mathbb{E}[\boldsymbol{X} - \boldsymbol{\mu}] = 0$, the linear term vanishes:
    \begin{equation*}
        \mathbb{E}[\Phi_t(\boldsymbol{X})] - \Phi_t(\boldsymbol{\mu}) = \mathbb{E}\left[\int_0^1 (1-s) \nabla^2\Phi_t(\boldsymbol{\mu} + s(\boldsymbol{X} - \boldsymbol{\mu}))[\boldsymbol{X} - \boldsymbol{\mu}, \boldsymbol{X} - \boldsymbol{\mu}] ds\right].
    \end{equation*}
    Taking norms and applying \Cref{lem:phi_properties}\ref{lem:phi_hessian_bound}:
    \begin{align*}
        \norm{\mathbb{E}[\Phi_t(\boldsymbol{X})] - \Phi_t(\boldsymbol{\mu})} 
        &\le \mathbb{E}\left[\int_0^1 (1-s) \norm{\nabla^2\Phi_t(\boldsymbol{\mu} + s(\boldsymbol{X} - \boldsymbol{\mu}))[\boldsymbol{X} - \boldsymbol{\mu}, \boldsymbol{X} - \boldsymbol{\mu}]} ds\right] \\
        &\le \mathbb{E}\left[\int_0^1 (1-s) H_\Phi \norm{\boldsymbol{X} - \boldsymbol{\mu}}^2 ds\right] \\
        &= H_\Phi \int_0^1 (1-s) ds \cdot \mathbb{E}[\norm{\boldsymbol{X} - \boldsymbol{\mu}}^2] = \frac{H_\Phi \sigma^2 d}{2}.
    \end{align*}
\end{proof}

\begin{lemma}[Temporal Regularity]
    \label{lem:temporal}
    Under \Cref{ass:velocity,ass:time_grid}, for any $\boldsymbol{x} \in \mathbb{R}^d$:
    \begin{equation*}
        \norm{\Phi_{t_{k+1}}(\boldsymbol{x}) - \Phi_{t_k}(\boldsymbol{x})} \le C_t(1+\|\boldsymbol{x}\|)\Delta t_k,
    \end{equation*}
    where $C_t := B_u + L_t$.
\end{lemma}

\begin{proof}
    We compute
    \begin{align*}
        \Phi_{t_{k+1}}(\boldsymbol{x}) - \Phi_{t_k}(\boldsymbol{x})
        &= \left[\boldsymbol{x} + (1-t_{k+1})u_{t_{k+1}}(\boldsymbol{x})\right] - \left[\boldsymbol{x} + (1-t_k)u_{t_k}(\boldsymbol{x})\right] \\
        &= \delta_{k+1}u_{t_{k+1}}(\boldsymbol{x}) - \delta_k u_{t_k}(\boldsymbol{x}).
    \end{align*}
    Adding and subtracting $\delta_k u_{t_{k+1}}(\boldsymbol{x})$:
    \begin{align*}
        \Phi_{t_{k+1}}(\boldsymbol{x}) - \Phi_{t_k}(\boldsymbol{x})
        &= (\delta_{k+1} - \delta_k)u_{t_{k+1}}(\boldsymbol{x}) + \delta_k(u_{t_{k+1}}(\boldsymbol{x}) - u_{t_k}(\boldsymbol{x})) \\
        &= -\Delta t_k \cdot u_{t_{k+1}}(\boldsymbol{x}) + \delta_k(u_{t_{k+1}}(\boldsymbol{x}) - u_{t_k}(\boldsymbol{x})),
    \end{align*}
    since $\delta_{k+1} - \delta_k = -\Delta t_k$. Taking norms:
    \begin{align*}
        \norm{\Phi_{t_{k+1}}(\boldsymbol{x}) - \Phi_{t_k}(\boldsymbol{x})}
        &\le \Delta t_k \norm{u_{t_{k+1}}(\boldsymbol{x})} + \delta_k\norm{u_{t_{k+1}}(\boldsymbol{x}) - u_{t_k}(\boldsymbol{x})} \\
        &\le \Delta t_k \cdot B_u(1+\|\boldsymbol{x}\|) + 1 \cdot L_t(1+\|\boldsymbol{x}\|)\Delta t_k \quad \text{(by \Cref{ass:velocity}\ref{ass:velocity_linear_growth},\ref{ass:velocity_time_smooth})} \\
        &= [B_u + L_t](1+\|\boldsymbol{x}\|)\Delta t_k.
    \end{align*}

    This yields the result.
\end{proof}

\begin{lemma}[Uniform Second Moment Bound with Adaptive Guidance]\label{lem:moment}
    Under \Cref{ass:velocity,ass:dissipative,ass:loss,ass:time_grid}, assume the small gain condition:
    \begin{equation}\label{eq:kappa_Bu_condition}
        B_u^2 < 2\kappa.
    \end{equation}
    Let the adaptive guidance coefficient $c_\zeta$ satisfy the stability condition:
    \begin{equation}\label{eq:czeta_stability}
        c_\zeta^2 \le \frac{\kappa'}{4 L_\Phi^2 G_0^2(1+B_u)^2(1 + 1/\bar\delta_{ratio})},
    \end{equation}
    where $\bar\delta_{ratio} := \min_k (\delta_{k+1}/\delta_k)$ and $\kappa' := \min\{\kappa - B_u^2/2,\; 1\}$. Then there exists a constant $M_2 > 0$ (independent of $N$ and $\delta_{\min}$) such that:
    \begin{equation*}
        \mathbb{E}[\norm{\boldsymbol{x}_k}^2] \le M_2 \quad \text{for all } k \in \{0, 1, \ldots, N\}.
    \end{equation*}
    Explicitly:
    \begin{equation*} 
        M_2 := \max\left\{W_0,\; \frac{2C'_{loc}}{\kappa'}\right\}, \quad C'_{loc} := C_{loc} + d + 1,
    \end{equation*}
    where $C_{loc} := 2C'_{drift} + B_u^2 + \frac{B_u^4}{\kappa - B_u^2/2}$, $W_0 := \mathbb{E}[\norm{\boldsymbol{x}_0}^2]$ is the initial second moment. For the standard initialization $\boldsymbol{x}_0 \sim \mathcal{N}(\mathbf{0}, I)$, we have $W_0 = d$.
\end{lemma}

\begin{proof}
    First, We bound this ratio $\bar\delta_{ratio}$ in two cases. For steps with $t_k \ge 1-\epsilon_0$: by Assumption~\ref{ass:time_grid}(b), $\Delta t_k \le c_\Delta\delta_k$, so $\delta_{k+1}/\delta_k \ge 1-c_\Delta \ge 1/2$. For steps with $t_k < 1-\epsilon_0$: $\delta_k > \epsilon_0$ and $\Delta t_k \le \bar\Delta \le \epsilon_0/2$ by Assumption~\ref{ass:time_grid}(a), so $\delta_{k+1}/\delta_k = 1 - \Delta t_k/\delta_k \ge 1 - \bar\Delta/\epsilon_0 \ge 1/2$. Combining both cases: $\bar\delta_{ratio} \ge 1/2$. Define $W_k := \mathbb{E}[\norm{\boldsymbol{x}_k}^2]$. We establish a recursive inequality using dissipativity.

    From Lemma~\ref{lem:conditional} with $\zeta_k = c_\zeta\delta_k$, we can expand the second moment recursion as follows,
    \begin{equation*}
        \mathbb{E}_k[\norm{\boldsymbol{x}_{k+1}}^2] = \delta_{k+1}^2 d + \norm{t_{k+1}\hat{\boldsymbol{x}}_1^{(k)} - c_\zeta\delta_k J_{t_k}(\boldsymbol{x}_k)^\top \bar{\boldsymbol{g}}_k}^2.
    \end{equation*}

    Then, we apply global dissipativity to the endpoint prediction. Recall $\hat{\boldsymbol{x}}_1^{(k)} = \boldsymbol{x}_k + \delta_k u_{t_k}^{\theta}(\boldsymbol{x}_k)$. We have $\norm{\hat{\boldsymbol{x}}_1^{(k)}}^2 = \norm{\boldsymbol{x}_k}^2 + 2\delta_k\inner{\boldsymbol{x}_k}{u_{t_k}^{\theta}(\boldsymbol{x}_k)} + \delta_k^2\norm{u_{t_k}(\boldsymbol{x}_k)}^2$. By the global dissipativity bound~\eqref{eq:global_dissip}: $\inner{\boldsymbol{x}_k}{u_{t_k}^{\theta}} \le -\kappa\norm{\boldsymbol{x}_k}^2 + C'_{drift}$.

    For the growth term, we use a parameterized Young's inequality. By Assumption~\ref{ass:velocity}(a), $\norm{u_t(\boldsymbol{x})}^2 \le B_u^2(1+\norm{\boldsymbol{x}})^2 = B_u^2 + 2B_u^2\norm{\boldsymbol{x}} + B_u^2\norm{\boldsymbol{x}}^2$. Applying $2B_u^2\norm{\boldsymbol{x}} \le \eta\norm{\boldsymbol{x}}^2 + B_u^4/\eta$ with $\eta := \kappa - B_u^2/2 > 0$ (by the small gain condition):
    \begin{equation*}
        \norm{u_t(\boldsymbol{x})}^2 \le (B_u^2 + \eta)\norm{\boldsymbol{x}}^2 + B_u^2 + B_u^4/\eta.
    \end{equation*}
    Substituting:
    \begin{align}
        \norm{\hat{\boldsymbol{x}}_1^{(k)}}^2
        &\le [1 - 2\delta_k\kappa + \delta_k^2(B_u^2 + \eta)]\norm{\boldsymbol{x}_k}^2 + 2\delta_k C'_{drift} + \delta_k^2(B_u^2 + B_u^4/\eta) \nonumber\\
        &\le (1 - \delta_k\kappa')\norm{\boldsymbol{x}_k}^2 + C_{loc}\delta_k. \label{eq:endpoint_contraction}
    \end{align}
    Indeed, since $\delta_k \le 1$: $2\kappa - \delta_k(B_u^2 + \eta) \ge 2\kappa - B_u^2 - \eta = 2\kappa - B_u^2 - (\kappa - B_u^2/2) = \kappa - B_u^2/2 = \kappa'$, confirming the coefficient. The constant terms satisfy $2C'_{drift} + \delta_k(B_u^2 + B_u^4/\eta) \le C_{loc}$ since $\delta_k \le 1$.

    Furthermore, we control the gradient term in the expansion
    \begin{align*}
        \mathbb{E}_k[\norm{\boldsymbol{x}_{k+1}}^2]
        &= \delta_{k+1}^2 d + t_{k+1}^2\norm{\hat{\boldsymbol{x}}_1^{(k)}}^2 + c_\zeta^2\delta_k^2\norm{\boldsymbol{g}_k}^2 - 2t_{k+1}c_\zeta\delta_k\inner{\hat{\boldsymbol{x}}_1^{(k)}}{\boldsymbol{g}_k}.
    \end{align*}

    Using Young's inequality on the cross term:
    \begin{equation*}
        -2t_{k+1}c_\zeta\delta_k\inner{\hat{\boldsymbol{x}}_1^{(k)}}{\boldsymbol{g}_k} \le \delta_{k+1}\norm{\hat{\boldsymbol{x}}_1^{(k)}}^2 + \frac{t_{k+1}^2 c_\zeta^2\delta_k^2}{\delta_{k+1}}\norm{\boldsymbol{g}_k}^2.
    \end{equation*}

    Since $t_{k+1}^2 + \delta_{k+1} \le 1$, and using~\eqref{eq:endpoint_contraction}:
    \begin{align}
        \mathbb{E}_k[\norm{\boldsymbol{x}_{k+1}}^2]
        &\le \delta_{k+1}^2 d + (1-\delta_k\kappa')\norm{\boldsymbol{x}_k}^2 + C_{loc}\delta_k + \left(1 + \frac{t_{k+1}^2}{\delta_{k+1}}\right)c_\zeta^2\delta_k^2 L_\Phi^2\norm{\bar{\boldsymbol{g}}_k}^2. \label{eq:moment_recursion_adaptive}
    \end{align}

    Since $t_{k+1} \le 1$ and $\delta_{k+1} \ge \bar\delta_{ratio}\,\delta_k$, where $\bar\delta_{ratio} \ge 1/2$ as shown above, we have $\frac{t_{k+1}^2}{\delta_{k+1}} \le \frac{1}{\delta_{k+1}} \le \frac{1}{\bar\delta_{ratio}\,\delta_k}$. Thus the gradient coefficient satisfies
    \begin{equation*}
        \left(1 + \frac{t_{k+1}^2}{\delta_{k+1}}\right)c_\zeta^2\delta_k^2 \le c_\zeta^2\delta_k^2 + \frac{c_\zeta^2\delta_k}{\bar\delta_{ratio}} \le \frac{1+\bar\delta_{ratio}}{\bar\delta_{ratio}}\,c_\zeta^2\delta_k,
    \end{equation*}
    where the last step uses $\delta_k \le 1$. By the gradient growth bound (Remark~\ref{rem:gradient_growth}) and $\|\hat{\boldsymbol{x}}_1^{(k)}\| \leq B_u(1+\|\boldsymbol{x}_k\|) + \|\boldsymbol{x}_k\|$, we have $\|\bar{\boldsymbol{g}}_k\| \leq G_0 (1 + \|\hat{\boldsymbol{x}}_1^{(k)}\|) \leq G_0 (1 + B_u) (1 + \|\boldsymbol{x}_k\|) $, and thus,
    \begin{equation*}
        \norm{\bar{\boldsymbol{g}}_k}^2 \le G_0^2(1+B_u)^2 (1+\norm{\boldsymbol{x}_k})^2 \le 2G_0^2(1+B_u)^2(1+\norm{\boldsymbol{x}_k}^2).
    \end{equation*}

    Under the stability condition \eqref{eq:czeta_stability}:
    \begin{equation*}
        \frac{1+\bar\delta_{ratio}}{\bar\delta_{ratio}}\,c_\zeta^2 L_\Phi^2 \cdot 2 G_0^2(1+B_u)^2 \le \frac{\kappa'}{2}.
    \end{equation*}
    Therefore, the gradient term contributes at most $\frac{\delta_k\kappa'}{2}\norm{\boldsymbol{x}_k}^2 + C_{grad}\delta_k$ for a constant $C_{grad}$, which can be absorbed into $C'_{loc}$, and:
    \begin{equation*}
        \mathbb{E}_k[\norm{\boldsymbol{x}_{k+1}}^2] \le \left(1 - \frac{\delta_k\kappa'}{2}\right)\norm{\boldsymbol{x}_k}^2 + C'_{loc}\delta_k.
    \end{equation*}

    Taking full expectations and applying the invariant-region argument: if $W_k \le M$ then $W_{k+1} \le (1-\delta_k\kappa'/2)M + C'_{loc}\delta_k \le M$ whenever $M \ge 2C'_{loc}/\kappa'$. Since $W_0 = \mathbb{E}[\norm{\boldsymbol{x}_0}^2]$:
    \begin{equation*}
        W_k \le \max\left\{W_0,\; \frac{2C'_{loc}}{\kappa'}\right\} =: M_2.
    \end{equation*}
\end{proof}

\begin{lemma}[Uniform Fourth Moment Bound]\label{lem:fourth_moment}
    Under the assumptions of Lemma~\ref{lem:moment} (including the stability condition~\eqref{eq:czeta_stability} and the small gain condition~\eqref{eq:kappa_Bu_condition}), there exists a constant $M_4 > 0$, depending only on problem parameters $(d, \kappa', C'_{loc}, M_2, W_0^{(4)})$ where $W_0^{(4)} := \mathbb{E}[\norm{\boldsymbol{x}_0}^4]$, and independent of $N$ and $\delta_{\min}$, such that for all $k \in \{0, 1, \ldots, N\}$
    \begin{equation*}
        \mathbb{E}[\norm{\boldsymbol{x}_k}^4] \le M_4.
    \end{equation*}
    No additional threshold on $c_\zeta$ beyond the second moment stability condition~\eqref{eq:czeta_stability} is required.
\end{lemma}

\begin{proof}
    The proof leverages the pathwise second moment contraction already established in Lemma~\ref{lem:moment}, combined with the second moment bound $\mathbb{E}[\norm{\boldsymbol{x}_k}^2] \le M_2$ to control cross terms. Define $\gamma := \kappa'/2 > 0$. From the proof of Lemma~\ref{lem:moment} (specifically, the chain of inequalities leading to~\eqref{eq:moment_recursion_adaptive} through the stability condition), the following pathwise bound holds for each realization of $\boldsymbol{x}_k$:
    \begin{equation}\label{eq:mk_pathwise}
        \norm{\boldsymbol{m}_k}^2 \le (1 - \gamma\delta_k)\norm{\boldsymbol{x}_k}^2 + C'_{loc}\delta_k,
    \end{equation}
    where $\boldsymbol{m}_k = t_{k+1}\hat{\boldsymbol{x}}_1^{(k)} - \zeta_k\boldsymbol{g}_k$ is the conditional mean of $\boldsymbol{x}_{k+1}$ given $\mathcal{F}_k$. This bound incorporates the endpoint contraction~\eqref{eq:endpoint_contraction}, the Young's inequality on the cross term, and the absorption of the gradient contribution via the stability condition~\eqref{eq:czeta_stability}. It is $\mathcal{F}_k$-measurable and holds for every $k$.

    We then derive the fourth moment expansion. From $\boldsymbol{x}_{k+1} = \delta_{k+1}\xi_k + \boldsymbol{m}_k$ with $\xi_k \sim \mathcal{N}(\boldsymbol{0}, I)$ independent of $\mathcal{F}_k$, standard Gaussian moment computations yield:
    \begin{equation}\label{eq:fourth_expansion}
        \mathbb{E}_k[\norm{\boldsymbol{x}_{k+1}}^4] = \delta_{k+1}^4 d(d+2) + (2d+4)\delta_{k+1}^2\norm{\boldsymbol{m}_k}^2 + \norm{\boldsymbol{m}_k}^4.
    \end{equation}

    Squaring the pathwise bound~\eqref{eq:mk_pathwise}:
    \begin{equation}\label{eq:mk4_bound}
        \norm{\boldsymbol{m}_k}^4 \le (1 - \gamma\delta_k)^2\norm{\boldsymbol{x}_k}^4 + 2(1-\gamma\delta_k)C'_{loc}\delta_k\norm{\boldsymbol{x}_k}^2 + (C'_{loc})^2\delta_k^2.
    \end{equation}

    Since $\gamma = \kappa'/2 \le 1/2$ and $\delta_k \le 1$, we have $\gamma\delta_k \in (0, 1)$. The elementary inequality $(1 - a)^2 \le 1 - a$ for all $a \in (0, 1)$, which follows from $a^2 \le a$ for $a \in [0,1]$ gives:
    \begin{equation*}
        (1 - \gamma\delta_k)^2 \le 1 - \gamma\delta_k.
    \end{equation*}

    Substituting~\eqref{eq:mk4_bound} into~\eqref{eq:fourth_expansion} and taking full expectations:
    \begin{align*}
        \mathbb{E}[\norm{\boldsymbol{x}_{k+1}}^4]
        &\le (1 - \gamma\delta_k)\mathbb{E}[\norm{\boldsymbol{x}_k}^4] + 2C'_{loc}\delta_k\,\mathbb{E}[\norm{\boldsymbol{x}_k}^2] + (C'_{loc})^2\delta_k^2 \nonumber\\
        &\quad + (2d+4)\delta_{k+1}^2\,\mathbb{E}[\norm{\boldsymbol{m}_k}^2] + d(d+2)\delta_{k+1}^4. 
    \end{align*}

    We bound each driving term using established estimates:
    \begin{enumerate}[label=(\roman*),topsep=2pt,itemsep=1pt]
        \item $\mathbb{E}[\norm{\boldsymbol{x}_k}^2] \le M_2$ by Lemma~\ref{lem:moment}, so $2C'_{loc}\delta_k\,\mathbb{E}[\norm{\boldsymbol{x}_k}^2] \le 2C'_{loc}M_2\delta_k$.
        \item $\mathbb{E}[\norm{\boldsymbol{m}_k}^2] \le M_2 + C'_{loc}$ by~\eqref{eq:mk_pathwise} and $\mathbb{E}[\norm{\boldsymbol{x}_k}^2] \le M_2$, so $(2d+4)\delta_{k+1}^2\,\mathbb{E}[\norm{\boldsymbol{m}_k}^2] \le (2d+4)(M_2+C'_{loc})\delta_k$ (using $\delta_{k+1} \le \delta_k$ and $\delta_k \le 1$).
        \item $(C'_{loc})^2\delta_k^2 \le (C'_{loc})^2\delta_k$ and $d(d+2)\delta_{k+1}^4 \le d(d+2)\delta_k$.
    \end{enumerate}

    All driving terms are $O(\delta_k)$. Defining:
    \begin{equation*}
        \tilde{C}_4 := 2C'_{loc}M_2 + (C'_{loc})^2 + (2d+4)(M_2 + C'_{loc}) + d(d+2),
    \end{equation*}
    which depends only on $(d, M_2, C'_{loc})$ and is independent of $\delta_{\min}$ and $N$, we obtain:
    \begin{equation*}
        \mathbb{E}[\norm{\boldsymbol{x}_{k+1}}^4] \le (1 - \gamma\delta_k)\mathbb{E}[\norm{\boldsymbol{x}_k}^4] + \tilde{C}_4\delta_k.
    \end{equation*}

    We apply the same argument as Lemma~\ref{lem:moment}. If $\mathbb{E}[\norm{\boldsymbol{x}_k}^4] \le M$, then:
    \begin{equation*}
        \mathbb{E}[\norm{\boldsymbol{x}_{k+1}}^4] \le (1 - \gamma\delta_k)M + \tilde{C}_4\delta_k \le M
    \end{equation*}
    whenever $M \ge \tilde{C}_4/\gamma$. For the initial condition, let $W_0^{(4)} := \mathbb{E}[\norm{\boldsymbol{x}_0}^4]$. For the standard initialization $\boldsymbol{x}_0 \sim \mathcal{N}(0, I)$, we have $W_0^{(4)} = d(d+2)$. Therefore:
    \begin{equation*}
        M_4 := \max\left\{W_0^{(4)},\; \frac{\tilde{C}_4}{\gamma}\right\} = \max\left\{W_0^{(4)},\; \frac{2\tilde{C}_4}{\kappa'}\right\}
    \end{equation*}
    satisfies $\mathbb{E}[\norm{\boldsymbol{x}_k}^4] \le M_4$ for all $k$, and depends only on $(d, \kappa', C'_{loc}, M_2, W_0^{(4)})$.
\end{proof}

\begin{corollary}[Derived Moment Bounds]\label{cor:derived}
    Under the assumptions of Lemmas~\ref{lem:moment} and~\ref{lem:fourth_moment}:
    \begin{enumerate}[label=(\alph*),topsep=2pt,itemsep=1pt]
        \item $\mathbb{E}[\norm{\hat{\boldsymbol{x}}_1^{(k)}}^2] \le M_{\hat{x},2} := 2(1 + B_u)^2(1 + M_2)$ for all $k \le N$,
        \item $\mathbb{E}[\norm{\hat{\boldsymbol{x}}_1^{(k)}}^4] \le M_{\hat{x},4} := 8(1 + B_u)^4(1 + M_4)$ for all $k \le N$,
        \item Under the assumptions of Lemmas~\ref{lem:moment} and~\ref{lem:fourth_moment}, for all $k \le N$, we have
        \begin{equation*}
            V_k = \mathbb{E}[\mathcal{L}(\hat{\boldsymbol{x}}_1^{(k)})] \le M_{loss} := \frac{L_{\mathcal{L}}}{2} \left( M_{\hat{x},2} + M_{x^*}^2 + 2 M_{x^*} \sqrt{M_{\hat{x},2}} \right).
        \end{equation*}
    \end{enumerate}
\end{corollary}

\begin{proof}
    For part (a), Using $\hat{\boldsymbol{x}}_1^{(k)} = \boldsymbol{x}_k + \delta_k u_{t_k}(\boldsymbol{x}_k)$ and $\delta_k \le 1$, we have
    \begin{equation*}
        \norm{\hat{\boldsymbol{x}}_1^{(k)}} \le \norm{\boldsymbol{x}_k} + B_u(1+\norm{\boldsymbol{x}_k}) = (1+B_u)\norm{\boldsymbol{x}_k} + B_u \le (1+B_u)(1+\norm{\boldsymbol{x}_k}).
    \end{equation*}
    
    Thus $\norm{\hat{\boldsymbol{x}}_1^{(k)}}^2 \le (1+B_u)^2(1+\norm{\boldsymbol{x}_k})^2 \le 2(1+B_u)^2(1+\norm{\boldsymbol{x}_k}^2)$. Taking expectations yields
    \begin{equation*}
        \mathbb{E}[\norm{\hat{\boldsymbol{x}}_1^{(k)}}^2] \le 2(1+B_u)^2(1+M_2) =: M_{\hat{x},2}.
    \end{equation*}

    Part (b), using $\hat{\boldsymbol{x}}_1^{(k)} = \boldsymbol{x}_k + \delta_k u_{t_k}(\boldsymbol{x}_k)$ and $\delta_k \le 1$, we have
    \begin{equation*}
        \norm{\hat{\boldsymbol{x}}_1^{(k)}} \le \norm{\boldsymbol{x}_k} + B_u(1+\norm{\boldsymbol{x}_k}) = (1+B_u)\norm{\boldsymbol{x}_k} + B_u \le (1+B_u)(1+\norm{\boldsymbol{x}_k}).
    \end{equation*}

    The fourth moment bound satisfies
    \begin{equation*}
        \mathbb{E}[\norm{\hat{\boldsymbol{x}}_1^{(k)}}^4] \leq \mathbb{E}[(1+B_u)^4(1+\norm{\boldsymbol{x}_k})^4] \leq 8(1+B_u)^4(1+\mathbb{E}\norm{\boldsymbol{x}_k}^4).
    \end{equation*}

    Thus,
    \begin{equation*}
        \mathbb{E}[\norm{\hat{\boldsymbol{x}}_1^{(k)}}^4] \leq 8(1+B_u)^4 (1+M_4) =: M_{\hat{x}, 4}.
    \end{equation*}

    For part (c), 
    By $L_{\mathcal{L}}$-smoothness, $\mathcal{L}(\boldsymbol{x}^*) = 0$ and $\nabla \mathcal{L}(\boldsymbol{x}^*) = 0$, we have
    \begin{equation*}
        \begin{aligned}
            \mathcal{L}(\hat{\boldsymbol{x}}_1^{(k)}) &\leq \mathcal{L}(\boldsymbol{x}^*) + \inner{\nabla\mathcal{L}(\boldsymbol{x}^*)}{\hat{\boldsymbol{x}}_1^{(k)} - \boldsymbol{x}^* } + \dfrac{L_{\mathcal{L}}}{2} \| \hat{\boldsymbol{x}}_1^{(k)} - \boldsymbol{x}^* \|^2 = \dfrac{L_{\mathcal{L}}}{2} \| \hat{\boldsymbol{x}}_1^{(k)} - \boldsymbol{x}^* \|^2 \\
            &\leq \dfrac{L_{\mathcal{L}}}{2} \left( \|\hat{\boldsymbol{x}}_1^{(k)}\|^2 + M_{x^*}^2 - 2 \langle \hat{\boldsymbol{x}}_1^{(k)}, \boldsymbol{x}^* \rangle \right) \\
            &\leq \dfrac{L_{\mathcal{L}}}{2} \left( \|\hat{\boldsymbol{x}}_1^{(k)}\|^2 + M_{x^*}^2 + 2 | \|\hat{\boldsymbol{x}}_1^{(k)}\| \|\boldsymbol{x}^*\| | \right) \\
            &\leq \dfrac{L_{\mathcal{L}}}{2} \left( \|\hat{\boldsymbol{x}}_1^{(k)}\|^2 + M_{x^*}^2 + 2 M_{x^*} \|\hat{\boldsymbol{x}}_1^{(k)}\| \right) \\
        \end{aligned}
    \end{equation*}

    Taking expectations gives
    \begin{equation*}
        \begin{aligned}
            V_k = \mathbb{E} [\mathcal{L}(\hat{\boldsymbol{x}}_1^{(k)})] &\leq \dfrac{L_{\mathcal{L}}}{2} \left( \mathbb{E} [\|\hat{\boldsymbol{x}}_1^{(k)}\|^2] + M_{x^*}^2 + 2 M_{x^*}\mathbb{E} [\|\hat{\boldsymbol{x}}_1^{(k)}\|] \right) \\
            &\leq \dfrac{L_{\mathcal{L}}}{2} \left( M_{\hat{x},2} + M_{x^*}^2 + 2 M_{x^*} \sqrt{M_{\hat{x},2}} \right) =: M_{loss}.
        \end{aligned}
    \end{equation*}
\end{proof}

With uniform bounds established, we now analyze the algorithm's convergence. Before this, we give a definition of phase transition. Fix $\epsilon_s \in (0, \epsilon_0)$ where $\epsilon_0$ is from Assumption~\ref{ass:jacobian}. Define
\begin{equation*}
    K_{trans} := \min\{k \in \{0,\ldots,N\} : t_k \ge 1 - \epsilon_s\}.
\end{equation*}

We call $\{0, \ldots, K_{trans}-1\}$ Phase 1, which is the high noise regime, and $\{K_{trans}, \ldots, N\}$ Phase 2, which is the gradient descent regime. Phase 1 is covered by Corollary~\ref{cor:derived}, we have $V_{K_{trans}} \le M_{loss}$. The remainder of the analysis therefore focuses on Phase 2.

\begin{lemma}[One-Step Descent Structure]\label{lem:onestep}
    For $k \ge K_{trans}$, define:
    \begin{align*}
        \boldsymbol{d}_{k+1} &:= \mathbb{E}_k[\boldsymbol{x}_{k+1}] = t_{k+1}\hat{\boldsymbol{x}}_1^{(k)} - c_\zeta \delta_k J_{t_k}(\boldsymbol{x}_k)^\top\bar{\boldsymbol{g}}_k, \\
        M_k &:= J_{t_{k+1}}(\boldsymbol{x}_k)J_{t_k}(\boldsymbol{x}_k)^\top.
    \end{align*}
    Then:
    \begin{equation*}
        \mathbb{E}_k[\hat{\boldsymbol{x}}_1^{(k+1)}] - \hat{\boldsymbol{x}}_1^{(k)} = -c_\zeta \delta_k M_k\bar{\boldsymbol{g}}_k + \boldsymbol{E}_k,
    \end{equation*}
    where the error $\boldsymbol{E}_k$ is $\mathcal{F}_k$-measurable and satisfies
    \begin{equation}\label{eq:error_bound_random}
        \norm{\boldsymbol{E}_k} \le C_{E,1}\delta_{k}^2(1+\norm{\boldsymbol{x}_k})^2 + C_{E,2} c_{\zeta}^2 \delta_{k}^2(1+\norm{\boldsymbol{x}_k})^2 + C_{E,3}\delta_{k}(1+\norm{\boldsymbol{x}_k})
    \end{equation}
    for constants $C_{E,i}$ depending on $(d, H_\Phi, L_\Phi, B_u, L_t, G_0)$ but not on $k$ or $N$.
\end{lemma}

\begin{proof}
    Since $\boldsymbol{x}_{k+1} \mid \mathcal{F}_k \sim \mathcal{N}(\boldsymbol{d}_{k+1}, \delta_{k+1}^2 I)$, Lemma~\ref{lem:conditional} and Lemma~\ref{lem:jensen} give
    \begin{equation*}
        \mathbb{E}_k[\hat{\boldsymbol{x}}_1^{(k+1)}] = \mathbb{E}_k[\Phi_{t_{k+1}}(\boldsymbol{x}_{k+1})] = \Phi_{t_{k+1}}(\boldsymbol{d}_{k+1}) + \boldsymbol{B}_{k+1},
    \end{equation*}
    where $\norm{\boldsymbol{B}_{k+1}} \le \frac{H_\Phi d\delta_{k+1}^2}{2}$. Decompose the spatial and temporal changes
    \begin{align*}
        \Phi_{t_{k+1}}(\boldsymbol{d}_{k+1}) - \hat{\boldsymbol{x}}_1^{(k)}
        &= \Phi_{t_{k+1}}(\boldsymbol{d}_{k+1}) - \Phi_{t_{k+1}}(\boldsymbol{x}_k) + \Phi_{t_{k+1}}(\boldsymbol{x}_k) - \Phi_{t_k}(\boldsymbol{x}_k) \\
        &=: \boldsymbol{S}_k + \boldsymbol{T}_k.
    \end{align*}

    Taylor expansion around $\boldsymbol{x}_k$ with respect to $\boldsymbol{d}_{k+1}$ gives
    \begin{equation*}
        \boldsymbol{S}_k = \Phi_{t_{k+1}}(\boldsymbol{d}_{k+1}) - \Phi_{t_{k+1}}(\boldsymbol{x}_k) = J_{t_{k+1}}(\boldsymbol{x}_k)(\boldsymbol{d}_{k+1} - \boldsymbol{x}_k) + \boldsymbol{E}_{spatial},
    \end{equation*}
    where the quadratic error satisfies
    \begin{equation*}
        \norm{\boldsymbol{E}_{spatial}} \le \frac{H_\Phi}{2}\norm{\boldsymbol{d}_{k+1}-\boldsymbol{x}_k}^2.
    \end{equation*}

    Decompose $\boldsymbol{d}_{k+1} - \boldsymbol{x}_k$:
    \begin{align*}
        \boldsymbol{d}_{k+1} - \boldsymbol{x}_k
        &= t_{k+1}\hat{\boldsymbol{x}}_1^{(k)} - c_\zeta \delta_k J_{t_k}^\top\bar{\boldsymbol{g}}_k - \boldsymbol{x}_k \nonumber\\
        &= -\delta_{k+1}\boldsymbol{x}_k + t_{k+1}\delta_k u_{t_k}(\boldsymbol{x}_k) - c_\zeta \delta_k J_{t_k}^\top\bar{\boldsymbol{g}}_k \nonumber\\
        &=: \boldsymbol{r}_k - c_{\zeta} \delta_k J_{t_k}^\top\bar{\boldsymbol{g}}_k,
    \end{align*}
    where $\boldsymbol{r}_k := -\delta_{k+1}\boldsymbol{x}_k + t_{k+1}\delta_k u_{t_k}(\boldsymbol{x}_k)$. The linear part gives:
    \begin{equation*}
        J_{t_{k+1}}(\boldsymbol{x}_k)(\boldsymbol{d}_{k+1} - \boldsymbol{x}_k) = J_{t_{k+1}}(\boldsymbol{x}_k)\boldsymbol{r}_k - c_{\zeta} \delta_k M_k\bar{\boldsymbol{g}}_k.
    \end{equation*}

    Using Assumption~\ref{ass:velocity}(a):
    \begin{align*}
        \norm{\boldsymbol{r}_k}
        &\le \delta_{k+1}\norm{\boldsymbol{x}_k} + \delta_k B_u(1+\norm{\boldsymbol{x}_k}) \nonumber\\
        &\le \delta_k[(1+B_u)(1+\norm{\boldsymbol{x}_k})] \quad \text{(since } \delta_{k+1} \le \delta_k\text{)} \nonumber\\
        &=: C_r \delta_k(1+\norm{\boldsymbol{x}_k}), 
    \end{align*}
    where $C_r := 1 + B_u$. Thus,
    \begin{equation*}
        \norm{J_{t_{k+1}}(\boldsymbol{x}_k)\boldsymbol{r}_k} \le L_\Phi C_r \delta_k(1+\norm{\boldsymbol{x}_k}).
    \end{equation*}

    Then we have
    \begin{equation*}
        \norm{\boldsymbol{d}_{k+1} - \boldsymbol{x}_k}^2 \le 2C_r^2\delta_k^2(1+\norm{\boldsymbol{x}_k})^2 + 2 c_{\zeta}^2 \delta_k^2 L_\Phi^2\norm{\bar{\boldsymbol{g}}_k}^2.
    \end{equation*}

    Using $\norm{\bar{\boldsymbol{g}}_k} \le G_0(1+\norm{\hat{\boldsymbol{x}}_1^{(k)}}) \le G_0(1+B_u)(1+\norm{\boldsymbol{x}_k})$:
    \begin{align*}
        \norm{\boldsymbol{d}_{k+1} - \boldsymbol{x}_k}^2
        &\le 2C_r^2\delta_k^2(1+\norm{\boldsymbol{x}_k})^2 + 2 c_{\zeta}^2 \delta_k^2 L_\Phi^2 G_0^2(1+B_u)^2(1+\norm{\boldsymbol{x}_k})^2 \\
        &\le C_d^{\prime}\delta_k^2(1+\norm{\boldsymbol{x}_k})^2 + C_d^{\prime}c_{\zeta}^2\delta_k^2(1+\norm{\boldsymbol{x}_k})^2,
    \end{align*}
    where $C'_d := 2\max\{C_r^2, L_\Phi^2 G_0^2(1+B_u)^2\}$. Therefore:
    \begin{equation*}
        \norm{\boldsymbol{E}_{spatial}} \le \frac{H_\Phi C'_d}{2} \delta_k^2 (1+\norm{\boldsymbol{x}_k})^2 + \frac{H_\Phi C'_d}{2} c_{\zeta}^2 \delta_k^2 (1+\norm{\boldsymbol{x}_k})^2.
    \end{equation*}

    For the temporal part $\boldsymbol{T}_k$, by Lemma~\ref{lem:temporal} and Assumption~\ref{ass:time_grid}(b) gives
    \begin{equation*}
        \norm{\boldsymbol{T}_k} = \norm{\Phi_{t_{k+1}}(\boldsymbol{x}_k) - \Phi_{t_k}(\boldsymbol{x}_k)} \le C_t(1+\norm{\boldsymbol{x}_k})\Delta t_k \le C_t(1+\norm{\boldsymbol{x}_k}) c_{\Delta} \delta_k.
    \end{equation*}

    Combine all terms, we have:
    \begin{align*}
        \mathbb{E}_k[\hat{\boldsymbol{x}}_1^{(k+1)}] - \hat{\boldsymbol{x}}_1^{(k)}
        &= J_{t_{k+1}}\boldsymbol{r}_k - c_{\zeta} \delta_k M_k\bar{\boldsymbol{g}}_k + \boldsymbol{E}_{spatial} + \boldsymbol{T}_k + \boldsymbol{B}_{k+1} \nonumber\\
        &= -c_{\zeta} \delta_k M_k\bar{\boldsymbol{g}}_k + \boldsymbol{E}_k,
    \end{align*}
    where $\boldsymbol{E}_k := J_{t_{k+1}}\boldsymbol{r}_k + \boldsymbol{E}_{spatial} + \boldsymbol{T}_k + \boldsymbol{B}_{k+1}$ satisfies
    \begin{equation*}
        \begin{aligned}
            \norm{\boldsymbol{E}_k}
            &\leq L_\Phi C_r\delta_k(1+\norm{\boldsymbol{x}_k}) + \frac{H_\Phi C'_d}{2}\delta_k^2(1+\norm{\boldsymbol{x}_k})^2 + \frac{H_\Phi C'_d}{2} c_{\zeta}^2 \delta_k^2 (1+\norm{\boldsymbol{x}_k})^2 \\
            &\quad + C_t(1+\norm{\boldsymbol{x}_k})c_{\Delta} \delta_{k} + \frac{H_\Phi d\delta_{k+1}^2}{2} + \frac{H_\Phi C'_d}{2} c_{\zeta}^2 \delta_k^2 (1+\norm{\boldsymbol{x}_k})^2 \\
            &\leq \frac{H_\Phi C'_d}{2} \delta_k^2 (1+\norm{\boldsymbol{x}_k})^2 + (L_\Phi C_r + c_\Delta C_t) \delta_{k}(1+\norm{\boldsymbol{x}_k}) \\
            &\quad + \frac{H_\Phi d}{2} \delta_{k}^2 (1+\norm{\boldsymbol{x}_k})^2 + \frac{H_\Phi C'_d}{2} c_{\zeta}^2 \delta_k^2 (1+\norm{\boldsymbol{x}_k})^2 \\
            &\leq \left( \frac{H_\Phi C'_d}{2} + \frac{H_\Phi d}{2} \right)\delta_k^2 (1+\norm{\boldsymbol{x}_k})^2 + \frac{H_\Phi C'_d}{2} c_{\zeta}^2 \delta_k^2 (1+\norm{\boldsymbol{x}_k})^2 \\
            &\quad + (L_\Phi C_r + c_\Delta C_t) \delta_{k}(1+\norm{\boldsymbol{x}_k}) \\
            &:= C_{E,1} \delta_k^2 (1+\norm{\boldsymbol{x}_k})^2 + C_{E,2} c_{\zeta}^2 \delta_k^2 (1+\norm{\boldsymbol{x}_k})^2 + C_{E,3} \delta_{k}(1+\norm{\boldsymbol{x}_k}).
        \end{aligned}
    \end{equation*}

    This completes the proof.
\end{proof}

We now state the main convergence result for the stochastic sampler.
\begin{theorem}[Stochastic Convergence Analysis with Adaptive Guidance]\label{thm:stoc}
    Under \Cref{ass:velocity,ass:dissipative,ass:loss,ass:jacobian,ass:time_grid}, let $\zeta_k = c_\zeta\delta_k$ with $c_\zeta$ satisfying the stability condition~\eqref{eq:czeta_stability}. Define $\bar\zeta := c_\zeta\delta_{\min}$, which is the terminal guidance strength. Define the second-order absorption threshold:
    \begin{equation*}
        c_{\zeta,1} := \frac{1}{\delta_{\max}}\min\left\{\frac{\lambda_{\min}}{16L_{\mathcal{L}} L_\Phi^4}, \frac{16}{15 \mu c_0 \lambda_{\min}}, \delta_{\max}\right\}, 
    \end{equation*}
    where $\delta_{\max} := \max_k \delta_k = \delta_0 = 1 - t_0$ is the initial noise level. Assume $c_\zeta \le c_{\zeta,1}$ and~\eqref{eq:czeta_stability} throughout. Set $\epsilon_s = c_0\bar\zeta$ with $c_0 > 2/c_\zeta$ (ensuring Phase~2 is non-empty; this is a constraint on $c_0$ alone) and $\bar\zeta \le \min\{1, \epsilon_0/c_0\}$. If the number of steps of Phase 2 $N_2 \ge \frac{1}{\bar\rho}\log \left( \frac{2M_{loss}}{\bar{\zeta}} \right)$, then
    \begin{equation*}
        V_N \le C_{final}\bar\zeta + C_{floor},
    \end{equation*}
    where
    \begin{equation*}
        C_{final} := \frac{1}{2} + \frac{16\tilde{C}_2 c_0^2}{15\mu\lambda_{\min}} + \frac{128 \hat{C}_a c_0^3}{225 \mu c_{\zeta} \lambda_{\min}^2}
    \end{equation*}
    depends only on problem parameters, and
    \begin{equation*}
        C_{floor} := \frac{512C_{E,3}^2(1+M_2)c_0}{225\mu c_\zeta\lambda_{\min}^2}
    \end{equation*}
    is an irreducible constant reflecting the tension between $O(\delta_k)$ first-order discretization error and $O(\delta_k)$ contraction. For target accuracy $\varepsilon > C_{floor}$, the complexity is $N = O(\varepsilon^{-2}\log(1/\varepsilon))$.

\end{theorem}

\begin{proof}
    Phase 1 is bounded by Corollary~\ref{cor:derived}(c):
    \begin{equation*}
        V_{K_{trans}} \le M_{loss}.
    \end{equation*}

    We then analyze the Phase 2 refined bound. By Lemma~\ref{lem:moment}, $\mathbb{E}[\norm{\boldsymbol{x}_k}^2] \le M_2$ for all $k$, under the stability condition~\eqref{eq:czeta_stability} on $c_\zeta$. The fourth moment bound gives $M_4$. Furthermore, Phase 2 requires $\epsilon_s > \delta_{\min}$, thus, $c_0 c_\zeta \delta_{\min} > \delta_{\min}$, i.e., $c_0 c_\zeta > 1$. The Phase~2 step count requires at least half the phase interval, giving $c_0 c_\zeta > 2$. For $k \ge K_{trans}$, we have $\delta_{k+1} \le \epsilon_s = c_0\bar\zeta = c_0 c_\zeta \delta_{\min}$. By $L_{\mathcal{L}}$-smoothness (\Cref{ass:loss}(a)):
    \begin{equation*}
        \begin{aligned}
            \mathbb{E}_k[\mathcal{L}(\hat{\boldsymbol{x}}_1^{(k+1)})] &\le \mathcal{L}(\hat{\boldsymbol{x}}_1^{(k)}) + \inner{\bar{\boldsymbol{g}}_k}{\mathbb{E}_k[\hat{\boldsymbol{x}}_1^{(k+1)} - \hat{\boldsymbol{x}}_1^{(k)}]} + \frac{L_{\mathcal{L}}}{2}\mathbb{E}_k[\norm{\hat{\boldsymbol{x}}_1^{(k+1)} - \hat{\boldsymbol{x}}_1^{(k)}}^2] \\
            &:= \mathcal{L}(\hat{\boldsymbol{x}}_1^{(k)}) + \inner{\bar{\boldsymbol{g}}_k}{\mathbb{E}_k[\Delta\hat{\boldsymbol{x}}_1]} + \frac{L_{\mathcal{L}}}{2}\mathbb{E}_k[\norm{\Delta\hat{\boldsymbol{x}}_1}^2].
        \end{aligned}
    \end{equation*}

    Note that $\mathbb{E}_k[\Delta\hat{\boldsymbol{x}}_1] = \mathbb{E}_k[\hat{\boldsymbol{x}}_1^{(k+1)}] - \hat{\boldsymbol{x}}_1^{(k)}$, then by Lemma~\ref{lem:onestep}, we have
    \begin{equation*}
        \inner{\bar{\boldsymbol{g}}_k}{\mathbb{E}_k[\Delta\hat{\boldsymbol{x}}_1]} = -c_\zeta \delta_k\bar{\boldsymbol{g}}_k^\top M_k\bar{\boldsymbol{g}}_k + \inner{\bar{\boldsymbol{g}}_k}{\boldsymbol{E}_k}.
    \end{equation*}

    By Assumption~\ref{ass:jacobian}, since $k \ge K_{trans}$, $t_k, t_{k+1} \in [1-\epsilon, 1] \subseteq [1-\epsilon_0, 1]$, we have
    %
    \begin{equation*}
        -c_\zeta \delta_k\bar{\boldsymbol{g}}_k^\top M_k\bar{\boldsymbol{g}}_k \le -c_\zeta \delta_k\lambda_{\min}\norm{\bar{\boldsymbol{g}}_k}^2.
    \end{equation*}

    Taking full expectations yields
    \begin{align}
        V_{k+1} &\le V_k - c_{\zeta} \delta_k \lambda_{\min}\mathbb{E}[\norm{\bar{\boldsymbol{g}}_k}^2] + \mathbb{E}[|\inner{\bar{\boldsymbol{g}}_k}{\boldsymbol{E}_k}|] + \frac{L_{\mathcal{L}}}{2}\mathbb{E}[\norm{\Delta\hat{\boldsymbol{x}}_1}^2]. \label{eq:Vk_recursion}
    \end{align}

    Applying Cauchy--Schwarz and Young's inequality with parameter $\alpha = \frac{15 c_\zeta\delta_k\lambda_{\min}}{32}$, the error--gradient interaction satisfies
    \begin{align*}
        \mathbb{E}[|\inner{\bar{\boldsymbol{g}}_k}{\boldsymbol{E}_k}|]
        &\le \frac{15 c_\zeta\delta_k\lambda_{\min}}{32}\mathbb{E}[\norm{\bar{\boldsymbol{g}}_k}^2] + \frac{8\hat{C}_a\delta_k^3}{15\lambda_{\min} c_\zeta} + \frac{8\hat{C}_b\delta_{k}}{15 c_\zeta \lambda_{\min}}.
    \end{align*}

    By the variance decomposition:
    \begin{equation*}
        \mathbb{E}_k[\norm{\Delta\hat{\boldsymbol{x}}_1}^2] = \mathbb{E}_k[\norm{\Delta\hat{\boldsymbol{x}}_1 - \mathbb{E}_k[\Delta\hat{\boldsymbol{x}}_1]}^2] + \norm{\mathbb{E}_k[\Delta\hat{\boldsymbol{x}}_1]}^2.
    \end{equation*}

    For the variance term:
    \begin{equation*}
        \begin{aligned}
            \mathbb{E}_k[\norm{\Delta\hat{\boldsymbol{x}}_1 - \mathbb{E}_k[\Delta\hat{\boldsymbol{x}}_1]}^2]
            &= \mathbb{E}_k[\norm{\hat{\boldsymbol{x}}_1^{(k+1)} - \mathbb{E}_k[\hat{\boldsymbol{x}}_1^{(k+1)}]}^2] \\
            &\le L_\Phi^2 \mathbb{E}_k[\norm{\boldsymbol{x}_{k+1} - \boldsymbol{d}_{k+1}}^2] = L_\Phi^2 d\delta_{k+1}^2.
        \end{aligned}
    \end{equation*}

    For the squared mean term:
    \begin{align*}
        \norm{\mathbb{E}_k[\Delta\hat{\boldsymbol{x}}_1]}^2
        &= \norm{-c_{\zeta} M_k\bar{\boldsymbol{g}}_k + \boldsymbol{E}_k}^2 \le 2 c_{\zeta}^2 \delta_k^2 L_\Phi^4\norm{\bar{\boldsymbol{g}}_k}^2 + 2\norm{\boldsymbol{E}_k}^2.
    \end{align*}

    From~\eqref{eq:error_bound_random}:
    \begin{equation*}
        \begin{aligned}
            \norm{\boldsymbol{E}_k}
            &\le (C_{E,1} + C_{E,2} c_{\zeta}^2) \delta_k^2(1+\norm{\boldsymbol{x}_k})^2 + C_{E,3} \delta_{k}(1+\norm{\boldsymbol{x}_k}) \\
            &:= \tilde{C}_a \delta_k^2(1+\norm{\boldsymbol{x}_k})^2 + \tilde{C}_b \delta_{k}(1+\norm{\boldsymbol{x}_k}).
        \end{aligned}
    \end{equation*}

    Taking expectations using Corollary~\ref{cor:derived}(d):
    \begin{equation*}
        \mathbb{E}[\norm{\boldsymbol{E}_k}^2] \le 2 \tilde{C}_{a}^2 \delta_k^4 \cdot 8(1+M_4) + 2\tilde{C}_{b}^2 \delta_{k}^2 \cdot 2(1+M_2) =: \hat{C}_a \delta_k^4 + \hat{C}_b \delta_{k}^2,
    \end{equation*}
    where $\hat{C}_a = 16 \tilde{C}_a^2 (1+M_4)$, $\hat{C}_b = 4\tilde{C}_{b}^2(1+M_2)$. Thus:
    \begin{equation*}
        \frac{L_{\mathcal{L}}}{2}\mathbb{E}[\norm{\Delta\hat{\boldsymbol{x}}_1}^2] \le \frac{L_{\mathcal{L}} L_\Phi^2 d}{2}\delta_{k+1}^2 + L_{\mathcal{L}}c_{\zeta}^2 \delta_k^2 L_\Phi^4\mathbb{E}[\norm{\bar{\boldsymbol{g}}_k}^2] + L_{\mathcal{L}} \left( \hat{C}_a \delta_k^4 + \hat{C}_b \delta_k^2 \right)
    \end{equation*}

    By the condition $c_{\zeta} \leq \frac{\lambda_{\min}}{16L_{\mathcal{L}} L_\Phi^4}$, we have $L_{\mathcal{L}} c_{\zeta}^2 L_\Phi^4 \le L_{\mathcal{L}} \cdot c_{\zeta} \cdot \frac{\lambda_{\min}}{16L_{\mathcal{L}} L_\Phi^4} \cdot L_\Phi^4 =   \frac{c_{\zeta} \lambda_{\min}}{16}$. Substituting into Equation~\eqref{eq:Vk_recursion}:
    \begin{equation*}
        \begin{aligned}
            V_{k+1} 
            &\le V_k - c_{\zeta} \delta_k \lambda_{\min}\mathbb{E}[\norm{\bar{\boldsymbol{g}}_k}^2] + \frac{15 c_\zeta\delta_k\lambda_{\min}}{32}\mathbb{E}[\norm{\bar{\boldsymbol{g}}_k}^2] + \frac{8\hat{C}_a\delta_k^3}{15\lambda_{\min} c_\zeta} + \frac{8\hat{C}_b\delta_{k}}{15 c_\zeta \lambda_{\min}} \\
            &\quad + \frac{L_{\mathcal{L}} L_\Phi^2 d}{2}\delta_{k+1}^2 + \frac{c_{\zeta}\lambda_{\min}}{16}\mathbb{E}[\norm{\bar{\boldsymbol{g}}_k}^2] + L_{\mathcal{L}} \left( \hat{C}_a \delta_k^4 + \hat{C}_b \delta_k^2 \right) \\
            &\leq V_k - \dfrac{15 c_{\zeta} \delta_k \lambda_{\min}}{32} \mathbb{E}[\|\bar{\boldsymbol{g}}_k\|^2] + \frac{8\hat{C}_a\delta_k^3}{15\lambda_{\min} c_\zeta} + \frac{8\hat{C}_b\delta_{k}}{15 c_\zeta \lambda_{\min}} + \tilde{C}_2 \delta_k^2,
        \end{aligned}
    \end{equation*}
    where
    \begin{equation*}
        \tilde{C}_2 := \dfrac{L_{\mathcal{L}} L_{\Phi}^2 d}{2} + L_{\mathcal{L}} \left( 16 (C_{E,1} + C_{E,2} c_{\zeta, 1})^2 (1+M_4) + \hat{C}_b^2 \right).
    \end{equation*}

    By the PL condition:
    \begin{equation}\label{eq:V_recursion_parta}
        V_{k+1} \le (1-\rho_k)V_k + \frac{8\hat{C}_b\delta_k}{15 c_\zeta\lambda_{\min}} + \tilde{C}_2\delta_k^2 + \frac{8\hat{C}_a\delta_k^3}{15 c_\zeta\lambda_{\min}},
    \end{equation}
    where $\rho_k = \frac{15\mu c_\zeta\delta_k\lambda_{\min}}{16}$. Note that $\rho_k < 1$ is guaranteed for all Phase~2 steps: since $\delta_k \le \epsilon_s = c_0\bar\zeta \le c_0$ and $\bar\zeta \le 1$, we have $\rho_k \le \frac{15\mu c_\zeta c_0\lambda_{\min}}{16} < 1$ whenever $c_\zeta c_0 < \frac{16}{15\mu\lambda_{\min}}$. This condition is compatible with the stability condition~\eqref{eq:czeta_stability}: both impose upper bounds on $c_\zeta$, and we require the parameter regime to satisfy both simultaneously.

    In Phase 2, $\delta_k \le \epsilon_s = c_0\bar\zeta$. Summing the geometric series:
    \begin{align*}
        V_{K_{trans}+N_2}
        &\le (1-\bar\rho)^{N_2} M_{loss} + \frac{128\hat{C}_b c_0}{225\mu c_\zeta\lambda_{\min}^2} + \frac{16\tilde{C}_2 c_0^2\bar\zeta}{15\mu\lambda_{\min}} + \frac{128\hat{C}_a c_0^3 \bar{\zeta}^2}{225\mu c_\zeta\lambda_{\min}^2},
    \end{align*}
    where $\bar\rho = \frac{15\mu\bar\zeta\lambda_{\min}}{16}$. Setting $N_2 \ge \frac{1}{\bar\rho}\log \left( \frac{2M_{loss}}{\bar{\zeta}} \right)$:
    \begin{align*}
        V_N &\le \frac{\bar\zeta}{2} + \frac{16\tilde{C}_2 c_0^2\bar\zeta}{15\mu\lambda_{\min}} + \frac{128\hat{C}_a c_0^3 \bar{\zeta}^2}{225\mu c_\zeta\lambda_{\min}^2} + \frac{128\hat{C}_b c_0}{225\mu c_\zeta\lambda_{\min}^2} \le C_{final} \bar\zeta + C_{floor},
    \end{align*}
    where $C_{final} := \frac{1}{2} + \frac{16 \tilde{C}_2 c_0^2}{15 \mu \lambda_{\min}} + \frac{128 \hat{C}_a c_0^3}{225 \mu c_{\zeta} \lambda_{min}^2}$ and $C_{floor} := \frac{128\hat{C}_b c_0}{225\mu c_\zeta\lambda_{\min}^2}$.

    For any target accuracy $\varepsilon > C_{floor}$, setting $\bar\zeta = (\varepsilon - C_{floor})/C_{final}'$ yields
    \begin{equation*}
        N = O \left(\frac{\log(1/\varepsilon)}{\varepsilon^2}\right).
    \end{equation*}
\end{proof}

\begin{remark}[Trade-off in Guidance Strength $c_\zeta$]\label{rem:cfloor_czeta_tradeoff}
    The floor constant $C_{floor}$ depends inversely on $c_{\zeta}$: smaller guidance strength increases the floor. This reflects a fundamental trade-off:
    \begin{itemize}[topsep=2pt,itemsep=1pt]
        \item Smaller $c_\zeta$: Weaker guidance reduces the systematic bias from guidance-induced drift, but increases the relative impact of first-order discretization error $O(\delta_k)$, which cannot be absorbed by the weaker contraction.
        \item Larger $c_\zeta$: Stronger guidance provides faster contraction (larger $\rho_k$), reducing the floor, but must satisfy the stability condition~\eqref{eq:czeta_stability} to prevent moment blowup.
    \end{itemize}
\end{remark}

The optimal choice of $c_\zeta$ balances these effects. In practice, $c_\zeta$ should be chosen as large as possible subject to the stability condition~\eqref{eq:czeta_stability}, which depends on problem parameters $(\kappa', L_\Phi, G_0, B_u, \bar\delta_{ratio})$.

Importantly, the total error $V_N \le C_{final}\bar\zeta + C_{floor}$ depends on $c_\zeta$ through both terms:
\begin{equation*}
    V_N \le C_{final} c_\zeta\delta_{\min} + \frac{C'}{c_\zeta},
\end{equation*}
where $C' := (512C_{E,3}^2(1+M_2)c_0)/(225\mu\lambda_{\min}^2)$ is independent of $c_\zeta$. Since $C_{floor} = C'/c_\zeta$, the floor diverges as $c_\zeta \to 0$: weaker guidance cannot compensate for the discretization error. Conversely, the guidance term $C_{final} c_\zeta\delta_{\min}$ grows with $c_\zeta$. The minimum of $V_N$ over $c_\zeta$ (ignoring the stability constraint) is achieved at $c_\zeta^* = \sqrt{C'/(C_{final} \delta_{\min})}$, yielding $V_N^* = 2\sqrt{C_{final} C'\delta_{\min}}$. In practice, $c_\zeta$ is bounded above by the stability condition~\eqref{eq:czeta_stability}, so one should choose $c_\zeta$ as large as permitted.

Furthermore, under a prediction consistency condition (Remark~\ref{rem:prediction_consistency_reduction}), the floor reduces to $O(\epsilon_{pc}^2/c_\zeta)$, retaining the same $1/c_\zeta$ dependence but with the smaller coefficient $\epsilon_{pc}^2$ replacing $4C_{E,3}^2(1+M_2)$. In particular, the floor vanishes when $\epsilon_{pc} = 0$ regardless of $c_\zeta$.

\begin{remark}[Reduced Floor under Prediction Consistency]\label{rem:prediction_consistency_reduction}
    The floor $C_{floor}$ in Theorem~\ref{thm:stoc} arises from the generic first-order prediction error bound $\norm{\boldsymbol{E}_k} \le C_{E,3}\delta_{k+1}(1+\norm{\boldsymbol{x}_k})$, which contributes $\hat{C}_b = 4C_{E,3}^2(1+M_2)$ to the recursion~\eqref{eq:V_recursion_parta}. If the prediction network satisfies a prediction consistency condition---namely, $\mathbb{E}[\norm{\boldsymbol{E}_k}^2 \mid \mathcal{F}_k] \le \epsilon_{pc}^2\delta_{k+1}^2$ for some $\epsilon_{pc} \ge 0$---then $\hat{C}_b$ is replaced by $\epsilon_{pc}^2$ in the recursion, and the floor reduces to:
    \begin{equation*}
        C_{floor}^{(pc)} = \frac{128\epsilon_{pc}^2 c_0}{225\mu c_\zeta\lambda_{\min}^2}.
    \end{equation*}
    When $\epsilon_{pc} = 0$ (perfectly consistent flow), $C_{floor}^{(pc)} = 0$ and the bound becomes $V_N \le C_{final}\bar\zeta$ with no constant floor. The proof is identical to that of Theorem~\ref{thm:stoc}, with $\hat{C}_b$ replaced by $\epsilon_{pc}^2$; all other steps, constants, and the complexity $O(\varepsilon^{-2}\log(1/\varepsilon))$ remain unchanged. This parallels the modular ``algorithm error + approximation error" decomposition in diffusion model convergence theory \citep{chen2022sampling,benton2023nearly}. 
\end{remark}

It's worth noting that in the above analysis, we analyze the expected guidance loss of $\boldsymbol{x}_1^{(k)}$, while the algorithm outputs $\boldsymbol{x}_{k+1}$ at each step. We can establish a relationship between the average losses of the two. From the algorithm update $\boldsymbol{x}_{k+1} = \delta_{k+1}\xi_k + \boldsymbol{m}_k$ where $\boldsymbol{m}_k = t_{k+1}\hat{\boldsymbol{x}}_1^{(k)} - c_{\zeta} \delta_k\boldsymbol{g}_k$ with $\boldsymbol{g}_k = J_{t_k}^{\top} \bar{\boldsymbol{g}}_k$ and the $L_{\mathcal{L}}$-smoothness of the guided loss, we have,
\begin{equation*}
    \begin{aligned}
        \mathbb{E}_k [\mathcal{L}( \boldsymbol{x}_{k+1})] &\le \mathcal L(\boldsymbol{m}_k) + \mathbb{E}_k [\langle \nabla \mathcal{L} (\boldsymbol{m}_k), \delta_{k+1}\xi_k \rangle] + \frac{L_{\mathcal L}}{2} \mathbb E_k [\|\delta_{k+1}\xi_k\|]^2 \\
        & \leq \mathcal L(\boldsymbol{m}_k)  + \frac{L_{\mathcal L}}{2} \delta_{k+1}^2 d  \\
        & \leq \mathcal L(\hat{\boldsymbol{x}}_1^{(k)}) +\langle \bar{\boldsymbol{g}}_k,\boldsymbol{m}_k-\hat{\boldsymbol{x}}_1^{(k)}\rangle +\frac{L_{\mathcal L}}{2}\|\boldsymbol{m}_k-\hat{\boldsymbol{x}}_1^{(k)}\|^2 + \frac{L_{\mathcal L}}{2} \delta_{k+1}^2 d \\
        &= \mathcal L(\hat{\boldsymbol{x}}_1^{(k)}) - \langle \bar{\boldsymbol{g}}_k, \delta_{k+1}\hat{\boldsymbol{x}}_1^{(k)} + c_{\zeta} \delta_k\boldsymbol{g}_k\rangle + \frac{L_{\mathcal L}}{2}\|\delta_{k+1}\hat{\boldsymbol{x}}_1^{(k)} + c_{\zeta} \delta_k\boldsymbol{g}_k\|^2 + \frac{L_{\mathcal L}}{2} \delta_{k+1}^2 d \\
        &\leq \mathcal L(\hat{\boldsymbol{x}}_1^{(k)}) + \delta_{k+1} \|\bar{\boldsymbol{g}}_k\| \|\hat{\boldsymbol{x}}_1^{(k)}\| + c_{\zeta} \delta_k \|J_{t_k}\| \|\bar{\boldsymbol{g}}_k\|^2 + L_{\mathcal{L}} \delta_{k+1}^2 \|\hat{\boldsymbol{x}}_1^{(k)}\|^2 \\
        &\quad + L_{\mathcal{L}} c_{\zeta}^2 \delta_k^2 \|J_{t_k}\|^2 \|\bar{\boldsymbol{g}}_k\|^2 + \frac{L_{\mathcal L}}{2} \delta_{k+1}^2 d.
    \end{aligned}
\end{equation*}

Taking the full expectation, the upper bound in Phase 1 is established by \Cref{ass:loss} and \Cref{cor:derived}; whereas in Phase 2, the bound is governed by the condition $\delta_k < \epsilon_s = c_0 \bar{\zeta}$, in conjunction with \Cref{ass:loss} and \Cref{cor:derived}. Thus, we finally obtain
\begin{equation*}
    \mathbb{E} [\mathcal{L}( \boldsymbol{x}_{N})] \leq V_{N-1} + C_{final}^{\prime} \bar{\zeta}.
\end{equation*}
with $C_{final}^{\prime}$ depending on $G_0, L_{\Phi}, c_0, M_{\hat{x},2}, c_{\zeta}$. This derivation does not introduce an additional irreducible floor.

The upper bound from Theorem~\ref{thm:stoc} features an irreducible floor $C_{floor} > 0$. While the adaptive guidance design mitigates this issue, a natural question is whether any algorithm using constant guidance strength can avoid a positive bias. We further investigate this issue in Appendix~\ref{apx:lower_bound} through a representative one-dimensional example. The analysis shows that incorporating guidance with constant strength is insufficient to fundamentally alter the lower-bound behavior.

\subsection{Hybrid Sampler Analysis}
The stochastic sampler and deterministic optimizer have complementary strengths. This section combines them into a hybrid framework that leverages the best of both.

\begin{theorem}[Hybrid Framework Convergence]\label{thm:hybrid}
    Consider Algorithm~\ref{alg:fm4pde} with the Deterministic $\to$ Stochastic ordering: deterministic phase for $t \in [\epsilon, t^*]$ where $t^* \in [t_*, 1-\epsilon_0]$, followed by stochastic phase for $t \in [t^*, 1 - \delta_{\min}]$ with adaptive guidance $\zeta_k = c_\zeta\delta_k$. Under Assumptions~\ref{ass:velocity}--\ref{ass:initial}, suppose:
    \begin{enumerate}[label=(\roman*)]
        \item The deterministic phase uses admissible step sizes (Definition~\ref{def:det_step_size}) and $G_c \ge \tilde{G}_0(1+R_{disc})$ (Theorem~\ref{thm:det_discrete}); the stochastic phase grid satisfies Assumption~\ref{ass:time_grid};
        \item The adaptive guidance coefficient $c_\zeta$ satisfies the stability condition~\eqref{eq:czeta_stability};
        \item $t^* \le 1 - \epsilon_0$ and $t^* < 1 - \delta_{\min}$ so that the stochastic phase is non-empty and the deterministic grid does not extend into the Assumption~\ref{ass:time_grid}(b) refinement region;
        \item $N_1 = O(\log(1/\epsilon))$ steps for the deterministic phase and $N_2$ sufficient steps for the stochastic phase to achieve the contraction specified in Theorem~\ref{thm:stoc}. Here $M_{loss}^{(hyb)}$ is the loss bound from Corollary~\ref{cor:derived}(c) with the modified moment constant $M_2^{(hyb)} := \max\{R_{disc}^2, M_2\}$, where $R_{disc}$ is the deterministic trajectory bound from Lemma~\ref{lem:det_discrete_traj} and $M_2$ is from Lemma~\ref{lem:moment}.
    \end{enumerate}
    Then the expected loss of the final endpoint prediction satisfies:
    \begin{equation}
    V_N := \mathbb{E}[\mathcal{L}(\hat{\boldsymbol{x}}_1^{(N)})] \le C_{final}^{(hyb)}\bar\zeta + C_{floor}^{(hyb)},
    \end{equation}
    where $\bar\zeta = c_\zeta\delta_{\min}$ and $C_{final}^{(hyb)}$, $C_{floor}^{(hyb)}$ depend on problem parameters (including $R_{disc}$ from Lemma~\ref{lem:det_discrete_traj}). The deterministic phase provides fast initial convergence to a loss basin, while the stochastic phase (using adaptive guidance) provides the final accuracy. Under a prediction consistency condition (Remark~\ref{rem:prediction_consistency_reduction}), the floor reduces to $O(\epsilon_{pc}^2)$.

    With total complexity:
    \begin{equation}
    N = O\left(\log(1/\epsilon) + \frac{1}{\bar\zeta^2}\log\frac{1}{\bar\zeta}\right).
    \end{equation}
    Setting $\bar\zeta = \varepsilon$ yields $N = O(\varepsilon^{-2}\log(1/\varepsilon))$.
\end{theorem}

\begin{proof}
    Phase I (Deterministic, $t \in [\epsilon, t^*]$). Applying the analysis of Theorem~\ref{thm:det_discrete} to the time interval $[\epsilon, t^*]$:
    \begin{itemize}
        \item Phase A ($t \in [\epsilon, t_*]$): Exponential contraction with rate $\epsilon^{2\mu}$.
        \item Phase B ($t \in [t_*, t^*]$): Linear accumulation.
    \end{itemize}

    Let $K^*$ denote the step index at which $t_{K^*} = t^*$. By the two-phase analysis of Theorem~\ref{thm:det_discrete}:
    \begin{equation}
        \mathcal{L}(\boldsymbol{x}_{K^*}) \le C_A \cdot \epsilon^{2\mu} \cdot \mathcal{L}(\boldsymbol{x}_0) + C'_A + C'_B,
    \end{equation}
    where $C'_A$ and $C'_B$ account for Phase A residual and Phase B linear accumulation, respectively.

    Phase II (Stochastic with adaptive guidance, $t \in [t^*, 1-\delta_{\min}]$). Starting from the deterministic output $\boldsymbol{x}_{K^*}$ with $\|\boldsymbol{x}_{K^*}\| \le R_{disc}$, we apply the stochastic analysis with adaptive guidance $\zeta_k = c_\zeta\delta_k$.

    Note that the deterministic phase measures $\mathcal{L}(\boldsymbol{x}_k)$ (current state loss), while the stochastic phase measures $V_k = \mathbb{E}[\mathcal{L}(\hat{\boldsymbol{x}}_1^{(k)})]$ (endpoint prediction loss). This transition requires no explicit conversion: the stochastic analysis uses only the trajectory bound $\|\boldsymbol{x}_{K^*}\| \le R_{disc}$ as its initial condition, and the Phase~1 bound (Corollary~\ref{cor:derived}(c)) provides $V_{K_{trans}} \le M_{loss}^{(hyb)}$ independently of the deterministic loss value.

    The moment bounds from Lemma~\ref{lem:moment} apply with the modified initial condition $\|\boldsymbol{x}_{K^*}\|^2 \le R_{disc}^2$, yielding $M_2^{(hyb)} := \max\{R_{disc}^2, M_2\}$. Crucially, these bounds are independent of $\delta_{\min}$ due to adaptive guidance.

    The convergence analysis of Theorem~\ref{thm:stoc} then applies directly to the stochastic phase. By Theorem~\ref{thm:stoc}:
    \begin{equation}
    V_N \le C_{final}^{(hyb)}\bar\zeta + C_{floor}^{(hyb)},
    \end{equation}
    where the constants $C_{final}^{(hyb)}$ and $C_{floor}^{(hyb)}$ are as in Theorem~\ref{thm:stoc} but with the hybrid moment constants replacing the standard ones. Under a prediction consistency condition (Remark~\ref{rem:prediction_consistency_reduction}), the floor reduces to $O(\epsilon_{pc}^2)$.

    For Complexity, the deterministic phase requires $N_1 = O(\log(1/\epsilon))$ steps (geometric grid from $\epsilon$ to $t^*$). The stochastic phase requires $N_2 = O(\bar\zeta^{-2}\log(1/\bar\zeta))$ steps by Theorem~\ref{thm:stoc}. Total: $N = N_1 + N_2 = O(\log(1/\epsilon) + \bar\zeta^{-2}\log(1/\bar\zeta))$. To achieve $V_N \le \varepsilon$, set $\bar\zeta = (\varepsilon - C_{floor}^{(hyb)})/C_{final}^{(hyb)}$, giving $N = O(\varepsilon^{-2}\log(1/\varepsilon))$.
\end{proof}

\begin{remark}[Optimal Transition Point]
    The choice of $t^*$ balances:
    \begin{itemize}
        \item Running the deterministic phase long enough to include the full Phase A region $[\epsilon, t_*]$, ensuring exponential convergence $\epsilon^{2\mu}$;
        \item Leaving room for the stochastic phase to provide exploration and moment bounds.
    \end{itemize}
    A natural choice is $t^* = t_*$ or slightly larger. This captures all of Phase A contraction, providing fast initial convergence, while maximizing the stochastic exploration phase. Optimizing $t^*$ for specific problem instances is an interesting direction.
\end{remark}

\begin{remark}[Alternative Orderings and Unified Framework]\label{rem:alternative_orderings}
    All hybrid variants can be described using a single transition point $t^* \in [\epsilon, T]$ where $T := 1 - \delta_{\min}$. Recall from the critical time that Phase A (contraction) occurs for $t \in [\epsilon, t_*]$ where $b_t > \beta_1$, and Phase B (accumulation) occurs for $t \in [t_*, T]$ where $b_t \le \beta_1$.

    \textbf{Ordering 1: Deterministic $\to$ Stochastic.}
    \begin{itemize}
        \item Deterministic phase: $t \in [\epsilon, t^*]$; Stochastic phase: $t \in [t^*, T]$.
        \item If $t^* \ge t_*$: The deterministic phase covers all of Phase A. Result: $\mathbb{E}[\mathcal{L}(\boldsymbol{x}_N)] \le C_{final}^{(hyb)}\bar\zeta + C_{floor}^{(hyb)}$ (Theorem~\ref{thm:hybrid}).
        \item If $t^* < t_*$: The contraction phase has been reduced, yielding a weaker rate.
        \item If $t^* = \epsilon$: Pure stochastic. Result: $C_{final}\bar\zeta + C_{floor}$ (Theorem~\ref{thm:stoc}).
        \item If $t^* = T$: Pure deterministic. Result: $O(\epsilon^{2\mu})$ (Theorem~\ref{thm:det_discrete}).
    \end{itemize}

    \textbf{Ordering 2: Stochastic $\to$ Deterministic.}
    \begin{itemize}
        \item Stochastic phase: $t \in [\epsilon, t^*]$; Deterministic phase: $t \in [t^*, T]$.
        \item The deterministic phase starts at $t = t^* > \epsilon$, so it cannot exploit the Phase A region $[\epsilon, t_*]$ where $b_t > \beta_1$.
        \item If $t^* \ge t_*$: The deterministic phase is entirely in Phase B. By Lemma~\ref{lem:det_one_step}(b), only linear accumulation occurs. The exponential rate $\epsilon^{2\mu}$ is not achieved.
        \item If $t^* < t_*$: The contraction integral is only on $(t^*, t_*)$, giving a contraction factor $(t^*/t_*)^{2\mu} = O(1)$, which is a constant, since $t^*$ is bounded away from $0$, instead of the vanishing factor $\epsilon^{2\mu}$.
    \end{itemize}
\end{remark}

\begin{remark}[When to Use Which]
    Based on the above analysis, regarding strategy selection, the decision should be tailored to the specific objective: Pure Stochastic methods are suitable for scenarios requiring distributional sampling or multi-modal exploration; conversely, when seeking a single mode under the conditions of PL, coercivity, and $\beta_1 < 1$, Pure Deterministic methods are preferred due to their logarithmic convergence complexity. The Hybrid strategy (Deterministic $\to$ Stochastic) combines the strengths of both approaches, utilizing a deterministic Phase A to rapidly locate the convergence basin within $O(\log(1/\varepsilon))$ steps, before switching to stochastic mode to enhance subsequent exploration and robustness.
\end{remark}

\section{Experiments}\label{sec:experiments}

In this section, we empirically demonstrate the efficacy of our proposed framework through a series of numerical experiments. We benchmark our method against established operator learning paradigms and contemporary generative models, evaluating performance across both forward and inverse problem settings. To ensure reproducibility, the complete source code—covering data generation, model training, and sampling algorithms—is publicly available on GitHub \url{https://github.com/Astringency/FM4PDE}. All experiments were conducted using dual NVIDIA A100 GPUs (80GB).

\subsection{Data Preparation and Problem Settings}

\textbf{Darcy Flow.} We consider the static Darcy Flow with homogeneous Dirichlet boundary conditions on $\partial \Omega$:
\begin{equation*}
    \begin{aligned}
        - \nabla \cdot (\boldsymbol {a} (\boldsymbol {c}) \nabla \boldsymbol{u} (\boldsymbol {c})) & = q (\boldsymbol {c}), \quad \boldsymbol {c} \in \Omega \\
    \boldsymbol{u} (\boldsymbol {c}) & = 0, \quad \boldsymbol {c} \in \partial \Omega
    \end{aligned}
\end{equation*}
We set $q(\boldsymbol{c}) = 1$, and the coefficient $\boldsymbol{a}$ takes binary values of either 12 or 4.

\textbf{Helmholtz Equations and Poisson Equations.} We consider the inhomogeneous Helmholtz Equations with homogeneous Dirichlet boundary conditions on $\partial \Omega$, which describes wave propagation:
\begin{equation}\label{eq:inhom_helmholtz}
    \begin{aligned}
        \nabla^2 \boldsymbol{u}(\boldsymbol{c}) + k^2 \boldsymbol{u}(\boldsymbol{c}) &= \boldsymbol{a}(\boldsymbol{c}), \quad \boldsymbol{c} \in \Omega, \\
        \boldsymbol{u}(\boldsymbol{c}) &= 0, \quad \boldsymbol{c} \in \partial \Omega.
    \end{aligned}
\end{equation}
In this setting, $\boldsymbol{a}$ is a piecewise constant function. The constant $k$ serves as a control parameter and we fix $k=1$ for Helmholtz case, When $k$ vanishes, the system reduces to a Poisson equation.

\textbf{Non-bounded Navier-Stokes Equations.} We consider the vorticity formulation of the incompressible Navier-Stokes equations in an unbounded domain:
\begin{equation*}
    \begin{aligned}
        \partial_t w(\boldsymbol{c}, \tau) + v(\boldsymbol{c}, \tau) \cdot \nabla w(\boldsymbol{c}, \tau) & = \nu \Delta w(\boldsymbol{c}, \tau) + q(\boldsymbol{c}), \\
        \nabla \cdot v(\boldsymbol{c}, \tau) & = 0, \quad \boldsymbol{c} \in \Omega, \tau \in (0, \mathcal{T}].
    \end{aligned}
\end{equation*}
Here, $w = \nabla \times v$ denotes the vorticity, $v(\boldsymbol{c}, \tau)$ is the velocity field, and $q(\boldsymbol{c})$ represents the external forcing. The viscosity is set to $\nu = 10^{-3}$, corresponding to a Reynolds number of $R_e = 1000$. Our study focuses on the joint distribution of the initial state $w_0$ and the terminal state $w_{\mathcal{T}}$, with the time horizon set to $\mathcal{T} = 10$. The structural PDE guidance is imposed as $\nabla \cdot w = \nabla \cdot (\nabla \times v)$.


\textbf{Shallow-Water Equations.} Derived from the Navier-Stokes equations, the Shallow-Water Equations (SWE) provide a robust framework for modeling free-surface flows. In two dimensions, this system of hyperbolic PDEs is governed by:
\begin{equation*}
    \begin{aligned}
        \partial_t h + \nabla \cdot (h \boldsymbol{u}) & = 0, \\
        \partial_t (h \boldsymbol{u}) + \nabla \cdot \left( \boldsymbol{u}^2 h + \dfrac{1}{2} g_r h^2 \right) & = - g_r h \nabla b,
    \end{aligned}
\end{equation*}
where $\boldsymbol{u}=(u,v)$ represents the velocity vector, $h$ is the water depth, and $b$ denotes the spatially varying bathymetry. The term $h \boldsymbol{u}$ corresponds to the momentum flux, with $g_r$ representing gravitational acceleration. 


\textbf{Reaction-Diffusion Equations.} We consider a two-component reaction-diffusion system in a 2D domain, governed by the FitzHugh-Nagumo model. The system describes the interaction between an activator $u(\boldsymbol{c}, \tau)$ and an inhibitor $v(\boldsymbol{c}, \tau)$ through the following coupled PDEs:
\begin{equation*}
    \begin{aligned}
        \partial_{t} u & = D_u \Delta u + R_u(u, v), \\
        \partial_{t} v & = D_v \Delta v + R_v(u, v),
    \end{aligned}
\end{equation*}
where $D_u = 1 \times 10^{-3}$ and $D_v = 5 \times 10^{-3}$ are the diffusion coefficients. The nonlinear reaction kinetics are defined as:
\begin{equation*}
    \begin{aligned}
        R_u(u, v) & = u - u^3 - k - v, \\
        R_v(u, v) & = u - v,
    \end{aligned}
\end{equation*}
with the parameter set to $k = 5 \times 10^{-3}$.


\textbf{Burgers’ Equations.} We study the viscous Burgers’ equation on a one-dimensional unit domain $\Omega = (0, 1)$ with periodic boundary conditions. The governing equation is given by:
\begin{equation*}
    \begin{aligned}
        \partial_t u (\boldsymbol{c}, \tau) + u \partial_{\boldsymbol{c}} u (\boldsymbol{c}, \tau) & = \nu \partial^2_{\boldsymbol{c}} u (\boldsymbol{c}, \tau), \\
        u(\boldsymbol{c}, 0) & = u_0(\boldsymbol{c}), \quad \boldsymbol{c} \in \Omega, \tau \in (0, \mathcal{T}],
    \end{aligned}
\end{equation*}
where we set the viscosity to $\nu = 0.01$. For the numerical experiment, we discretize the spatial domain with 128 points ($u_0 \in \mathbb{R}^{128}$). The system is evolved for 127 subsequent time steps to produce a full spatiotemporal trajectory $u_{0:\mathcal{T}}$ with dimensions $128 \times 128$.


Coefficient fields for static equations and initial conditions for time-dependent problems are sampled from Gaussian Random Fields (GRFs), with ground truth trajectories computed via standard numerical techniques like finite difference methods. To strictly evaluate model generalization, we employ different GRF settings for training and testing. In particular, the test samples are generated from a smoother distribution compared to the training set. To train the flow matching model, we generate 50000 training samples for each PDE.

\subsection{Baseline Approaches and Evaluation Metrics}

We compare the proposed method with several representative baselines across different PDE-solving scenarios. First, to evaluate the model’s capability in conventional PDE-solving tasks, we provide full observation data and compare FM4PDE with FNO and DeepONet. We implement both the forward and inverse problems of FNO using the PDEBench \citep{takamoto2022pdebench} framework, while DeepONet is trained and inferred using the DeepXDE \citep{lu2021deepxde} library.

Specifically, the inverse problem of the FNO is solved using a gradient-based optimization approach \citep{takamoto2022pdebench}, where the solution is obtained by minimizing the following prediction loss $\mathcal{L}_{\mathrm{inverse}}(\boldsymbol{u}, \hat{\boldsymbol{u}} (\hat{\boldsymbol{a}}))$, where $\hat{\boldsymbol{a}}$ denotes the inferred inverse solution, and $\hat{\boldsymbol{u}} (\hat{\boldsymbol{a}})$ represents the predicted forward solution obtained by feeding $\hat{\boldsymbol{a}}$ into the FNO forward model. Subsequently, to assess the performance of our approach in sparse reconstruction tasks, we compare it against DiffusionPDE.

To evaluate the model’s ability to solve PDEs and reconstruct the global solution from sparse observations, we adopt the relative error as the evaluation metric, which is defined as follows. Let $\boldsymbol{x}_{\text{pred}}$ denote the predicted result and $\boldsymbol{x}_{\text{true}}$ denote the ground-truth result. The relative error is given by
\begin{equation*}
    \epsilon_{\text{rel}} = \frac{\|\boldsymbol{x}_{\text{pred}} - \boldsymbol{x}_{\text{true}}\|_2}{\|\boldsymbol{x}_{\text{true}}\|_2},
\end{equation*}
where $\|\cdot\|_2$ denotes the $L^2$ norm. We evaluate the model on a test set that is distinct from the training data. To ensure robustness, we independently generate 20 test samples for each evaluation, and the final metric is reported as the average over these 20 runs.

\subsection{Experimental Results}

We evaluate the capabilities of FM4PDE through a variety of simulation experiments. Since FNO and DeepONet primarily focus on solving PDEs under fully observed settings, we first compare FM4PDE with these two models using complete observation data. Experiments are conducted on Darcy flow, the Poisson equation, the Helmholtz equation, and the Navier–Stokes equations. The results are summarized in \Cref{tab:full_eval}.

\begin{table}[H]
    \footnotesize
    \centering
    \caption{Comparative analysis of relative errors between FM4PDE, DeepONet, and FNO for forward and inverse problems across various PDEs under full observations. The number of steps for the FM4PDE algorithm is set to 1000.}
    \label{tab:full_eval}
    \begin{tabular}{llrrr}
        \toprule
        \textbf{PDE} & \textbf{Problem} & \textbf{FM4PDE} & \textbf{DeepONet} & \textbf{FNO} \\
        \midrule
        Darcy Flow & Forward & 1.33\% & 39.58\% & 3.62\% \\
         & Inverse & 12.73\% & 5.21\% & 91.20\% \\
        \midrule
        Poisson & Forward & 1.65\% & 45.38\% & 54.69\% \\
         & Inverse & 17.02\% & 76.86\% & 59.32\% \\
        \midrule
        Helmholtz & Forward & 13.85\% & 41.06\% & 51.64\% \\
         & Inverse & 24.78\% & 40.26\% & 65.16\% \\
        \midrule
        Navier-Stokes & Forward & 0.32\% & 26.29\% & 35.61\% \\
         & Inverse & 6.39\% & 33.33\% & 19.59\% \\
        \bottomrule
    \end{tabular}
\end{table}

For the sparse reconstruction task, we compare the performance of FM4PDE and DiffusionPDE for both forward and inverse problems of Darcy flow, the Poisson equation, the Helmholtz equation, and the Navier–Stokes equations. For the forward problem, we provided 500 randomly selected sparse observations of either the coefficients (for static equations) or the initial state (for time-dependent equations); for the inverse problem, we provided 500 observations of either the solution (for static equations) or the final state (for time-dependent equations). The relative errors of the reconstruction results are shown in \Cref{tab:sparse_eval}.

\begin{table}[H]
    \footnotesize
    \centering
    \caption{Comparative analysis of relative errors between FM4PDE and DiffusionPDE for forward and inverse problems across various PDEs under sparse observations. The number of steps for both algorithms is set to 1000.}
    \label{tab:sparse_eval}
        \begin{tabular}{llrr}
        \toprule
        \textbf{PDE} & \textbf{Problem} & \textbf{FM4PDE} & \textbf{DiffusionPDE} \\
        \midrule
        Darcy Flow & Forward & 1.27\% & 13.44\% \\
         & Inverse & 17.76\% & 35.83\% \\
        \midrule
        Poisson & Forward & 3.02\% & 15.32\% \\
         & Inverse & 24.82\% & 56.54\% \\
        \midrule
        Helmholtz & Forward & 7.48\% & 22.02\% \\
         & Inverse & 40.82\% & 44.50\% \\
        \midrule
        Navier-Stokes & Forward & 5.60\% & 2.71\% \\
         & Inverse & 8.16\% & 9.84\% \\
        \bottomrule
    \end{tabular}
\end{table}

For the Burgers equation, we evaluate two different sparse recovery modes: one is similar to the other two-dimensional equations, where 500 observation points are randomly assigned across 128 temporal and 128 spatial grids, and we attempt to recover solutions across all spatiotemporal grids; the other is to assign spatial solutions at 5 of the 128 time points and consider recovering spatial solutions across all time points. The recovery results are shown in \Cref{fig:burgers_time_interval}. 

\begin{figure}[H]
    \centering
    \includegraphics[width=0.8\linewidth]{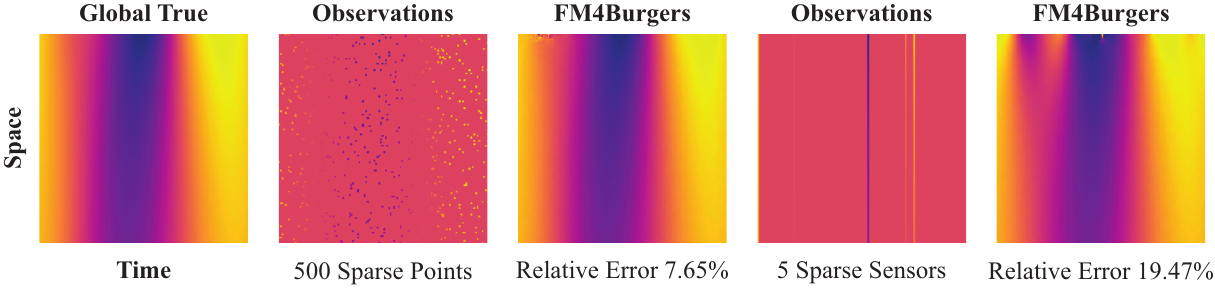}
    \caption{Performance of FM4PDE in reconstructing the full temporal evolution of the Burgers' equation under two sparse observation patterns.}
    \label{fig:burgers_time_interval}
\end{figure}

For the Reaction-Diffusion equation and the Shallow-Water equation, the number of physical quantities required to formulate the PDE loss is greater than one, and the PDE loss cannot be simplified in the same manner as the Navier-Stokes equation. Therefore, we design a network architecture that accepts input and produces output with more channels. Based on this architecture, we perform a global solution reconstruction given 500 sparse observation points. The results are presented in \Cref{fig:reaction_diffusion} and \Cref{fig:shallow_water}. 

\begin{figure}[H]
    \centering
    \includegraphics[width=0.75\linewidth]{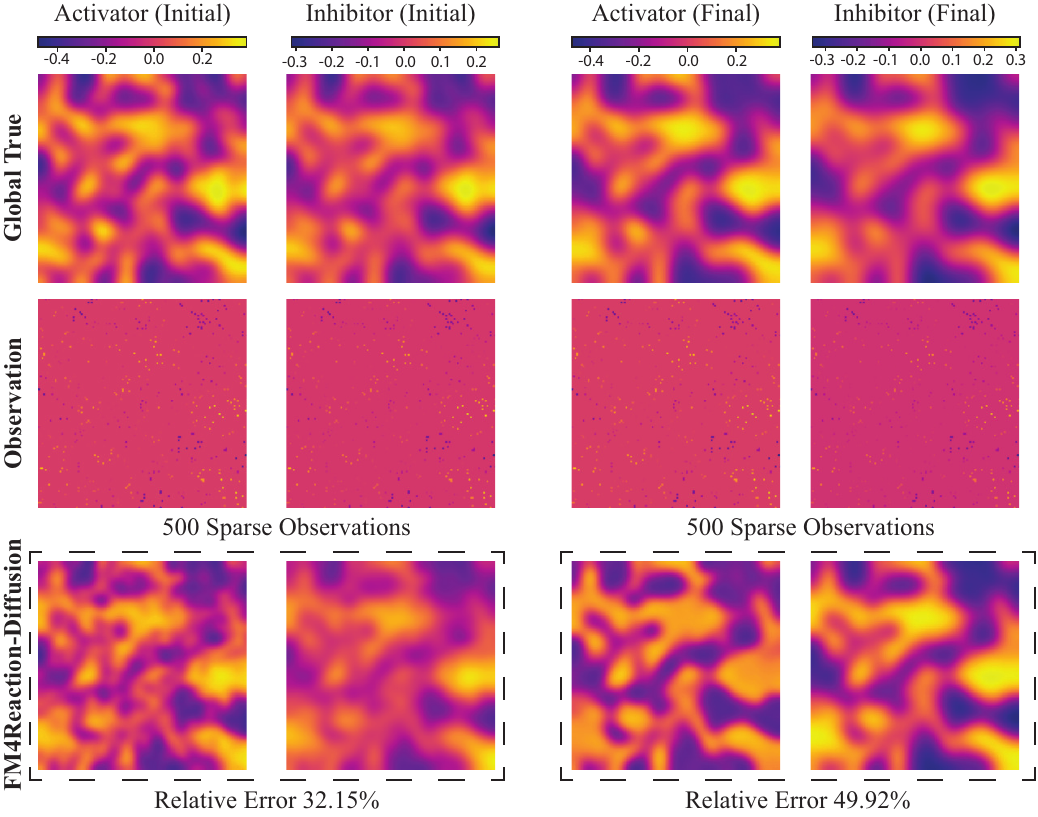}
    \caption{Reconstruction of the reaction-diffusion equation from 500 sparse observation points using FM4PDE.}
    \label{fig:reaction_diffusion}
\end{figure}

\begin{figure}[H]
    \centering
    \includegraphics[width=0.8\linewidth]{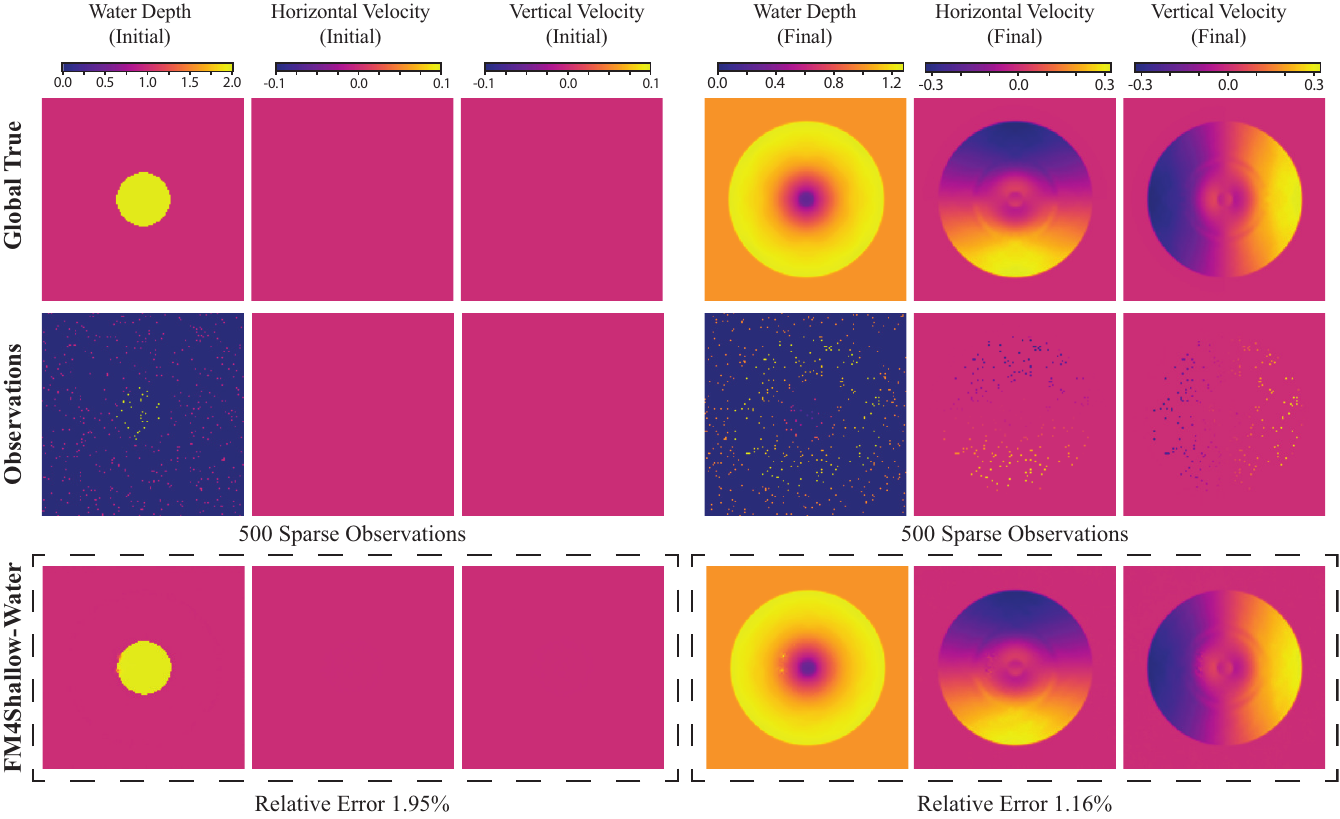}
    \caption{Reconstruction of the shallow-water equations from 500 sparse observation points using FM4PDE.}
    \label{fig:shallow_water}
\end{figure}

We also compare the sampling times of FM4PDE and DiffusionPDE across different tasks. As shown in \Cref{fig:time_eval_forward} and \Cref{fig:time_eval_inverse}, FM4PDE achieves comparable or faster sampling times across all tested equations.

\begin{figure}[H]
    \centering
    \includegraphics[width=0.8\linewidth]{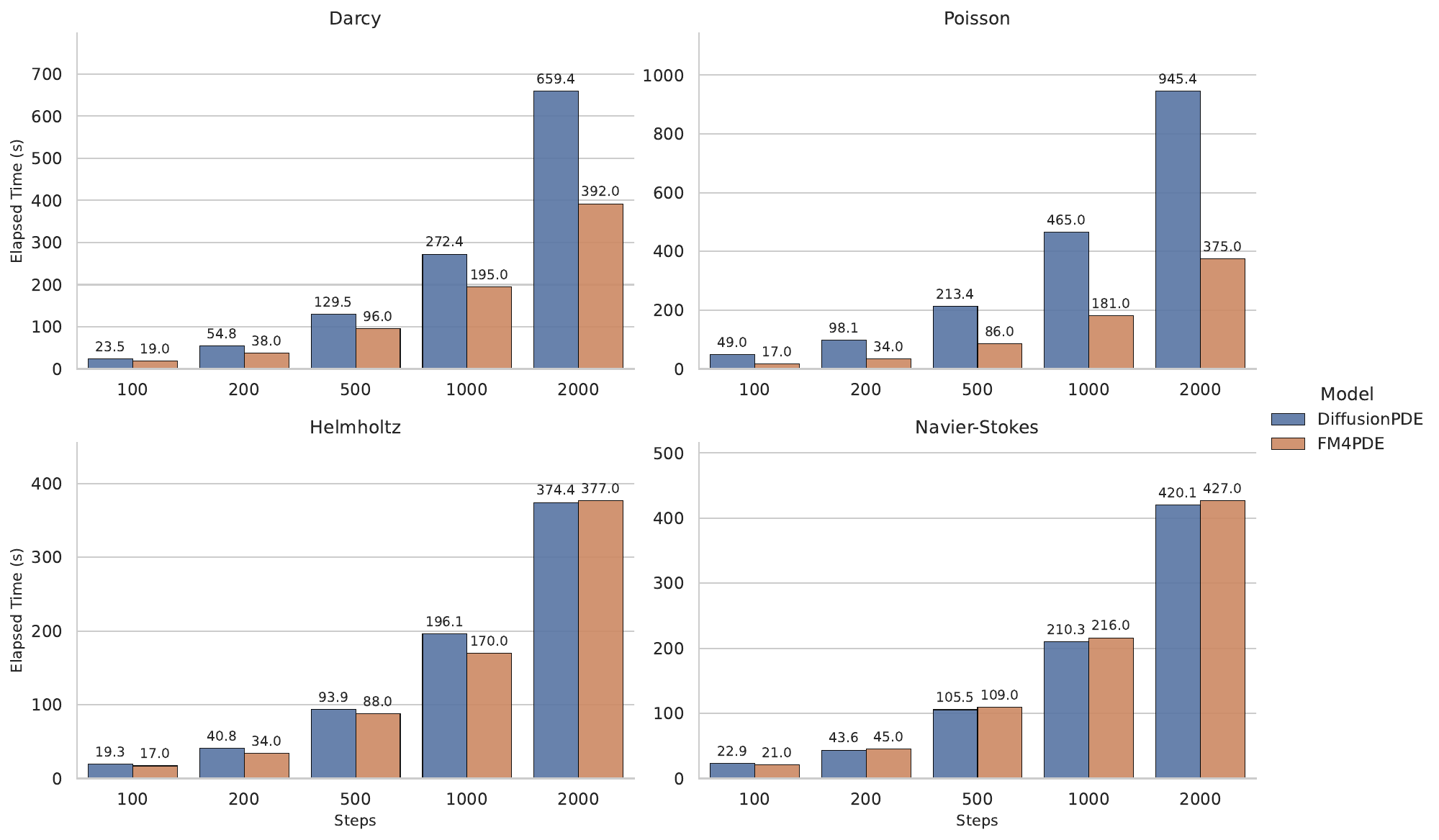}
    \caption{Comparison of sampling times between FM4PDE and DiffusionPDE for solving forward problems of different equations.}
    \label{fig:time_eval_forward}
\end{figure}

\begin{figure}[H]
    \centering
    \includegraphics[width=0.8\linewidth]{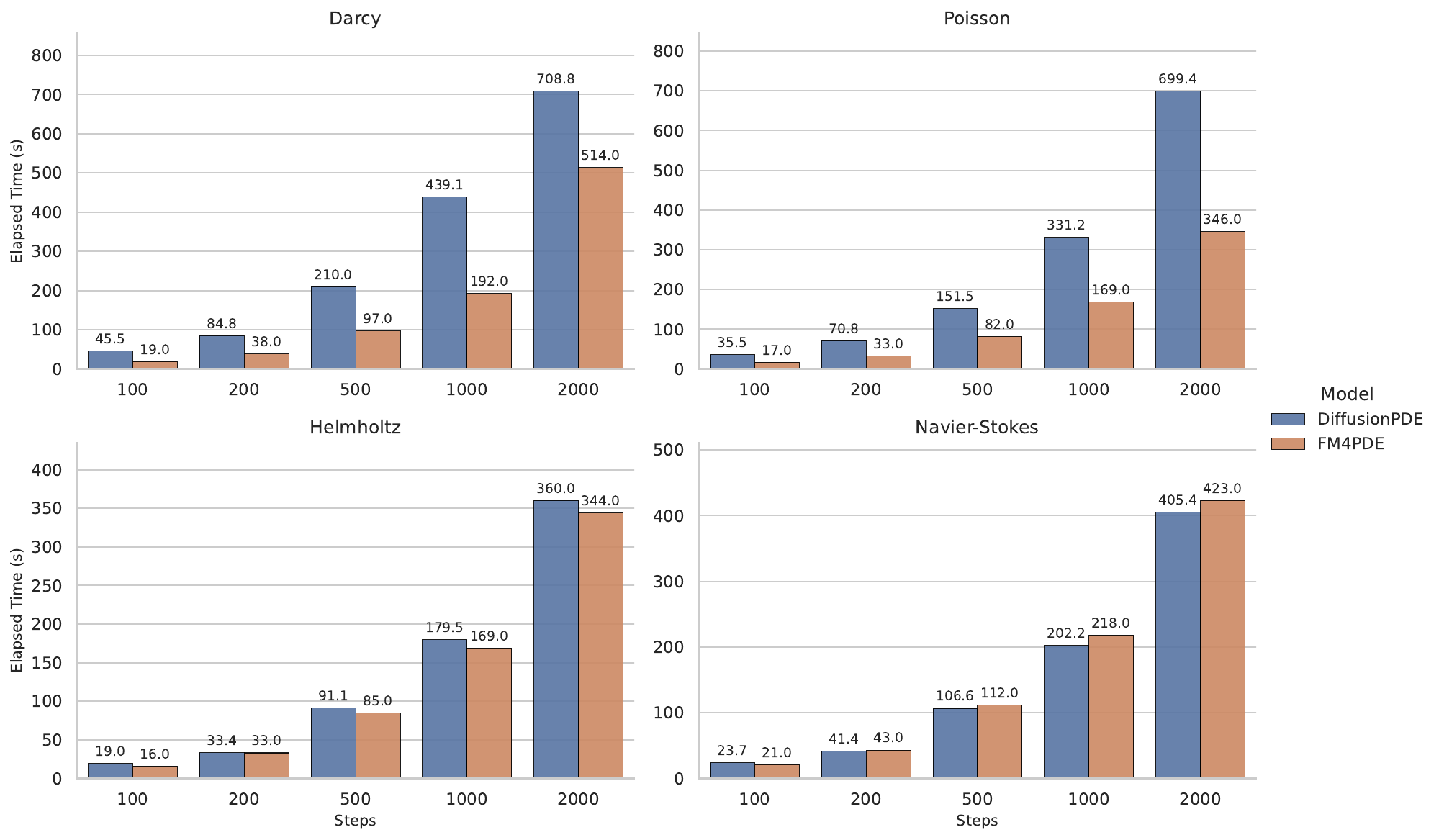}
    \caption{Comparison of sampling times between FM4PDE and DiffusionPDE for solving inverse problems of different equations.}
    \label{fig:time_eval_inverse}
\end{figure}

\section{Discussion}\label{sec:discussion}

\subsection{Impact of Observation Sparsity}

We evaluate the impact of different sparsity levels on the generation performance of the FM4PDE algorithm. \Cref{fig:burgers_diffsensors} shows the sparse recovery results of the Burgers equation given solutions for different time segments.

\begin{figure}[H]
    \centering
    \includegraphics[width=0.8\linewidth]{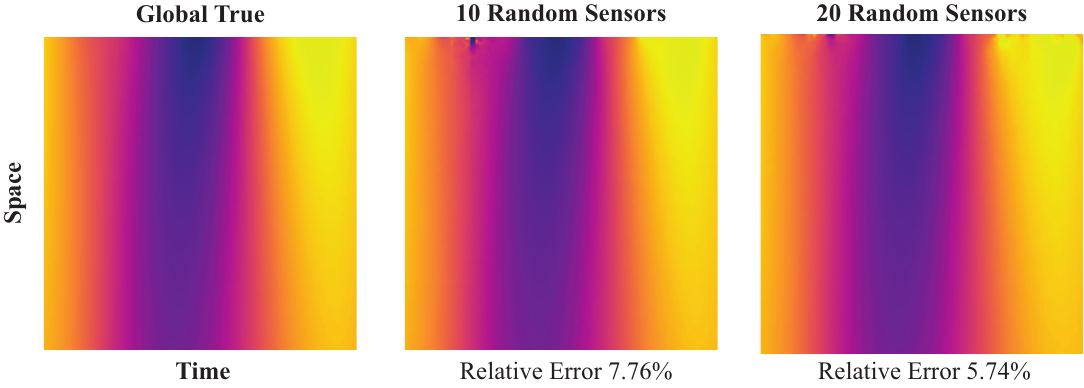}
    \caption{Reconstruction results of FM4PDE on the Burgers equation given 10 and 20 time segments.}
    \label{fig:burgers_diffsensors}
\end{figure}

\subsection{Impact of Step Sizes}

We examine the performance of the FM4PDE algorithm under different sampling step sizes and the number of steps. \Cref{fig:poisson_diffstepsize} shows the reconstruction of the Poisson equation. The reconstruction results tend to improve as the step size becomes finer.

\begin{figure}[H]
    \centering
    \includegraphics[width=0.9\linewidth]{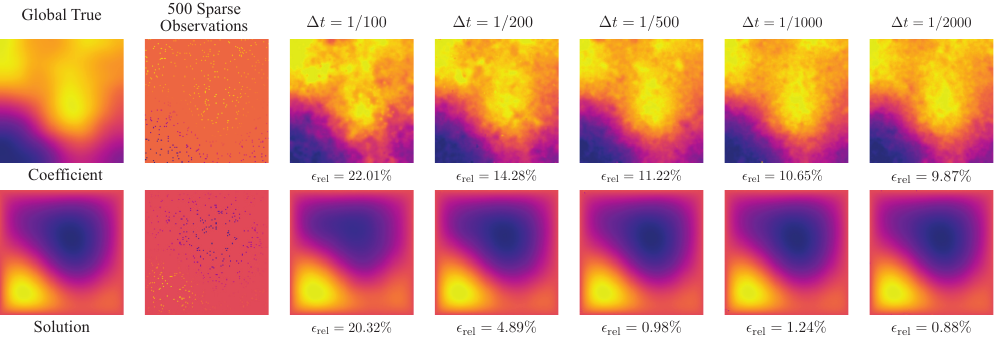}
    \caption{The sampling results of solving the Poisson equation using FM4PDE at different step sizes.}
    \label{fig:poisson_diffstepsize}
\end{figure}

\subsection{Comparison of Sampling Strategies}

We also examine the performance of different sampling schemes for FM4PDE, including purely deterministic, purely stochastic, and various hybrid configurations. In the notation below, ``D'' denotes the deterministic sampler and ``S'' denotes the stochastic sampler; for example, 0.2D+0.8S indicates that the deterministic sampler is used in the first 20\% of the sampling steps and the stochastic sampler in the remaining 80\%. \Cref{tab:sampler_eval} reports the global error, observation loss, and PDE loss obtained by different sampling schemes in the sparse recovery task of the Poisson equation. Global error is measured as relative error, while observation loss and PDE loss are measured as mean squared error.

\begin{table}[htbp]
    \footnotesize
    \centering
    \caption{Global relative error (Loss), observation mean squared error (ObsLoss), and PDE mean squared error (PDELoss) for various sampling strategy mixing ratios on the Poisson equation sparse recovery task.}
    \label{tab:sampler_eval}
    \begin{tabular}{cccccc}
        \toprule
        \textbf{Sampler} & \textbf{Loss(Coef)} & \textbf{Loss(Sol)} & \textbf{ObsLoss(Coef)} & \textbf{ObsLoss(Sol)} & \textbf{PDELoss} \\
        \midrule
        Pure D & 84.02\% & 62.76\% & 0.0125 & 0.0002 & 0.0083 \\
        Pure S & 7.81\% & 0.73\% & 0.0027 & $< 0.0001$ & 0.0078 \\
        0.7S+0.3D & 25.36\% & 6.67\% & 0.0036 & $< 0.0001$ & 0.0092 \\
        0.8S+0.2D & 19.67\% & 4.56\% & 0.0026 & $< 0.0001$ & 0.0021 \\
        0.9S+0.1D & 11.67\% & 3.13\% & 0.0015 & $< 0.0001$ & 0.0015 \\
        0.1D+0.9S & 8.14\% & 1.48\% & 0.0027 & $< 0.0001$ & 0.0077 \\
        0.2D+0.8S & 7.84\% & 1.02\% & 0.0027 & $< 0.0001$ & 0.0073 \\
        0.3D+0.7S & 8.31\% & 1.78\% & 0.0027 & $< 0.0001$ & 0.0067 \\
        0.8D+0.2S & 22.69\% & 29.41\% & 0.0026 & $< 0.0001$ & 0.0088 \\
        \bottomrule
    \end{tabular}
\end{table}

In our experiments, the purely stochastic scheme consistently achieves the best performance, followed by the hybrid strategy that transitions from deterministic to stochastic updates (Det$\to$Stoch). These findings are consistent with the theoretical analysis. \Cref{thm:stoc} shows that the purely stochastic sampler with adaptive guidance achieves $V_N = O(\bar\zeta)$, while \Cref{thm:hybrid} establishes the same asymptotic bound for the Det$\to$Stoch hybrid. The advantage of the hybrid is structural: the deterministic Phase~I provides fast initial convergence to a loss basin in $O(\log(1/\varepsilon))$ steps via Phase~A contraction, while the stochastic Phase~II provides the final accuracy through adaptive guidance. The empirical results confirm that a short deterministic initialization (e.g., 0.2D+0.8S) slightly improves over purely stochastic sampling, consistent with this structural benefit.

Conversely, Stoch$\to$Det orderings (e.g., 0.8S+0.2D) perform worse because the deterministic phase starts at $t = t^* > \epsilon$, missing the Phase~A contraction region where $b_t = 1/t - 1 \gg 1$. As shown in \Cref{rem:alternative_orderings}, this yields only $O(1)$ contraction instead of the vanishing factor $\epsilon^{2\mu}$. Similarly, the purely deterministic method performs poorly, as the error accumulation inherent to Phase~B dominates and degrades convergence.

\section{Conclusion}\label{sec:conclusion}

In this work, we introduced FM4PDE, a generative modeling framework based on flow matching, for reconstructing global solutions of partial differential equations from sparse observations. By learning a continuous transformation from a noise distribution to the joint distribution of PDE solutions and their coefficients, FM4PDE captures the underlying dynamics and statistical structure of the solution space. Our method integrates sparse observational guidance and PDE residual constraints into the sampling process, enabling physically consistent and high-fidelity reconstructions. The proposed sampler achieves faster inference than existing generative approaches, and we provide theoretical guarantees to support its convergence and reliability. FM4PDE demonstrates the potential of combining generative priors with physics-informed constraints for solving forward and inverse PDE problems under limited data conditions.

Despite its strengths, the framework has limitations that merit future exploration. Currently, FM4PDE's treatment of time-dependent systems is restricted to initial and final states for high-dimensional PDEs. Integrating flow-matching temporal evolution with the intrinsic dynamics of PDEs represents a promising path toward more physically faithful modeling. Furthermore, moving beyond residual-guided sampling to develop efficient, data-driven methods for hyperparameter inference remains an essential task for practical deployment.


\acks{This work is supported by the National Key R\&D Program of China 2025YFA1018700,  the Beijing Outstanding Young Scientist Program (No.JWZQ20240101027), and the Beijing Natural Science Foundation (No. JR25003).}


\newpage

\appendix

\section{Lower Bound of Stochastic Sampler with Constant Guidance}\label{apx:lower_bound}

This section establishes a mechanism-level lower bound for the stochastic sampler with constant and adaptive guidance: we prove that the steady-state loss $V_{ss}$ is bounded away from zero for any constant guidance strength $\zeta$, independent of the number of steps or the grid refinement. The lower bound is demonstrated on a canonical one-dimensional (1D) linear instance ($u_t(x) = -x$, $\mathcal{L}(x) = x^2/2$) that satisfies all standing assumptions, representing the best-case scenario for constant guidance: a perfectly learned velocity field with a strongly convex loss. For more complex PDE-constrained problems, the actual floor may be larger due to additional sources of error, but the mechanism-level floor $V_{ss} \ge \delta_{\min}/40$ established here is unavoidable under constant guidance. 

\begin{remark}[Scope Compatibility of Upper and Lower Bounds]\label{rem:scope_compatibility}
    The 1D instance used below ($u_t(x) = -x$, $\kappa = B_u = 1$) satisfies the small gain condition $B_u^2 = 1 < 2 = 2\kappa$, so it falls within the scope of both the upper bound (Theorem~\ref{thm:stoc}) and the lower bound (Theorem~\ref{thm:lower_bound}). 
\end{remark}

\begin{theorem}[Lower Bound for Constant Guidance]\label{thm:lower_bound}
Consider the 1D instance of Algorithm~\ref{alg:fm4pde_stoc_sampler} with constant guidance $\zeta \in (0,1)$:
\begin{itemize}[topsep=2pt,itemsep=1pt]
    \item Velocity field $u_t(x) = -x$ (dissipativity constant $\kappa = 1$),
    \item Loss function $\mathcal{L}(x) = x^2/2$ (PL constant $\mu = 1$, smoothness $L_{\mathcal{L}} = 1$),
    \item Endpoint prediction $\Phi_t(x) = x + (1-t)(-x) = tx$ (so $J_t = t$).
\end{itemize}
Under the constant-guidance variant of Algorithm~\ref{alg:fm4pde_stoc_sampler} with a uniform grid satisfying $t_k = t_N$ for all Phase~2 steps (i.e., representing repeated evaluation at the terminal noise level $\delta_{\min}$), the exact second moment recursion gives $\mathbb{E}[x_{k+1}^2] = \delta_{\min}^2 + a^2 \cdot \mathbb{E}[x_k^2]$ with $a := (1-\zeta)t_N^2$ and $t_N := 1 - \delta_{\min}$. In the steady-state regime, the expected loss satisfies:
\begin{equation}\label{eq:Vss_lower}
    V_{ss} = \frac{t_N^2}{2} \cdot \frac{\delta_{\min}^2}{1 - (1-\zeta)^2 t_N^4} \ge \frac{\delta_{\min}}{40}
\end{equation}
for all $\zeta \le \delta_{\min}/2$ and $\delta_{\min} \le 1/4$. In particular, for any fixed $\delta_{\min} > 0$, the bias $V_{ss} \ge \delta_{\min}/40 > 0$ is bounded away from zero for all $\zeta \le \delta_{\min}/2$, confirming that a positive bias is unavoidable when constant guidance is used with a fixed noise floor.
\end{theorem}

\begin{proof}

    For this instance, $\hat{x}_1^{(k)} = \Phi_{t_k}(x_k) = t_k x_k$, and $\nabla\mathcal{L}(\hat{x}_1^{(k)}) = t_k x_k$. The backpropagated gradient is $g_k = J_{t_k}^\top \nabla\mathcal{L}(\hat{x}_1^{(k)}) = t_k \cdot t_k x_k = t_k^2 x_k$. The algorithm update becomes:
    \begin{equation*}
        x_{k+1} = \delta_{k+1}\xi_k + t_{k+1} \cdot t_k x_k - \zeta t_k^2 x_k = \delta_{k+1}\xi_k + t_k(t_{k+1} - \zeta t_k) x_k.
    \end{equation*}
    Since $\xi_k \sim \mathcal{N}(0,1)$ is independent of $\mathcal{F}_k$, we have $\mathbb{E}[x_{k+1}^2] = \delta_{k+1}^2 + [t_k(t_{k+1} - \zeta t_k)]^2 \mathbb{E}[x_k^2]$. When $\delta_k = \delta_{\min}$ for all Phase~2 steps, we have $t_k = t_N = 1 - \delta_{\min}$ and the contraction coefficient is $a = (1-\zeta)t_N^2$. The recursion becomes:
    \begin{equation*}
        W_{k+1} = \delta_{\min}^2 + (1-\zeta)^2 t_N^4 \cdot W_k,
    \end{equation*}
    where $W_k := \mathbb{E}[x_k^2]$. The contraction factor $(1-\zeta)^2 t_N^4 < 1$ for $\zeta > 0$ and $t_N < 1$, so the unique steady state is:
    \begin{equation*}
        W_{ss} = \frac{\delta_{\min}^2}{1 - (1-\zeta)^2 t_N^4}.
    \end{equation*}

    Since $V_k = \frac{t_k^2}{2}\mathbb{E}[x_k^2]$ for $\mathcal{L}(x) = x^2/2$ and $\hat{x}_1^{(k)} = t_k x_k$:
    \begin{equation*}
        V_{ss} = \frac{t_N^2}{2} W_{ss} = \frac{t_N^2 \delta_{\min}^2}{2(1 - (1-\zeta)^2 t_N^4)}.
    \end{equation*}

    Since $\zeta^2 \le \zeta$ and $(1-\delta_{\min})^4 \ge 1 - 4\delta_{\min}$ for $\delta_{\min} \le 1/4$:
    \begin{align*}
        1 - (1-\zeta)^2 t_N^4 = 1 - (1-2\zeta+\zeta^2)(1-\delta_{\min})^4 \nonumber \le 2\zeta + 4\delta_{\min}.
    \end{align*}
    Also, $t_N^2 = (1-\delta_{\min})^2 \ge 1/4$ for $\delta_{\min} \le 1/2$. Therefore:
    \begin{equation*}
        V_{ss} \ge \frac{(1/4)\delta_{\min}^2}{2(2\zeta + 4\delta_{\min})} = \frac{\delta_{\min}^2}{8(2\zeta + 4\delta_{\min})}.
    \end{equation*}
    For $\zeta \le \delta_{\min}/2$, we have $2\zeta + 4\delta_{\min} \le 5\delta_{\min}$, giving:
    \begin{equation*}
        V_{ss} \ge \frac{\delta_{\min}^2}{8 \cdot 5\delta_{\min}} = \frac{\delta_{\min}}{40} > 0.
    \end{equation*}
\end{proof}


\begin{proposition}[Structural Floor under Adaptive Guidance]\label{prop:adaptive_floor}
    Consider the same 1D instance as Theorem~\ref{thm:lower_bound} ($u_t(x) = -x$, $\mathcal{L}(x) = x^2/2$), but now with adaptive guidance $\zeta_k = c_\zeta\delta_k$ and a geometric grid $\delta_{k+1} = (1-c_\Delta)\delta_k$ in Phase~2. At step $k$, the exact one-step recursion gives $W_{k+1} = \delta_{k+1}^2 + (1-c_\zeta\delta_k)^2 t_k^4\, W_k$. Iterating from $\delta_0 = \epsilon_s$ to $\delta_N = \delta_{\min}$, the terminal second moment satisfies:
    \begin{equation*}
        W_N \ge \frac{\delta_{\min}^2}{1 - (1-c_\zeta\delta_{\min})^2(1-\delta_{\min})^4} = O(\delta_{\min}/c_\zeta),
    \end{equation*}
    and the terminal expected loss is $V_N = \frac{t_N^2}{2}W_N = O(\delta_{\min}/c_\zeta) = O(\bar\zeta/c_\zeta^2)$. Thus, with adaptive guidance, $V_N = O(\bar\zeta)$---the loss scales \emph{linearly} with the terminal guidance strength, consistent with the upper bound $V_N \le C_{final} \bar\zeta + C_{floor}$ from Theorem~\ref{thm:stoc}. The floor $C_{floor}$ reflects the generic $O(\delta_k(1+\norm{\boldsymbol{x}_k}))$ prediction error, which in this 1D instance reduces to $O(\delta_k |x_k|)$. In the steady state, $|x_k| = O(1)$, so the error is $O(\delta_k)$ and the floor is $O(1)$.
\end{proposition}

\begin{table}[H]
    \footnotesize
    \centering
    \caption{Comparison of achievable accuracy under different guidance strategies.}
    \label{tab:lower_bound_summary}
    \begin{tabular}{lccc}
        \toprule
        \textbf{Guidance strategy} & \textbf{Accuracy} & \textbf{Floor} & \textbf{Reference} \\
        \midrule
        Constant $\zeta$, fixed $\delta_{\min}$ & $\alpha\zeta + \beta$ & $\beta \ge \delta_{\min}/40 > 0$ & \Cref{thm:lower_bound} \\
        Adaptive $\zeta_k = c_\zeta\delta_k$ & $C_{final}\bar\zeta + C_{floor}$ & $C_{floor} > 0$ & \Cref{thm:stoc}, \Cref{prop:adaptive_floor} \\
        Adaptive + pred.\ consistency & $C_{final}\bar\zeta + C_{floor}^{(pc)}$ & $C_{floor}^{(pc)} \propto \epsilon_{pc}^2$ & \Cref{rem:prediction_consistency_reduction} \\
        Adaptive, $\epsilon_{pc} = 0$ & $C_{final}\bar\zeta$ & $0$ & \Cref{rem:prediction_consistency_reduction} \\
        \bottomrule
    \end{tabular}
\end{table}

\vskip 0.2in
\bibliography{references}

@inproceedings{zhang2023controlnet,
  title={Adding conditional control to text-to-image diffusion models},
  author={Zhang, Lvmin and Rao, Anyi and Agrawala, Maneesh},
  booktitle={Proceedings of the IEEE/CVF international conference on computer vision},
  pages={3836--3847},
  year={2023}
}

@article{raissi2019pinns,
  title={Physics-informed neural networks: A deep learning framework for solving forward and inverse problems involving nonlinear partial differential equations},
  author={Raissi, Maziar and Perdikaris, Paris and Karniadakis, George E},
  journal={Journal of Computational physics},
  volume={378},
  pages={686--707},
  year={2019},
  publisher={Elsevier}
}

@article{li2020fno,
  title={Fourier neural operator for parametric partial differential equations},
  author={Li, Zongyi and Kovachki, Nikola and Azizzadenesheli, Kamyar and Liu, Burigede and Bhattacharya, Kaushik and Stuart, Andrew and Anandkumar, Anima},
  journal={arXiv preprint arXiv:2010.08895},
  year={2020}
}

@article{lu2021deeponet,
  title={Learning nonlinear operators via DeepONet based on the universal approximation theorem of operators},
  author={Lu, Lu and Jin, Pengzhan and Pang, Guofei and Zhang, Zhongqiang and Karniadakis, George Em},
  journal={Nature machine intelligence},
  volume={3},
  number={3},
  pages={218--229},
  year={2021},
  publisher={Nature Publishing Group UK London}
}

@article{huang2024diffusionpde,
  title={DiffusionPDE: Generative PDE-solving under partial observation},
  author={Huang, Jiahe and Yang, Guandao and Wang, Zichen and Park, Jeong Joon},
  journal={Advances in Neural Information Processing Systems},
  volume={37},
  pages={130291--130323},
  year={2024}
}

@article{jacobsen2025cocogen,
  title={Cocogen: Physically consistent and conditioned score-based generative models for forward and inverse problems},
  author={Jacobsen, Christian and Zhuang, Yilin and Duraisamy, Karthik},
  journal={SIAM Journal on Scientific Computing},
  volume={47},
  number={2},
  pages={C399--C425},
  year={2025},
  publisher={SIAM}
}

@article{lu2021deepxde,
  title={DeepXDE: A deep learning library for solving differential equations},
  author={Lu, Lu and Meng, Xuhui and Mao, Zhiping and Karniadakis, George Em},
  journal={SIAM review},
  volume={63},
  number={1},
  pages={208--228},
  year={2021},
  publisher={SIAM}
}

@article{takamoto2022pdebench,
  title={Pdebench: An extensive benchmark for scientific machine learning},
  author={Takamoto, Makoto and Praditia, Timothy and Leiteritz, Raphael and MacKinlay, Daniel and Alesiani, Francesco and Pfl{\"u}ger, Dirk and Niepert, Mathias},
  journal={Advances in Neural Information Processing Systems},
  volume={35},
  pages={1596--1611},
  year={2022}
}

@article{utkarsh2025PCFM,
  title={Physics-Constrained Flow Matching: Sampling Generative Models with Hard Constraints},
  author={Utkarsh, Utkarsh and Cai, Pengfei and Edelman, Alan and Gomez-Bombarelli, Rafael and Rackauckas, Christopher Vincent},
  journal={arXiv preprint arXiv:2506.04171},
  year={2025}
}

@article{baldan2025PBFM,
  title={Flow Matching Meets PDEs: A Unified Framework for Physics-Constrained Generation},
  author={Baldan, Giacomo and Liu, Qiang and Guardone, Alberto and Thuerey, Nils},
  journal={arXiv preprint arXiv:2506.08604},
  year={2025}
}

@inproceedings{karimi2016linear,
  title={Linear convergence of gradient and proximal-gradient methods under the polyak-{\l}ojasiewicz condition},
  author={Karimi, Hamed and Nutini, Julie and Schmidt, Mark},
  booktitle={Joint European conference on machine learning and knowledge discovery in databases},
  pages={795--811},
  year={2016},
  organization={Springer}
}

@article{lipman2022flow,
  title={Flow matching for generative modeling},
  author={Lipman, Yaron and Chen, Ricky TQ and Ben-Hamu, Heli and Nickel, Maximilian and Le, Matt},
  journal={arXiv preprint arXiv:2210.02747},
  year={2022}
}

@article{lipman2024flow,
  title={Flow matching guide and code},
  author={Lipman, Yaron and Havasi, Marton and Holderrieth, Peter and Shaul, Neta and Le, Matt and Karrer, Brian and Chen, Ricky TQ and Lopez-Paz, David and Ben-Hamu, Heli and Gat, Itai},
  journal={arXiv preprint arXiv:2412.06264},
  year={2024}
}

@article{singh2024stochastic,
  title={Stochastic sampling from deterministic flow models},
  author={Singh, Saurabh and Fischer, Ian},
  journal={arXiv preprint arXiv:2410.02217},
  year={2024}
}

@article{lai2025principles,
  title={The principles of diffusion models},
  author={Lai, Chieh-Hsin and Song, Yang and Kim, Dongjun and Mitsufuji, Yuki and Ermon, Stefano},
  journal={arXiv preprint arXiv:2510.21890},
  year={2025}
}

@book{leveque2007finite,
  title={Finite difference methods for ordinary and partial differential equations: steady-state and time-dependent problems},
  author={LeVeque, Randall J},
  year={2007},
  publisher={SIAM}
}

@article{chen2022sampling,
  title={Sampling is as easy as learning the score: theory for diffusion models with minimal data assumptions},
  author={Chen, Sitan and Chewi, Sinho and Li, Jerry and Li, Yuanzhi and Salim, Adil and Zhang, Anru R},
  journal={arXiv preprint arXiv:2209.11215},
  year={2022}
}

@article{benton2023nearly,
  title={Nearly $ d $-linear convergence bounds for diffusion models via stochastic localization},
  author={Benton, Joe and De Bortoli, Valentin and Doucet, Arnaud and Deligiannidis, George},
  journal={arXiv preprint arXiv:2308.03686},
  year={2023}
}

@book{strauss2007partial,
  title={Partial Differential Equations: An Introduction},
  author={Strauss, Walter A},
  year={2007},
  edition={2nd},
  publisher={John Wiley \& Sons}
}

@book{evans2022partial,
  title={Partial Differential Equations},
  author={Evans, Lawrence C},
  year={2010},
  edition={2nd},
  series={Graduate Studies in Mathematics},
  volume={19},
  publisher={American Mathematical Society}
}

@inproceedings{ho2020denoising,
  title={Denoising diffusion probabilistic models},
  author={Ho, Jonathan and Jain, Ajay and Abbeel, Pieter},
  booktitle={Advances in Neural Information Processing Systems},
  volume={33},
  pages={6840--6851},
  year={2020}
}

@inproceedings{song2020score,
  title={Score-based generative modeling through stochastic differential equations},
  author={Song, Yang and Sohl-Dickstein, Jascha and Kingma, Diederik P and Kumar, Abhishek and Ermon, Stefano and Poole, Ben},
  booktitle={International Conference on Learning Representations},
  year={2021}
}

@article{ho2022classifier,
  title={Classifier-free diffusion guidance},
  author={Ho, Jonathan and Salimans, Tim},
  journal={arXiv preprint arXiv:2207.12598},
  year={2022}
}

@inproceedings{liu2022flow,
  title={Flow straight and fast: Learning to generate and transfer data with rectified flow},
  author={Liu, Xingchao and Gong, Chengyue and Liu, Qiang},
  booktitle={International Conference on Learning Representations},
  year={2023}
}

@article{kovachki2023neur,
  title={Neural operator: Learning maps between function spaces with applications to {PDE}s},
  author={Kovachki, Nikola and Li, Zongyi and Liu, Burigede and Azizzadenesheli, Kamyar and Bhattacharya, Kaushik and Stuart, Andrew and Anandkumar, Anima},
  journal={Journal of Machine Learning Research},
  volume={24},
  number={89},
  pages={1--97},
  year={2023}
}

@article{boulle2024operator,
  title={Operator learning without the adjoint},
  author={Boull{\'e}, Nicolas and Halikias, Diana and Otto, Samuel E and Townsend, Alex},
  journal={Journal of Machine Learning Research},
  volume={25},
  number={364},
  pages={1--54},
  year={2024}
}

@article{lanthaler2023operator,
  title={Operator learning with {PCA-Net}: upper and lower complexity bounds},
  author={Lanthaler, Samuel},
  journal={Journal of Machine Learning Research},
  volume={24},
  number={318},
  pages={1--67},
  year={2023}
}

@article{stepaniants2023learning,
  title={Learning partial differential equations in reproducing kernel {H}ilbert spaces},
  author={Stepaniants, George},
  journal={Journal of Machine Learning Research},
  volume={24},
  number={86},
  pages={1--72},
  year={2023}
}

@article{dalton2024boundary,
  title={Boundary constrained {G}aussian processes for robust physics-informed machine learning of linear partial differential equations},
  author={Dalton, David and Lazarus, Alan and Gao, Hao and Husmeier, Dirk},
  journal={Journal of Machine Learning Research},
  volume={25},
  number={272},
  pages={1--61},
  year={2024}
}

\end{document}